\documentclass[acmsmall]{acmart}

\acmJournal{PACMPL}
\acmVolume{1}
\acmNumber{CONF} %
\acmArticle{1}
\acmYear{2018}
\acmMonth{1}
\acmDOI{} %
\startPage{1}

\setcopyright{none}

\usepackage{booktabs}   %
\usepackage{subcaption} %
\usepackage{thmtools, thm-restate}

\declaretheorem{corollary}

\usepackage{pgf,tikz,pgfplots}
\usepackage{mathrsfs}
\usetikzlibrary{arrows, decorations.markings}
\tikzstyle{vecArrow} = [thick, decoration={markings,mark=at position
   1 with {\arrow[semithick]{open triangle 60}}},
   double distance=1.4pt, shorten >= 5.5pt,
   preaction = {decorate},
   postaction = {draw,line width=1.4pt, white,shorten >= 4.5pt}]
\pgfplotsset{compat=1.13}

\usepackage{xcolor}
\definecolor{COCcolor}{HTML}{1F77B4}
\definecolor{WLcolor}{HTML}{FDB863}
\definecolor{SLcolor}{HTML}{E66101}
\definecolor{WRcolor}{HTML}{B2ABD2}
\definecolor{SRcolor}{HTML}{5E3C99}

\usepackage{prettyref}
\newcommand{\pref}{\prettyref}
\newrefformat{thm}{Theorem~\ref{#1}}
\newrefformat{cor}{Corollary~\ref{#1}}
\newrefformat{lem}{Lemma~\ref{#1}}
\newrefformat{cha}{Chapter~\ref{#1}}
\newrefformat{sec}{Section~\ref{#1}}
\newrefformat{app}{Appendix~\ref{#1}}
\newrefformat{tab}{Table~\ref{#1}}
\newrefformat{fig}{Figure~\ref{#1}}
\newrefformat{alg}{Algorithm~\ref{#1}}
\newrefformat{exa}{Example~\ref{#1}}
\newrefformat{def}{Definition~\ref{#1}}
\newrefformat{li}{Line~\ref{#1}}
\newrefformat{eq}{Equation~\ref{#1}}

\usepackage[ruled,vlined,linesnumbered]{algorithm2e}
\newcommand{\eqdef}{\buildrel \mbox{\tiny\textrm{def}} \over =}

\newcommand{\restrname}{restriction domain of interest}
\newcommand{\restrsname}{restriction domains of interest}

\newcommand{\Restrsname}{Restriction Domains of Interest}
\newcommand{\restr}[2]{{#1}_{\restriction#2}}

\newcommand{\extendname}{\textsc{Extend}}
\newcommand{\fullyconnected}{\textsc{FullyConnected}}
\newcommand{\conv}{\textsc{2DConvolution}}
\newcommand{\bnorm}{\textsc{BatchNorm}}
\newcommand{\relu}{\textsc{ReLU}}
\newcommand{\maskrelu}{Masking \relu{}}
\newcommand{\mrelu}{\textsc{MReLU}}
\newcommand{\maxpool}{\textsc{MaxPool}}

\newcommand{\prer}[3]{\textsc{Pre}(\restr{#1}{#2}, #3)}

\newcommand{\weakprer}[3]{\textsc{WPre}(\restr{#1}{#2}, #3)}
\newcommand{\post}[2]{\textsc{Post}(#1, #2)}
\newcommand{\strongpost}[2]{\textsc{SPost}(#1, #2)}
\newcommand{\trimmedacas}[1]{\includegraphics[width=\linewidth]{#1}}
\DeclareMathOperator*{\argmax}{arg\,max}

\DeclareMathOperator*{\DPPre}{DPPre}

\SetKw{returnKw}{return}
\SetKw{breakKw}{break}
\SetKw{continueKw}{continue}
\SetKwFunction{Queue}{ConstructQueue}
\SetKwFunction{Push}{Push}
\SetKwFunction{Pop}{Pop}
\SetKwFunction{Sign}{Sign}
\SetKwFunction{Next}{Next}
\SetKwFunction{Append}{Append}
\SetKwFunction{Dim}{Dim}
\SetKwFunction{CWVert}{CWVert}
\SetKwFunction{Vert}{Vert}
\SetKwFunction{Hull}{ConvexHull}
\SetKwFunction{OrthantSign}{OrthantSign}
\SetKwFunction{PlaneIntersect}{HyperplaneIntersections}
\SetKwFunction{OrthantIntersect}{EnumerateOrthantIntersections}
\SetKwFunction{SplitPlane}{SplitPlane}

\newcommand{\subsubsubsection}[1]{\smallskip\noindent\textbf{\emph{#1}}\enspace}

\renewcommand{\hat}[1]{\widehat{#1}}
\newcommand{\hatr}[2]{\widehat{\restr{#1}{#2}}}

\newcommand{\habnetfirstlayer}[1]{
\begin{bmatrix}
    -1 & 0.25 & 1 \\
    1 & 0.5 & 1 \\
    0 & 1 & 0 \\
    0.5 & 0.5 & 2 \\
\end{bmatrix}
#1
+
\begin{bmatrix}
    1 \\
    -1 \\
    -1 \\
    -5 \\
\end{bmatrix}
}
\newcommand{\habnetsecondlayer}[1]{
\begin{bmatrix}
    -2 & 1 & 1  & 1 \\
    1  & 2 & -1 & 2 \\
\end{bmatrix}
#1
+
\begin{bmatrix}
    1 \\
    0 \\
\end{bmatrix}
}

\begin{document}
\title[A Symbolic Neural Network Representation and Applications]{A Symbolic
Neural Network Representation and its Application to 
Understanding, Verifying, and Patching Networks}

\author{Matthew Sotoudeh}
\affiliation{
  \department{Computer Science}
  \institution{University of California, Davis}
  \city{Davis}
  \state{California}
  \postcode{95616}
  \country{United States of America}
}
\email{masotoudeh@ucdavis.edu}

\author{Aditya V.\ Thakur}
\affiliation{
  \position{Assistant Professor}
  \department{Computer Science}
  \institution{University of California, Davis}
  \city{Davis}
  \state{California}
  \postcode{95616}
  \country{United States of America}
}
\email{avthakur@ucdavis.edu}

\begin{abstract}
    Analysis and manipulation of trained neural networks is a challenging and
    important problem. We propose a symbolic representation for
    piecewise-linear neural networks and discuss its efficient computation.
    With this representation, one can translate the problem of analyzing a
    complex neural network into that of analyzing a finite set of affine
    functions. We demonstrate the use of this representation for three
    applications. First, we apply the symbolic representation to computing
    weakest preconditions on network inputs, which we use to exactly visualize
    the advisories made by a network meant to operate an aircraft collision
    avoidance system. Second, we use the symbolic representation to compute
    strongest postconditions on the network outputs, which we use to perform
    bounded model checking on standard neural network controllers. Finally, we
    show how the symbolic representation can be combined with a new form of
    neural network to perform patching; i.e., correct user-specified behavior
    of the network.
\end{abstract}

\ccsdesc[500]{Software and its engineering~General programming languages}
\ccsdesc[300]{Social and professional topics~History of programming languages}

\maketitle

\section{Introduction}
\label{sec:Introduction}

\subsection{Introduction to Neural Networks}

The past decade has seen the rise of deep neural networks
(DNNs)~\cite{Goodfellow:DeepLearning2016} to solve a variety of problems,
including image recognition~\cite{Szegedy:CVPR2016,Krizhevsky:CACM2017},
natural-language processing~\cite{BERT:CoRR2018}, and autonomous vehicle
control~\cite{julian2018deep}. Typically, DNNs are trained using large data
sets, validated on a test set, and then deployed. As they permeate more and
more systems, there is growing need for tools to more deeply analyze and modify
such networks once they have been trained. This paper presents techniques for
understanding, verifying, and patching (correcting) trained neural networks.
These various applications rely on a new symbolic representation of neural
networks introduced in this paper.

In this work, we focus on the major class of \emph{piecewise-linear neural
networks}, which can be decomposed into a set of affine functions
(\pref{def:PWL}). For example, the function $\max(x, 0)$ is piecewise-linear;
when $x \leq 0$ it takes the affine form $x \mapsto 0$, and when $x > 0$ it
takes the affine form $x \mapsto x$. Although this may seem like a restrictive
definition, we will show in~\pref{sec:SymbolicRepresentation} that neural
networks using the most common building blocks (like \conv{}, \relu{}, and
\maxpool{} layers) are in fact piecewise-linear. Notably, this work deals with
\emph{trained} networks (we assume that the training has been already been
completed), and makes no restriction on the training process.

\subsection{Contributions}
\label{sec:Contributions}

\subsubsubsection{A symbolic representation for deep neural networks.}
In \pref{sec:SymbolicRepresentation}, we propose a \emph{symbolic
representation for neural networks} which expresses a complex,
highly-non-linear neural network in terms of a set of affine functions, each of
which fully captures the behavior of the network for a particular subspace of
the input domain. We refer to this subspace as the \restrname{}, and consider
in this work two-dimensional \restrsname{}, which we show
in~\pref{sec:Evaluation} leads to many important applications. Fundamentally,
this representation allows us to translate questions about the
highly-non-linear neural network into a series of questions about finitely-many
affine functions, which are one of the best-studied class of functions in
modern mathematics. As we will see, this symbolic representation enables
understanding, verifying, and patching trained neural networks.

\subsubsubsection{Understanding network behavior using weakest precondition.}
In \pref{sec:WeakestPrecondition}, we show how the symbolic representation can
be used to compute \emph{weakest preconditions} on neural networks; i.e., all
points in the input space which are mapped to a particular subset of the output
space by the network. We show that prior work can be adapted to compute
preconditions, but they are not guaranteed to produce the \emph{weakest}
precondition (i.e., they produce a subset of the points).  This
weakest-precondition primitive allows us to understand the behavior of a neural
network by \emph{exactly visualizing decision boundaries}.  In particular, we
apply this technique to the aircraft collision avoidance system ACAS Xu
\cite{julian2018deep}, comparing the use of the proposed symbolic
representation to prior work based on an (over-approximating) abstraction of
the network function.

In~\pref{sec:Evaluation-Precondition} we use the preconditions to plot network
decision boundaries, and compare against the preconditions found using a
representation from prior work. We find that the representations in prior work
are not precise enough to compute particularly useful preconditions, whereas
the representation used in this work can find weakest preconditions in a matter
of seconds.

\subsubsubsection{Bounded model checking of safety properties using strongest postcondition.}
In \pref{sec:StrongestPostcondition}, we show how the symbolic representation
can be adapted to compute \emph{strongest post-conditions} on the network
output, i.e. all outputs reachable by the network from a set of inputs.

In~\pref{sec:Evaluation-Postcondition}, we use this to perform \emph{bounded
model checking} \cite{biere2009bounded} on three reinforcement learning
controller models, comparing against the state-of-the-art neural network SMT
solver ReluPlex~\cite{reluplex:CAV2017}. We find that the ability to directly
handle disjunctions as well as re-use information across multiple verification
steps allows the proposed technique to verify significantly more steps in the
same time limit (including finding a counter-example to the safety
specification for one of the models). 

\subsubsubsection{Patching deep neural networks.}
In \pref{sec:PatchingNetworks}, we consider the problem of \emph{patching a
neural network}: changing a small number of weights in the network to precisely
manipulate the network decision boundaries (eg. to fix erroneous or undesired
behavior).  Three features in particular make this a challenging problem:
\begin{enumerate}
    \item Neural networks are usually high-dimensional and highly non-linear,
        pushing the bounds of traditional SMT solvers such as Z3
        \cite{TACAS:deMB08}.
    \item The behavior we would like to patch occurs over entire
        \emph{polytopes} (containing infinitely many points) in the input
        region, making it difficult to apply gradient-descent-based methods
        which assume a finite set of training points.
    \item One may not have the original training or test data for a network
        (eg. due to user privacy concerns), making it difficult to ensure that
        applying a patch to fix one set of erroneous behavior does not result
        in corrupting the behavior for another set of inputs.
\end{enumerate}

We discuss prior work that relies on over-approximations and/or differentiable
relaxations, but find them sub-optimal for this particular problem. We then
define a new class of neural networks termed \emph{Masking Networks}
(\pref{sec:MaskingNetworks}), and show how standard networks can be transformed
into equivalent Masking Networks.  Finally, we show how our symbolic
representation can be used on such networks to \emph{lower the problem of
patching on polytopes to that of patching on a finite set of vertices},
intuitively similar to how the simplex algorithm lowers optimization over a
polytope to optimization over a finite set of vertices. These changes allow us
to rephrase the problem of patching as solving a finite MAX SMT problem. We
discuss a number of approaches to this problem, including using the Z3 theorem
prover \cite{TACAS:deMB08}, which we find to be too slow for our purposes. We
then discuss an algorithm that can exactly and efficiently solve such problems
when only a single weight is changed. Finally, we greedily apply that solver to
find better solutions by changing multiple weights.

In~\pref{sec:Evaluation-Netpatch} we test this network patching technique on an
aircraft collision-avoidance network with three patch specifications. We
show quantitatively and qualitatively that patching is effective even with
relatively few iterations. We also find that patches made to one \restrname{}
can have beneficial effects on others (i.e., patches \emph{generalize}).

\pref{sec:Overview} presents an overview of the symbolic representation and its
applications, while \pref{sec:RelatedWork} describes related work and
\pref{sec:Conclusion} summarizes the main results of this paper.

\section{Overview}
\label{sec:Overview}

\subsection{Deep Neural Networks and Piecewise-Linear Functions}
\label{sec:Overview-DNNs}
Deep neural networks \cite{Goodfellow:DeepLearning2016} generally consist of
multiple \emph{layers} (which are themselves vector-valued functions) applied
sequentially to an input vector. 
Thus, a DNN can be thought of as a function $f = f_n \circ f_{n-1} \circ
\cdots \circ f_1$, where $f_i$ is function representing the $i$th layer.
Four layer types are particularly common in
feed-forward neural networks:

\begin{enumerate}
    \item \fullyconnected{} layers correspond to arbitrary affine maps, where
        the particular map is determined by the \emph{weight matrix} of the
        layer.
    \item \conv{} layers correspond to a restricted subset of affine maps
        particularly well-suited to image recognition and parameterized by a
        \emph{filter tensor}.
    \item \relu{} layers effectively enforce a lower-bound on the value of
        each coefficient in the output vector. If the output of the previous
        layer was $x$, then $\relu{}(x)_i = 0$ when $x_i \leq 0$ and $x_i$ when
        $x_i > 0$ (where $\relu{}(x)_i$ is the $i$th component of the layer's
        output vector). For example, $\relu{}((10, -2)) = (10, 0)$.
    \item \maxpool{} layers condense their input vectors by taking only the
        maximum value in each of a number of coefficient groups.
    \item \bnorm{} layers normalize each component given a fixed mean and
        variance.
\end{enumerate}
Of particular importance to our paper, all of these layers are 
examples of \emph{piecewise-linear functions}:

\begin{definition}
    \label{def:PWL}
    A function $f : A \to B$ is referred to as a \emph{piecewise-linear
    function} if its input domain $A$ can be partitioned by a finite set of
    (possibly unbounded) convex polytopes $\{ P_1, P_2, \ldots, P_n \}$ such
    that, within any partition $P_i \subseteq A$, there exists an \emph{affine}
    map $F_i : P_i \to B$ satisfying $f(x) = F_ix$ for any $x \in P_i$.
\end{definition}

In other words, each layer function can be broken up into finitely many affine
functions, with one of those functions being applied depending on the location
in the input space of $x$.
For example, the \relu{} function, when restricted to any one orthant (i.e.,
where the signs of all input coefficients are constant), corresponds to a
single affine map that zeros out entries with a non-positive sign.  Because
(1) there are finitely many orthants that together partition the input space,
(2) each orthant can be expressed as a convex polytope, and (3) within each
orthant $\relu{}(x)$ is affine, we can say that $\relu{}$ is a
\emph{piecewise-linear} function.

\subsection{\Restrsname{}}
\label{sec:RestrictionDescription}

\begin{figure}
    \centering
    \newcommand{\Depth}{2}
\newcommand{\Height}{2}
\newcommand{\Width}{2.3}
\begin{tikzpicture}
\coordinate (O) at (0,0,0);
\coordinate (A) at (0,\Width,0);
\coordinate (B) at (0,\Width,\Height);
\coordinate (C) at (0,0,\Height);
\coordinate (D) at (\Depth,0,0);
\coordinate (E) at (\Depth,\Width,0);
\coordinate (F) at (\Depth,\Width,\Height);
\coordinate (G) at (\Depth,0,\Height);

\coordinate (P) at (0,1.3,0);
\coordinate (Q) at (0,1.3,\Height);
\coordinate (R) at (\Depth,.4,\Height);
\coordinate (S) at (\Depth,.4,0);
\coordinate (XC) at (.6,.6,0); %
\coordinate (AC) at (-.2,2.5,0); %

\coordinate (ArrowStart) at (\Depth-1.1,.6,0);
\coordinate (ArrowEnd) at (4-.1,.6,0);

\coordinate (fArrowStart) at (\Depth-.5,1.3,0);
\coordinate (fArrowEnd) at (4-.5,1.3,0);

\coordinate (P1) at (3.5+.1,1,0);
\coordinate (Q1) at (3.5,0.5,0);
\coordinate (R1) at (3.5+\Depth-1,0,0);
\coordinate (S1) at (3.5+\Depth-.1,1,0);
\coordinate (T1) at (3.5+\Depth-.5,1.5,0);
\coordinate (U1) at (3.5+\Depth-1.5,.8,0);
\coordinate (XC1) at (4.5,.6,0); %

\draw (O) -- (C) -- (G) -- (D) -- cycle;%
\draw (O) -- (A) -- (E) -- (D) -- cycle;%
\draw (O) -- (A) -- (B) -- (C) -- cycle;%
\draw[blue,fill=red!20,opacity=0.8] (P) -- (Q) -- (R) -- (S) -- cycle;%
\draw (D) -- (E);
\draw[blue,fill=red!20,opacity=0.8] (P1) -- (Q1) -- (R1) -- (S1) -- (T1) -- (U1) -- cycle;%
\node at (XC1) {\small $\restr{f}{X}(x)$};
\draw[vecArrow] (ArrowStart) to (ArrowEnd);
\node[text width=5.7cm] at (5,.85,0) {\tiny $\restr{f}{X} \colon X \to B$};
\draw[vecArrow] (fArrowStart) to (fArrowEnd);
\node[text width=5.7cm] at (5,1.5,0) {\tiny $f \colon A \to B$};

\draw (E) -- (F) -- (G) -- (D);%
\draw (A) -- (B) -- (F) -- (E) -- cycle;%
\draw (C) -- (B) -- (F) -- (G) -- cycle;%

\node at (XC) {$X$};
\node at (AC) {$A$};

\coordinate (O1) at (3.5+0,0,0);
\coordinate (A1) at (3.5+0,\Width-.5,0);
\coordinate (BC) at (3.5-.1,\Width-.4,0);
\coordinate (D1) at (3.5+\Depth,0,0);
\coordinate (E1) at (3.5+\Depth,\Width-.5,0);
\draw (O1) -- (A1) -- (E1) -- (D1) -- cycle;%
\node at (BC) {$B$};

\end{tikzpicture}
    \caption{An illustration of restriction domain of interest.  The function
    $f$ has a three-dimensional domain and $\restr{f}{X}$ is $f$ restricted to
    a particular two-dimensional \restrname{} $X$.}
    \label{fig:restrict-domain}
\end{figure}
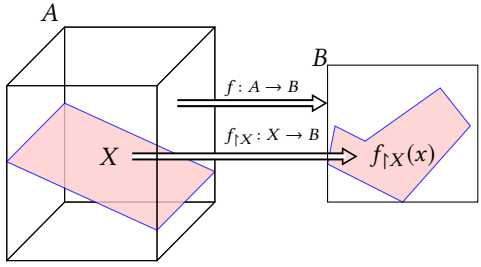

One particular insight which we will utilize throughout the paper is that, when
analyzing networks, one is usually only interested in a particular \emph{domain
of interest}. For example, suppose a network $f : \mathbb{R}^2 \to
\mathbb{R}^1$ predicts the probability of acquiring cancer in the next five
years, taking as input a vector with two components, $x = (x_1, x_2)$, with
$x_1$ being the age of the patient and $x_2$ being the percent of their
immediate family that has died of cancer. Then, although $f$ is theoretically
able to make predictions about, say, $1,000$-year-old individuals with $110\%$
of their immediate family having died of cancer, such scenarios are in practice
impossible. Furthermore, in particular scenarios, we may be concerned with an
even smaller subspace of the input domain, for example if we only want to
understand how the network classifies children under the age of $10$.

We call such restricted subsets of the input domain the \emph{\restrname{}} for
the particular analysis being performed, usually denoted by $X$, and to
indicate that we only intend to consider the behavior over the \restrname{}, we
will often write $\restr{f}{X}$ (read ``$f$ restricted to the \restrname{}
$X$''). This is illustrated in~\pref{fig:restrict-domain}, where a function
with a three-dimensional domain is restricted to a particular two-dimensional
\restrname{}. In~\pref{sec:Algorithms-ReLU2D} we will focus on two-dimensional
\restrsname{}, which we will show can help significantly improve the efficiency
of our analysis.

\begin{definition}
    \label{def:Restriction}
    A function $f : A \to B$ \emph{restricted} to a particular
    \emph{\restrname{}} $X \subseteq A$, denoted $\restr{f}{X}$, is a new
    function $\restr{f}{X} : X \to B$ such that $\restr{f}{X}(x) = f(x)$ for
    any $x \in X$.
\end{definition}

\subsection{A Symbolic Representation for Deep Neural Networks}
\label{sec:Overview-SymbolicRepresentation}

In this work, we develop a \emph{symbolic representation for a deep neural
network} $f = f_n \circ f_{n-1} \circ \cdots \circ f_1$, denoted $\hat{f}$, that
enables precise and efficient analyses of $f$. The symbolic representation
partitions the input domain of $f$ such that, within each partition, the output
of $f(x)$ is affine with respect to $x$.  That is, $\hat{f} = \{ (P_1, F_1),
(P_2, F_2), \ldots, (P_n, F_n) \}$, where the set $\{P_1, P_2, \ldots, P_n \}$
partitions the domain of $f$ and each $F_i$ is an affine map that exactly
matches the output of $f$ on any points in $P_i$.  The existence of
such a partitioning is guaranteed for most common neural networks
by~\pref{thm:PWL-Hat} and the fact that compositions of piecewise-linear
functions are themselves piecewise-linear.  

Our key insight is that \emph{the symbolic
representation $\hat{f}$ allows us to translate problems related to a single
highly-non-linear network $f$ into a series of problems dealing with
finitely-many affine functions.} The efficiency of this symbolic representation
can be further improved by noting that we are usually only concerned about
$\restr{f}{X}$ for some restriction domain of interest $X$
(\pref{def:Restriction}). Thus, instead of computing the full $\hat{f}$, we only
need to compute $\hatr{f}{X}$. In this paper, we consider two-dimensional
\restrsname{} $X$.

Consider the following neural network $N_1$ defined as:
\begin{equation}
    \label{eq:habitability-network}
    f(x) = \habnetsecondlayer{
        \relu{}\left(
            \habnetfirstlayer{
                \begin{bmatrix}
                    x_1 \\
                    x_2 \\
                    x_3 \\
                \end{bmatrix}
            }
        \right)
    }
\end{equation}
This network $N_1$ transforms points in a three-dimensional
space (coordinates $x_1, x_2, x_3$) to points in a two-dimensional space
(coordinates $y_1, y_2$). The output of this network may be interpreted,
perhaps, to be a prediction as to whether a particular insect can survive in a
certain spot. We define two \emph{classification regions}, the first is $H = \{
    y \mid y_2 > y_2 \}$, which we may interpret as the network predicting a
spot is ``habitable,'' while the second is $U = \{ y \mid y_1 > y_2 \}$, which
we might interpret as the network predicting a spot as ``uninhabitable.'' Note
that, \emph{in general}, classification regions need not span the entire output
space---for example, in this scenario, we do not define the classification of
the network when $y_1 = y_2$. Such a network might have been trained using
gradient descent and a set of labeled training points collected by surveying a
number of points in the (``real-world'') space of interest.

Suppose we are interested in how the network $N_1$ behaves on a
particular plot of land at sea-level, $X = \Hull(\{ (0, 0, 0), (0, 3, 0), (3,
0, 0), (3, 3, 0) \})$, where $\Hull(S)$ is the smallest convex set that
contains the set of points $S$.  We can compute $\hatr{f}{X}$ for the $f$ in
\pref{eq:habitability-network}, which is visualized
in~\pref{fig:habit-partitions}.  Each colored region in
\pref{fig:habit-partitions} shows a partitions $P_i$ of $X$ such that
$\restr{f}{P_i} = F_i$ can be written as an affine map.  The particular
$\hatr{f}{X} = \{ (P_1, F_1), (P_2, F_2), \ldots, (P_6, F_6)\}$ is listed
below:
\begingroup %
\setlength\arraycolsep{3pt}
{\small
\begin{align*}
    \hatr{f}{X} = \biggr\{
    \biggr(&\Hull(\{ (0.5, 1, 0), (1, 0, 0), (1.25, 1, 0) \}), 
    \hspace{1ex} x \mapsto \begin{bmatrix} 3 & 0 & -1\\ 1 & 1.25 & 3\\ \end{bmatrix}x + \begin{bmatrix} -2\\ -1\\ \end{bmatrix}\biggr), \\
    \biggr(&\Hull(\{ (0, 2, 0), (0, 3, 0), (0.5, 1, 0), (1.25, 1, 0), (1.75, 3, 0) \}), 
    \hspace{1ex} x \mapsto \begin{bmatrix} 3 & 1 & -1\\ 1 & 0.25 & 3\\ \end{bmatrix}x + \begin{bmatrix} -3\\ 0\\ \end{bmatrix}\biggr), \\
    \biggr(&\Hull(\{ (0, 1, 0), (0, 2, 0), (0.5, 1, 0) \}), 
    \hspace{1ex} x \mapsto \begin{bmatrix} 2 & 0.5 & -2\\ -1 & -0.75 & 1\\ \end{bmatrix}x + \begin{bmatrix} -2\\ 2\\ \end{bmatrix}\biggr), \\
    \biggr(&\Hull(\{ (0, 0, 0), (0, 1, 0), (0.5, 1, 0), (1, 0, 0) \}), 
    \hspace{1ex} x \mapsto \begin{bmatrix} 2 & -0.5 & -2\\ -1 & 0.25 & 1\\ \end{bmatrix}x + \begin{bmatrix} -1\\ 1\\ \end{bmatrix}\biggr), \\
    \biggr(&\Hull(\{ (1, 0, 0), (1.25, 1, 0), (3, 0, 0), (3, 1, 0) \}), 
    \hspace{1ex} x \mapsto \begin{bmatrix} 1 & 0.5 & 1\\ 2 & 1 & 2\\ \end{bmatrix}x + \begin{bmatrix} 0\\ -2\\ \end{bmatrix}\biggr), \\
    \biggr(&\Hull(\{ (1.25, 1, 0), (1.75, 3, 0), (3, 1, 0), (3, 3, 0) \}), 
    \hspace{1ex} x \mapsto \begin{bmatrix} 1 & 1.5 & 1\\ 2 & 0 & 2\\ \end{bmatrix}x + \begin{bmatrix} -1\\ -1\\ \end{bmatrix}\biggr)
    \biggr\}
\end{align*}
}
\endgroup

\begin{figure}[t]
    \centering
    \begin{subfigure}{.48\linewidth}
        \includegraphics[width=\linewidth]{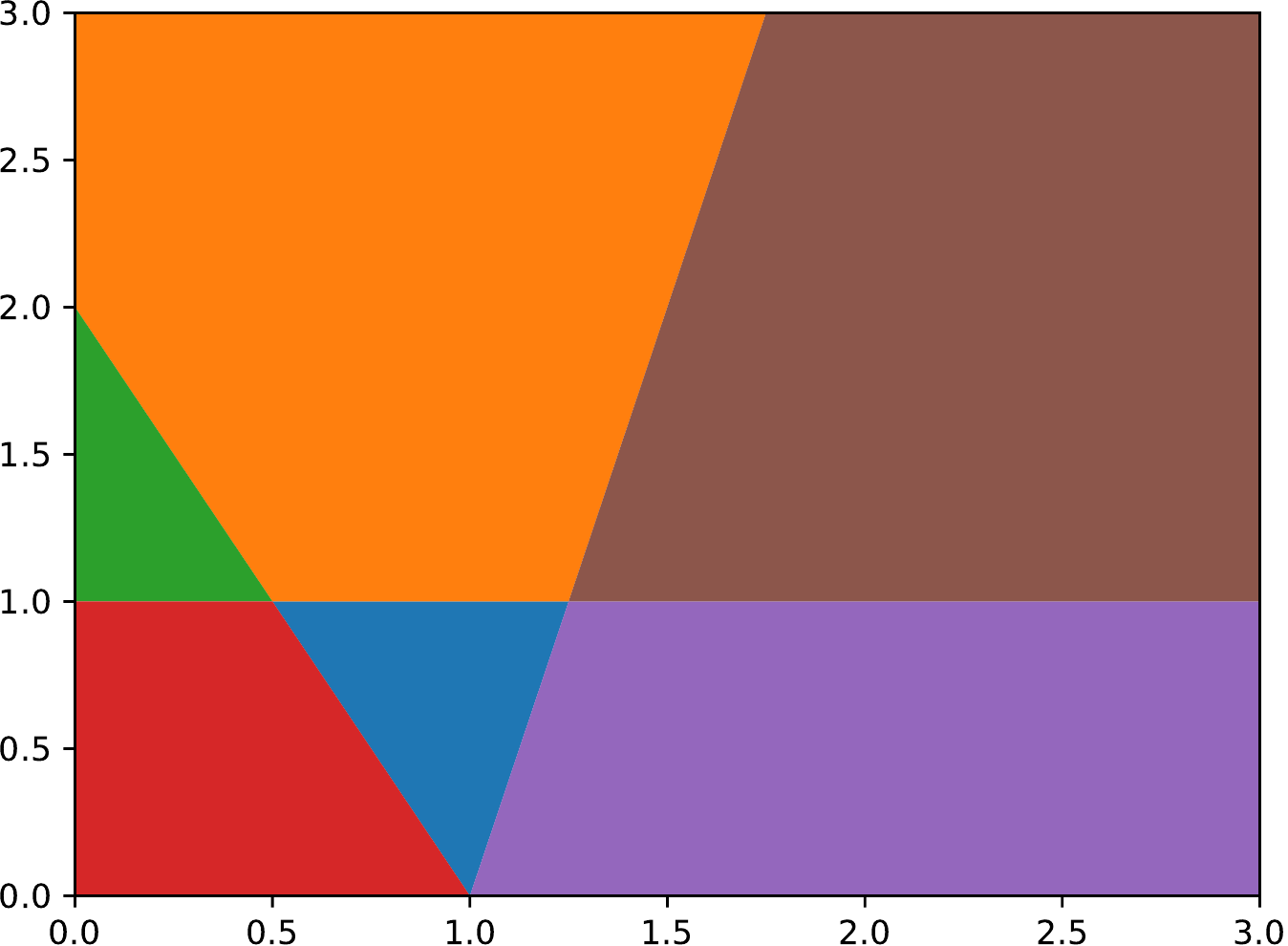}
        \caption{Linear partitions in $\hatr{f}{X}$ for the network}
        \label{fig:habit-partitions}
    \end{subfigure}
    \begin{subfigure}{.48\linewidth}
        \includegraphics[width=\linewidth]{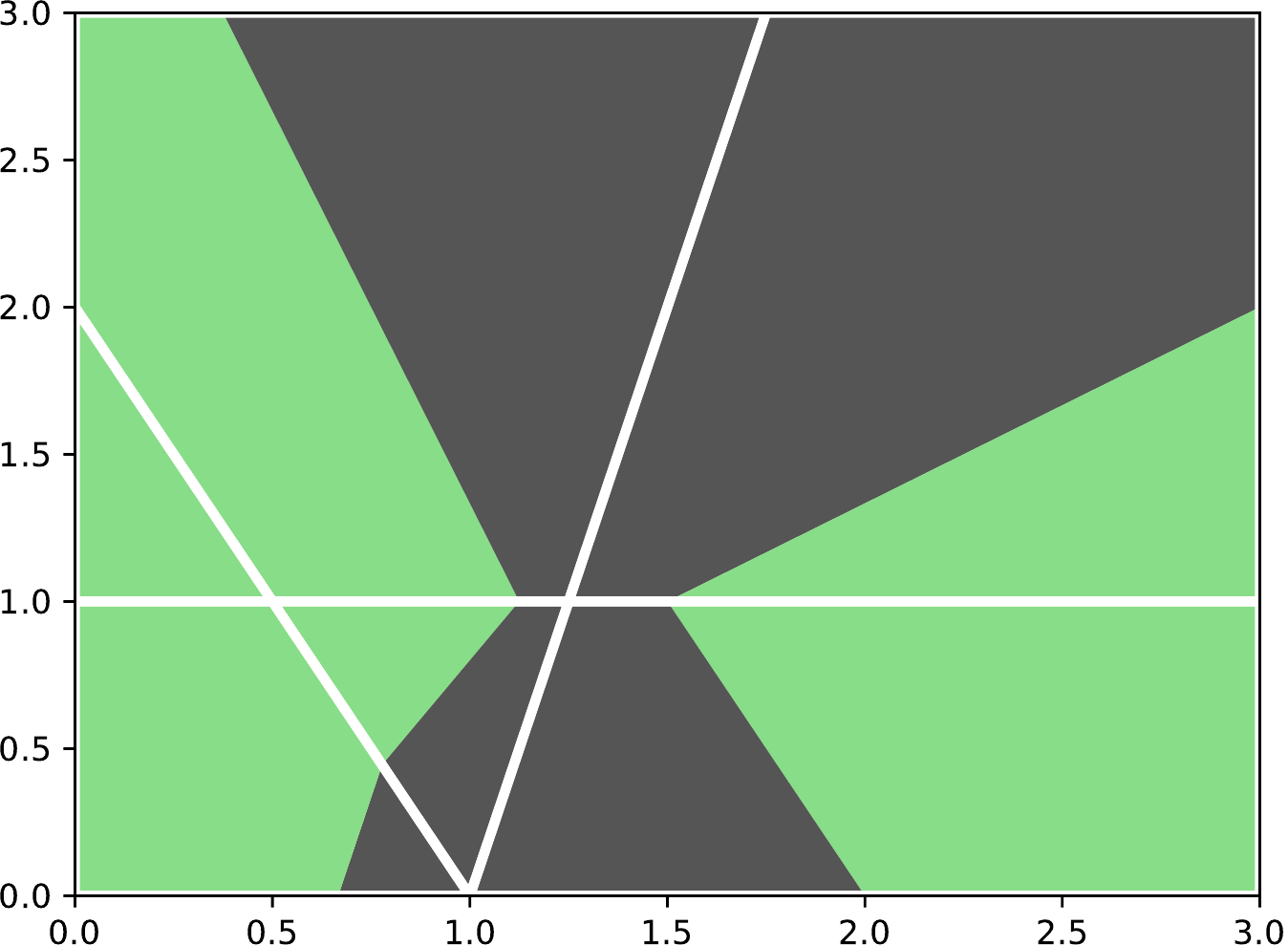}
        \caption{Classification regions of the network}
        \label{fig:habit-classes}
    \end{subfigure}
    \caption{Symbolic representation of neural network $N_1$ defined by function $f$ in \pref{eq:habitability-network}}
    \label{fig:habit-net}
\end{figure}

\subsection{Understanding Network Behavior using Weakest Precondition}
\label{sec:Overview-WeakestPrecondition}

A first natural question to ask about a neural network is \emph{what inputs
result in a particular set of outputs?}Answering such questions corresponds to computing the \emph{weakest
precondition}, i.e. all points in $X$ that are mapped to $Y$ by $\restr{f}{X}$:
$\weakprer{f}{X}{Y} = \{ x \in X \mid f(x) \in Y \}$.
\pref{sec:WeakestPrecondition-WithHat} shows how $\hatr{f}{X}$ can be used to
efficiently compute this set, which may in general be non-convex.

For example, in our running scenario, we may wish to ask: \emph{which areas does
the network $N_1$ predict will be habitable?} The polytope $H = \{ y \mid y_2 >
y_1 \}$ denotes the set of output points that are habitable, similarly $U$ for
uninhabitable points. Using $\hatr{f}{X}$ (\pref{fig:habit-partitions}), we can
compute $\weakprer{f}{X}{H}$ and $\weakprer{f}{X}{U}$, which represent the set
of \emph{all} habitable and uninhabitable points in $X$, respectively.
\pref{fig:habit-classes} plots these precondition sets on a two-dimensional axis
with each set assigned a different color. Thus,  we can \emph{precisely and
exactly} visualize the \emph{decision boundaries} of a deep neural network using
$\weakprer{f}{X}{Y}$.

Effectively, within each linear partition in~\pref{fig:habit-classes}
(delineated with white borders), the network's output is affine, and thus the
decision boundary is linear --- notice that the decision boundary is a single
straight line in any given linear partition in~\pref{fig:habit-classes}, even
though the overall decision boundaries have ``corners.'' This allows us to
quickly and precisely determine decision boundaries within each linear partition
using standard linear algebra techniques
(\pref{sec:WeakestPrecondition-WithHat}). An evaluation of $\weakprer{f}{X}{Y}$
on a large network is presented in \pref{sec:Evaluation-Precondition}, where we
produce figures like~\pref{fig:overview-acas} that show how the advisory made by
an aircraft collision-avoidance network depends on the location of an intruder.

\subsection{Bounded Model Checking of Safety Properties using Strongest Postcondition}
\label{sec:Overview-StrongestPostcondition}

Consider another neural network, which controls a motor attached to an \emph{inverted
pendulum} which is in turn described by its current position $\theta$ and
angular velocity $\omega$. The network takes $\theta$ and $\omega$ as inputs,
and produces as output the angular acceleration to apply to the pendulum by the
motor. The network is trained to keep the pendulum inverted, but it has not
been formally verified that the network correctly accomplishes this goal for
\emph{all} valid starting conditions.

In~\pref{sec:StrongestPostcondition}, we address the problem of \emph{bounded
model checking} for such control networks.  In particular, we show how the
symbolic representation enables us to compute the \emph{strongest
postcondition}, i.e. the set of all possible outputs given that the input is in
a particular region: $\strongpost{f}{X} = \{ f(x) \mid x \in X \}$.
Consequently, we show how to verify the following claim for the inverted
pendulum controller:
\emph{Starting from any valid initial condition and applying the network
controller to the system for $K$ time-steps, there is no timestep for which the
pendulum becomes non-inverted (i.e., dips below the horizontal).}

In \pref{sec:Evaluation-Postcondition}, we evaluate our approach against one
using ReluPlex \cite{reluplex:CAV2017} on the pendulum model and two other
standard neural network controller models from \citet{DBLP:conf/pldi/ZhuXMJ19}.

\begin{figure}[t]
    \centering
    \begin{subfigure}[t]{.35\linewidth}
        \trimmedacas{figures/netviz/fhat-all}
        \caption{Network visualization.}
        \label{fig:overview-acas}
    \end{subfigure}
    \hspace{2cm}
    \begin{subfigure}[t]{.35\linewidth}
        \trimmedacas{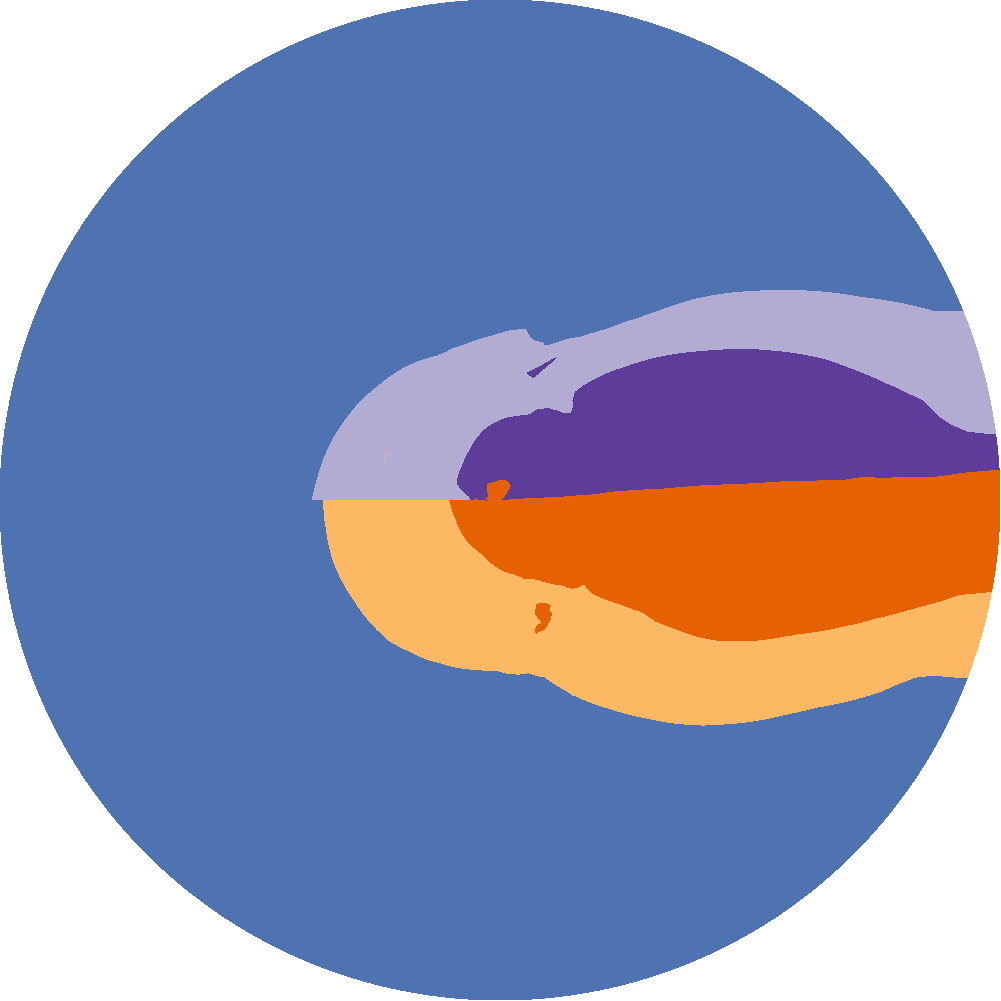}
        \caption{Patched network visualization (plane icons removed to better
        show decision boundaries).}
        \label{fig:overview-acas-patched}
    \end{subfigure}
    \\
    Legend: \textcolor{COCcolor}{\rule[.2\baselineskip]{1em}{2pt}} Clear-of-Conflict,
    \textcolor{WRcolor}{\rule[.2\baselineskip]{1em}{2pt}} Weak Right, 
    \textcolor{SRcolor}{\rule[.2\baselineskip]{1em}{2pt}} Strong Right, 
    \textcolor{SLcolor}{\rule[.2\baselineskip]{1em}{2pt}} Strong Left, 
    \textcolor{WLcolor}{\rule[.2\baselineskip]{1em}{2pt}} Weak Left.
    \caption{Visualizing preconditions for and patching an aircraft collision
    avoidance network.}
    \label{fig:overview-acas-all}
\end{figure}

\subsection{Patching Deep Neural Networks}
\label{sec:Overview-PatchingNetworks}

Finally, we introduce the problem of \emph{patching} a neural network. Patching
a network $N$ entails fixing the behavior of $N$ in a particular region $R$ of
the input space by modifying the weights of $N$.  For example,
in~\pref{fig:overview-acas}, we see a visualization of the actions suggested by
an aircraft collision avoidance system when an intruder is at different places
relative to the ownship. For the most part, the network's outputs seem
reasonable; when the intruder is reasonably far away, it reports that the plane
is ``clear of conflict'' (blue), when the intruder is nearby but on the left of
the plane it instructs to make a ``hard right'' (dark purple), etc. However, we
can also see that there are a number of regions in the input space where the
network seems to make unsafe suggestions, eg. when the intruder is
behind-and-to-the-left of the plane, we see a ``band'' of light orange that
indicates the network has instructed a ``weak left turn,'' which would actually
turn the ownship \emph{towards} the intruder.  We may want to \emph{patch this
behavior}, i.e. modify the weights of the network such that the network instead
instructs ``weak right,'' which would turn the plane away from the intruder. A
\emph{patch specification} describes which regions of the input space we want to
change the classification over and which we want the classification to stay the
same. 

With this patch specification in hand, we can formulate the problem as a MAX-SMT
instance that can, in theory, be solved with an SMT solver such as
Z3~\cite{TACAS:deMB08}. However, the inherently non-linear and high-dimensional
nature of our networks makes this infeasible. To remedy this, we transform the
network into an equivalent \emph{Masking Network}, described
in~\pref{sec:MaskingNetworks}.  Then, we show in~\pref{thm:KeyPoints} how to use
$\hat{f}$ to translate this patch specification into a set of (finitely many)
\emph{key points} in the input space, for which patching on the key points is
equivalent to patching on the regions. This allows us to lower the problem of
patching on infinite regions to patching on finitely-many points, similar to how
the simplex algorithm lowers optimization on polytopes to optimization on
vertices. If we utilize these key points and restrict the number of weights we
change at any one iteration, we show how to translate the previously-non-linear
MAX-SMT instance into a linear one.  Finally, we discuss a greedy MAX-SMT solver
that can quickly find updates to individual weights that will maximize the
number of key points with the correct classification. The result of this process
is shown in~\pref{fig:overview-acas-patched}, where after changing only five
weights the offending region has been removed.

In~\pref{sec:Evaluation-Netpatch}, we perform this process for the ACAS Xu
network for a variety of different patch specifications, finding that it is
able to effectively patch network behavior and that behavior patched in one
region of the input space generalizes well to other regions.

\section{A Symbolic Representation for Deep Neural Networks}
\label{sec:SymbolicRepresentation}

In this section, we will introduce our symbolic function representation and
discuss a number of immediate theoretical results relating to its definition.

In all below discussion, we will use the term ``polytope'' or ``(bounded)
polytope'' to refer to the convex hull of finitely many points while we will
use ``(potentially unbounded) polytope'' to refer to the intersection of
finitely many half-spaces.

Given a function $f : A \to B$ where $A \subseteq \mathbb{R}^a, B \subseteq
\mathbb{R}^b$, we represent our \emph{symbolic representation $\hat{f}$} as a
set $\hat{f} = \{ (P_1, F_1), (P_2, F_2), \ldots, (P_n, F_n) \}$, where:

\begin{enumerate}
    \item The set $\{ P_1, P_2, \ldots, P_n \}$ should have each $P_i \subseteq
        A$ and, further, should partition $A$ (the domain of $f$), except
        possibly for overlapping boundaries.
    \item Given any such $(P_i, F_i)$ pair, with $F_i : P_i \to B$, the
        following holds: $\forall x \in P_i: f(x) = F_i(x)$.
\end{enumerate}

These two requirements describe a \emph{symbolic representation of $f$} by
breaking its input into regions $P_i$ where the behavior of $f$ is captured by
a (possibly simpler) function $F_i$.

However, as described, there is no reason to believe that the constituent $F_i$s
will be more analysis-friendly than the original function $f$; indeed, $\hat{f}
= \{ (A, f) \}$ meets the constraints listed above but clearly gives us no
additional insight into $f$.

To remedy this, we further restrict our description to a form which noticeably
improves our ability to analyze $f$:

\begin{enumerate}
    \setcounter{enumi}{2}
    \item Each $P_i$ is a convex polytope.
    \item Each $F_i : P_i \to B$ is affine.
\end{enumerate}

Indeed, the first condition ensures that the partitioning is efficiently
representable and manipulatable, while the second condition ensures that we can
use standard linear algebra techniques and algorithms to perform the desired
analysis within any given partition.

\begin{definition}
    \label{def:fhat}
    Given a function $f : A \to B$ where $A \subseteq \mathbb{R}^a$ and $B
    \subseteq \mathbb{R}^b$, we define the \emph{symbolic representation of
    $f$}, written $\hat{f}$, to be a set of tuples
    $\hat{f} = \{ (P_1, F_1), \ldots, (P_n, F_n) \}$,
    such that:
    \begin{enumerate}
        \item The set $\{ P_1, P_2, \ldots, P_n \}$ partitions the domain
            of $f$, except possibly for overlapping boundaries.
        \item Given any such $(P_i, F_i)$ pair, the following holds: $\forall x
            \in P_i: f(x) = F_i(x)$.
        \item Each $P_i$ is a convex polytope.
        \item Each $F_i : P_i \to B$ is affine.
    \end{enumerate}
\end{definition}

As discussed in~\pref{sec:RestrictionDescription}, we can often improve the
efficiency of the symbolic representation by only computing it over the
restricted domain $X$, which we call the \emph{symbolic representation of $f$
restricted to $X$} and write $\hatr{f}{X}$. Notably,~\pref{def:fhat} still
holds; we have simply explicitly stated that the domain $A$ is the set we have
defined as $X$. Thus, in definitions we will usually only use the explicit
syntax $\hatr{f}{X}$ when we wish to place restrictions on the possible values
of the domain $X$ (for example, to state that an algorithm only works when the
domain is bounded).

Finally, we define another primitive, denoted by the operator $\otimes$ and
sometimes referred to as $\extendname$, such that $\extendname{}(h, \hat{g}) =
h \otimes \hat{g} = \hat{h \circ g}$.
This is the primitive which we will implement in our algorithms, as it enables
easy composition across multiple layers in a neural network. For example,
suppose one wishes to compute $\hatr{f}{X}$ where $f = f_n \circ f_{n-1} \circ
\cdots \circ f_1$, and has access to algorithms for computing
$\extendname{}(f_i, \cdot)$ for each composed function $f_i$.

Then, we can initially define the \emph{identity map} $I : x \mapsto x$, which
is affine across its entire input space, so we always have $\hatr{I}{X} = \{
    (X, I) \}$
Then, by the definition of $\extendname{}(f_1, \cdot)$, we can compute:
$f_1 \otimes \hatr{I}{X} = \hatr{(f_1 \circ I)}{X} = \hatr{f_1}{X}$ (The final
equality holds as, by the definition of the identity map $I$, $g \circ I = g$
holds for any function $g$.)

We can then iteratively apply this procedure to inductively compute
$\hatr{(f_i \circ f_{i-1} \circ \cdots f_1)}{X}$ from $\hatr{(f_{i-1} \circ
\cdots f_1)}{X}$ like so:
$f_i \otimes \hatr{(f_{i - 1} \circ \cdots \circ f_1)}{X} = \hatr{(f_i \circ
f_{i - 1} \circ \cdots \circ f_1)}{X}$
until we have computed
$\hatr{(f_n \circ f_{n-1} \circ \cdots \circ f_1)}{X} = \hatr{f}{X}$.

A number of properties follow immediately from these definitions, which we
summarize here and prove in full in
Appendices~\ref{app:no-fhat}---\ref{app:network-hat}.
First, we find a number of \emph{negative results}, describing functions (or
classes of functions) for which $\hat{f}$ and $\hatr{f}{X}$ cannot be computed.

\begin{restatable}{theorem}{ThmNoFHat}
    \label{thm:no-fhat}
    There exists a continuous, fully-differentiable function $f$ such that
    $\hat{f}$ does not exist.
\end{restatable}

\begin{restatable}{theorem}{ThmNoRestrictedFHat}
    \label{thm:no-restricted-fhat}
    There exists a continuous, fully-differentiable function $f$ such that
    $\hatr{f}{X}$ does not exist for any non-singleton and non-empty choice of
    $X$.
\end{restatable}

Both of these theorems can be understood by considering the function $f(x) =
x^2$. The key insight is that $f$ behaves non-linearly on \emph{its entire
input domain}; there are no two $x_1 < x_2$ such that $f(x + x') = f(x) +
f(x')$ and $f(cx) = cf(x)$ for all $x_1 \leq x, x' \leq x_2$.
However, as the next theorem shows, sometimes restricting the function to a
particular input domain can make the symbolic representation computable.

\begin{restatable}{theorem}{ThmOnlyRestrictedFHat}
    \label{thm:only-restricted-fhat}
    There exists a continuous, fully-differentiable function $f$ and a
    non-singleton, non-empty polytope $X$ such that $\hat{f}$ does not exist
    but $\hatr{f}{X}$ does.
\end{restatable}

An example for this theorem can be seen in the function $f(x_1, x_2) =
\sin^2{x_1} + \cos^2{x_2}$, which \emph{in general} behaves non-linearly, but
when restricted to input points where $x_1 = x_2$, it satisfies the affine
relation $f(x_1, x_2) = 1$. A similar example is the function $f(x_1, x_2) =
(x_1 + x_2)(x_1 - x_2)$, which is non-linear except when $x_1 - x_2$ equals a
constant.

Next, we have a number of positive results, showing that $\otimes$ both exists
and is computable for functions that are commonly composed to create neural
networks.

\begin{restatable}{theorem}{ThmPWLHat}
    \label{thm:PWL-Hat}
    For any piecewise-linear function $f$, $f \otimes \hat{g}$ is computable
    for any $\hat{g}$.
\end{restatable}

\begin{corollary}
    \label{cor:fullyconnected-hat}
    $\fullyconnected{} \otimes \hat{g}$ is computable for any $\hat{g}$.
\end{corollary}

\begin{corollary}
    \label{cor:conv-hat}
    $\conv{} \otimes \hat{g}$ is computable for any $\hat{g}$.
\end{corollary}

\begin{corollary}
    \label{cor:bnorm-hat}
    $\bnorm{} \otimes \hat{g}$ is computable for any $\hat{g}$.
\end{corollary}

\begin{corollary}
    \label{cor:relu-hat}
    $\relu{} \otimes \hat{g}$ is computable for any $\hat{g}$.
\end{corollary}

\begin{corollary}
    \label{cor:maxpool-hat}
    $\maxpool{} \otimes \hat{g}$ is computable for any $\hat{g}$.
\end{corollary}

\begin{restatable}{corollary}{CorNetworkHat}
    \label{cor:network-hat}
    $f \otimes \hat{g}$ is computable for any $\hat{g}$ and neural network $f$
    consisting of sequentially-applied \fullyconnected{}, \conv{}, \bnorm{},
    \relu{}, and \maxpool{} layers.
\end{restatable}

\section{Algorithms for Computing $\hatr{f}{X}$}

In this section, we discuss algorithms for computing $\hatr{f}{X}$ for
piecewise-linear functions (see~\pref{def:PWL}). Recall the primitive
$\extendname$, where $\extendname{}(h, \hat{g}) = h \otimes \hat{g} = \hat{h
\circ g}$. We showed in~\pref{sec:SymbolicRepresentation} that, as long as you
can compute $\extendname$ for each layer type in a network, you can compute
$\hat{f}$ for the entire network.  Thus, we focus in this section on algorithms
for computing $\extendname$ for common neural network layers.

In~\pref{sec:AlgDefinitions} we define a number of standard functions that we
will use in the algorithms. In~\pref{sec:AlgCaseStudy}, we describe the
\emph{rectified linear unit}, a common piecewise-linear function used in deep
neural networks.  In~\pref{sec:Algorithms-ReLU2D}, we present an efficient
algorithm for the case when the \restrname{} is two-dimensional.
In~\pref{sec:Algorithms-Arbitrary2D} we generalize the ReLU algorithm to
arbitrary piecewise-linear functions with two-dimensional \restrname{}.
Finally, in~\pref{sec:Algorithms-HvsV}, we discuss the benefits of the polytope
representation used by our algorithm.

\subsection{Definitions and Common Functions}
\label{sec:AlgDefinitions}
Here we define a number of functions which we will use in our algorithms:

\begin{enumerate}
    \item Given a (bounded) polytope $X$, \Vert{$X$} returns a list of its
        vertices in counter-clockwise order, repeating the initial vertex at
        the end.
    \item Given a set of points, \Hull{$X$} computes their convex hull (i.e.,
        smallest bounded polytope containing all points in $X$).
    \item Given a scalar value $x$, \Sign{$x$} computes the sign of that value
        (i.e., $-1$ if $x < 0$, $+1$ if $x > 0$, and $0$ if $x = 0$).
    \item Given a polytope $X$ which must lie in a single orthant,
        \OrthantSign{$X$} computes the per-coefficient sign corresponding to
        that orthant. 
        
        For example, if $X = \Hull(\{ (1, -1), (2, -3), (4, -2)
        \})$, then \OrthantSign{$X$} $= (1, -1)$.
    \item Given a polytope $X$ and affine map $A$, $A^\#(X)$ represents the
        polytope formed by applying $A$ to every point in $X$.
\end{enumerate}

\subsection{A case study in PWL Functions: The Rectified Linear Unit}
\label{sec:AlgCaseStudy}
The Rectified Linear Unit function, \relu{}, is one of the standard neural
network non-linearities used today. Recall from~\pref{sec:Overview-DNNs} that
it can be defined as a piecewise-linear function taking in vector $x = (x_1,
x_2, \ldots, x_n)$ and computing $\relu{}(x) = (\relu{}(x)_1, \ldots,
\relu{}(x)_n)$ where:
\[
    \relu{}(x)_i =
    \begin{cases}
        0 & x_i \leq 0 \\
        x_i & x_i > 0
    \end{cases}
\]

Geometrically, the \relu{} function can be thought of as projecting points onto
the faces of the positive orthant by zeroing-out negative coefficients.
Notably, if a set of points all lie in the same orthant (i.e., all have the
same sign), then applying \relu{} to all of them corresponds to applying a
single \emph{affine projection} to all of them (namely, the identity map with
non-positive entries on the diagonal zeroed out). Thus, intuitively the
computation can be broken down into multiple affine projections, one for the
portion of the polytope in each orthant.

\subsection{Input-Aware, Rectified Linear Units in Two Dimensions}
\label{sec:Algorithms-ReLU2D}

We now present an efficient algorithm for computing $\relu{} \otimes \hat{g}$
when $\hat{g}$ involves only bounded, two-dimensional polytopes. This algorithm
strives to display \emph{both} efficient best- and worst-case execution time.
The key insight is that, when a polytope is all in a single orthant, the
application of \relu{} is an affine function. We read each polytope as a
counter-clockwise set of vertices, following around the edges between vertices
until the edge of the polytope intersects an orthant boundary. At that point,
we split the polytope in two, such that each half lies on only one side of the
orthant boundary. Repeating this process recursively ensures that the resulting
polytopes lie on only one side of \emph{all} orthant boundaries, i.e. each lies
completely in a single orthant.

\begin{figure}[t]
{\small
\begin{algorithm}[H]
    \DontPrintSemicolon
    \KwIn{$V$, the vertices of the polytope in the input space of $g$. $F$, the
    affine map that corresponds to $g$ within $\Hull(V)$. $i$ is the index of
    the last vertex lying on the same side of the orthant face as $V_1$. $j$ is
    the index of the last vertex lying on the opposite side of the orthant face
    as $V_1$. $d$ is the index of the orthant face, i.e. the face is $x_d =
    0$.}
    \KwOut{$\{ P_1, P_2 \}$, two sets of vertices whose convex hulls form a
    partitioning of $V$ such that each lies on only one side of the $x_d = 0$
    hyperplane.}
    $p_i \gets V_i - \frac{F(V_i)_d}{F(V_{i+1})_d - F(V_{i})_d}(V_{i+1} - V_i)$\;
    $p_j \gets V_j - \frac{F(V_j)_d}{F(V_{j+1})_d - F(V_{j})_d}(V_{j+1} - V_j)$\;
    $A \gets \{ p_i, p_j \} \cup \{ v \in V \mid \Sign(v_d) = \Sign({V_i}_d) \}$\;
    $B \gets \{ p_i, p_j \} \cup \{ v \in V \mid \Sign(v_d) = \Sign({V_j}_d) \}$\;
    \returnKw{$\{ A, B \}$}
    \caption{SplitPlane$(V, F, i, j, d)$}
    \label{alg:split-plane}
\end{algorithm}
}
\end{figure}
\begin{figure}[t]
{\small
\begin{algorithm}[H]
    \DontPrintSemicolon
    \KwIn{$\hat{g} = \{ (P_1, F_1), \ldots, (P_n, F_n) \}$.}
    \KwOut{$\hatr{\relu{} \circ g}{X}$}
    $W \gets$ \Queue{$\hatr{g}{X}$}\;
    $Y \gets \emptyset$\;
    \While{$W$ not empty}{
        $P, F \gets$ \Pop{$W$}\;
        $V \gets \Vert{$P$}$\;
        $K \gets \{ k \mid \exists i, j: \Sign(F(V_i)_k) > 0 \wedge \Sign(F(V_j)_k) < 0 \}$\;
        \If{$K = \emptyset$}{
            $Y \gets Y \cup \{ (P, \relu{}[\OrthantSign(F^\#(P))] \circ F) \}$\;
            \continueKw{}
        }
        $k \gets$ any element from $K$\;
        $i \gets \argmax_i \{ \Sign(F(V_i)_d) = \Sign(F(V_1)_d) \}$\;
        $j \gets \argmax_j \{ \Sign(F(V_j)_d) \neq \Sign(F(V_i)_d) \}$\;
        \For{$V' \in \SplitPlane(V, F, i, j, d)$}{
            $W \gets$ \Push{$W, (\Hull(V'), F)$}\;
        }
    }
    \returnKw{$Y$}
    \caption{$\relu{} \otimes \hat{g}$ for two-dimensional $\hat{g}$
    partitions.}
    \label{alg:relu-2d-extend}
\end{algorithm}
}
\end{figure}

Notably, in the best-case scenario where each partition is in a single orthant,
the algorithm never calls \SplitPlane{} at all --- it simply iterates over all
of the $n$ input partitions, checks their $v$ vertices, and appends to the
resulting set (for a best-case complexity of $O(nv)$). In the worst case, it
splits each polytope in the queue on each face, resulting in exponential time
complexity. In practice, however, we find that this algorithm is very efficient
and can be applied to real-world networks effectively
(see~\pref{sec:Evaluation}).

The proofs of the following theorems can be found in
Appendices~\ref{app:split-plane}---\ref{app:relu-extend}:
\begin{restatable}{theorem}{ThmSplitPlane}
    \label{thm:split-plane}
    \pref{alg:split-plane} correctly splits a 2D polytope $\Hull(V)$ by the
    hyperplane $x_d = 0$.
\end{restatable}

\begin{restatable}{theorem}{ThmReluExtend}
    \label{thm:relu-extend}
    \pref{alg:relu-2d-extend} correctly computes $\hatr{\relu{} \circ g}{X}$.
\end{restatable}

\subsection{Arbitrary Piecewise-Linear Function in Two Dimensions}
\label{sec:Algorithms-Arbitrary2D}
We can generalize~\pref{alg:relu-2d-extend} to arbitrary piecewise-linear
functions $f$. We let the domain of $f$ be partitioned by a finite number of
polytopes stored in H-representation (i.e., as a conjunction of half-spaces).
We then take the set of all hyperplanes defining the partitioning polytopes,
defined by affine maps $\{ A_1, A_2, \ldots, A_m \}$. The insight is now
essentially the same: subdividing the input polytope such that each sub-divided
polytope lies entirely on one side of \emph{each} hyperplane $A_i$ ensures that
it lies entirely within a particular linear region of $f$. Effectively, the
only difference between these algorithms and those specialized to \relu{} is
that we replace the particular affine map $v_d$ (i.e., projection onto the $d$
dimension) with ones specified by the function $f$ from the set of $A_i$s.

\subsection{Representing Polytopes}
\label{sec:Algorithms-HvsV}

Finally, we close this section with a discussion of implementation concerns
when representing the convex polytopes that make up the partitioning of
$\hatr{f}{X}$. In standard computational geometry, bounded polytopes can be
represented in two equivalent forms:

\begin{enumerate}
    \item The \emph{half-space} or \emph{H-representation}, which encodes the
        polytope as an intersection of finitely-many half-spaces. (Each
        half-space being defined as one side of a hyperplane, which can in turn
        be defined by an affine map $Ax \leq 0$.)
    \item The \emph{vertex} or \emph{V-representation}, which encodes the
        polytope as a set of finitely many points; the polytope is then taken
        to be the convex hull of the points (i.e., smallest convex shape
        containing all of the points.)
\end{enumerate}

However, choosing a particular representation can make certain problems
significantly easier or more challenging. Finding the intersection of two
polytopes in an H-representation, for example, can be done in linear time by
simply concatenating their representative half-spaces, but the same is not
possible in V-representation.

In our algorithms, there are two main operations we need to do with polytopes
in our algorithms: splitting a polytope with a hyperplane and applying an
affine map to all points in the polytope. In general, the first is more
efficient in an H-representation, while the latter is more efficient in a
V-representation. However, when restricted to two-dimensional polygons, the
former is also efficient in a V-representation, as demonstrated
by~\pref{alg:split-plane}, helping to motivate our use of the V-representation
in our algorithms.

Furthermore, the representations differ as to their resiliency to
floating-point operations. In particular, H-representations for polytopes in
$\mathbb{R}^n$ are notoriously difficult to achieve high-precision with, as the
error introduced from using floating point numbers gets arbitrarily large as
one goes in a particular direction along any hyperplane face.  Ideally, we
would like the hyperplane to be most accurate in the region of the polytope
itself, which corresponds to choosing the magnitude of the norm vector
correctly. Unfortunately, to our knowledge, there is no efficient algorithm for
computing the ideal floating point H-representation of a polytope, although
libraries such as APRON~\cite{apronlib} are able to provide reasonable results
for low-dimensional spaces.  However, as neural networks utilize extremely
high-dimensional spaces (i.e., thousands of dimensions) and we wish to
iteratively apply our analysis, we find that errors from using floating-point
H-representations can quickly multiply and compound to become infeasible. By
contrast, floating-point inaccuracies in a V-representation are directly
interpretable as slightly misplacing the vertices of the polytope; no
``localization'' process is necessary to penalize inaccuracies close to the
polytope more than those far away from it.

Another difference is in the space complexity of the representation. In
general, H-representations can be more space-efficient for common shapes than
V-representations, \emph{however}, when the polytope lies in a low-dimensional
subspace of a larger space, the V-representation is usually significantly more
efficient.

Thus, V-representations are a good choice for low-dimensionality polytopes
embedded in high-dimensional space, which is exactly what we need for analyzing
neural networks with two-dimensional \restrsname{}. This is why we designed our
algorithms to rely on $\Vert(X)$, so that they could be directly computed on a
V-representation.

\section{Understanding network behavior using weakest precondition}
\label{sec:WeakestPrecondition}

In this section, we investigate one of the most immediate questions one might
ask when analyzing a neural network: what inputs lead to a particular set of
outputs?

\subsection{The $\prer{f}{X}{Y}$ Primitive}
\subsubsection{Preconditions}
This problem is formalized by the notion of \emph{preconditions}; subsets of
the input space which, when $f$ is applied, map to a particular set of outputs.
We now introduce a primitive that describes the notion of such preconditions
for a neural network $f : A \to B$ and output polytope $Y \subseteq B$:
$    \prer{f}{X}{Y} \subseteq \{ x \in X \mid f(x) \in Y \}$.
$\prer{f}{X}{Y}$ \emph{always} satisfies the following properties:
\begin{enumerate}
    \item $\prer{f}{X}{Y} \subseteq X$
    \item For any $x \in \prer{f}{X}{Y}$, $f(x) \in Y$ holds.
\end{enumerate}

\subsubsection{Weakest Precondition}
The definition of $\prer{f}{X}{Y}$ only requires an \emph{under approximation};
there may be points $x \in X$ such that $f(x) \in Y$, even if $x \not\in
\prer{f}{X}{Y}$. This leads to the fact that $\emptyset$, which we refer to as
the \emph{strongest precondition}, always satisfies $\prer{f}{X}{Y} =
\emptyset$. However, such a solution is not particularly useful, and instead we
would like to find the \emph{weakest precondition} $    \weakprer{f}{X}{Y} = \{ x
\in X \mid f(x) \in Y \}$.

For \emph{arbitrary functions and regions}, $\weakprer{f}{X}{Y}$ may
not be convex or even representable as a union of convex shapes.  However, we
will show by construction in~\pref{sec:WeakestPrecondition-WithHat} that, when
$\hatr{f}{X}$ exists and $Y$ is a convex polytope, $\weakprer{f}{X}{Y}$ can be
represented precisely as \emph{a union of finitely many convex polytopes}.

\subsection{DPPre: Computing Preconditions with DeepPoly}
\label{sec:DPPre}
We first investigate one way in which a solution to $\prer{f}{X}{Y}$ can be
found using prior work in abstract interpretation.
\emph{DeepPoly} \cite{Singh:POPL2019} is an abstract representation of the
neural network $\restr{f}{X}$ included as part of the ERAN abstract
interpretation package~\cite{ERAN}. Namely, for any network $f$
restricted to restriction domain of interest $X$ (i.e., $\restr{f}{X}$) and any
particular output dimension $i$, DeepPoly can produce lower- and upper-bound
affine maps $A_i^l$ and $A_i^u$ such that, for all $x \in X$: $    A_i^l x \leq
\restr{f_i}{X}(x) \leq A_i^u x$. Thus, as long as the set $Y$ is convex, it
follows that $\{ x \in X \mid A_i^l x \in Y \wedge A_i^u x
\in Y \}$ is a valid $\prer{f}{X}{Y}$. This set can be computed explicitly from
$A_i^l$ and $A_i^u$ using a variety of linear algebra techniques, and a
geometric interpretation of this process is discussed below.

One can think of the DeepPoly representation as defining two
polytopes, each one a \emph{two-vocabulary polytope} in a space consisting of
both input and output dimensions of the network. The first polytope defines the
lower bounds on the network output:
$    D_l = \{ (x, y) \mid x \in X \wedge y = A^l(x) \}$
and the second defines the upper bounds on the network output:
$    D_u = \{ (x, y) \mid x \in X \wedge y = A^u(x) \}$, where $A^l(x) =
(A^l_1(x), \ldots, A^l_n(x))$, i.e. lower bounds for all output dimensions (and
similar for $A^u(x)$).
Computing $\prer{f}{X}{Y}$ can now be seen as a three step process:

\begin{enumerate}
    \item Compute the intersection of $D_l$ with $Y$ and project onto only the
        input $x$ dimensions, resulting in a new polytope $P_l$. $P_l$ contains
        only input points for which we can guarantee the image under $f$ is
        lower-bounded by some value within $Y$.
    \item Compute the intersection of $D_u$ with $Y$ and project onto only the
        input $x$ dimensions, resulting in a new polytope $P_u$. $P_u$ contains
        only input points for which we can guarantee the image under $f$ is
        upper-bounded by some value within $Y$.
    \item Compute the intersection of $P_l$ and $P_u$, representing inputs
        which have lower \emph{and} upper bounds in $Y$, thus (by convexity)
        the actual output must lie in $Y$.
\end{enumerate}

We refer to this approach as $\DPPre$ or $\DPPre[1]$. Notably, this
intersection-projection construction implies that $\prer{f}{X}{Y}$ found using
DPPre can \emph{only} represent convex shapes, thus it in general cannot find
$\weakprer{f}{X}{Y}$ which may be non-convex.  This is usually fine for
local-behavior verification problems where the network acts mostly convex, but
is in practice unsuited for computing precise $\prer{f}{X}{Y}$ sets when
$\weakprer{f}{X}{Y}$ is highly non-convex. In practice, the precision of
$\prer{f}{X}{Y}$ computed with DeepPoly can be improved by pre-partitioning the
input space $X$ into smaller regions and passing each region to the analyzer
separately. We notate this approach as DPPre[k], defined by:

\begin{definition}
    Suppose $X$ is a bounded \restrname{} spanned by $d$ basis vectors. Then
    DPPre[k]($\restr{f}{X}, Y$) is the application of the DPPre process after
    evenly splitting $X$ in each of the $d$ basis directions into $k$ equal
    partitions.
    \[ \DPPre[k](\restr{f}{X}, Y) \eqdef
        \bigcup_{i_1 = 1}^{k} \bigcup_{i_2 = 1}^{k}\cdots \bigcup_{i_d = 1}^{k} \DPPre[1](\restr{f}{Part(X, k, i_1, i_2,\ldots, i_d)}, Y) \]
    Where $Part(X, k, i_1, \cdots, i_d)$ is the $(i_1, \cdots, i_d)$th
    partition of $X$ when $X$ is split by $k$ equal partitions along each
    basis. For example, if $X = \{ (x, y) \mid 0 \leq x \leq 1 \wedge 0 \leq y
    \leq 1 \}$ and we use the standard basis $\{ (1, 0), (0, 1) \}$, then
    $Part(X, 4, 1, 1) = \{ (x, y) \mid 0 \leq x \leq \frac{1}{4} \wedge 0 \leq
    y \leq \frac{1}{4} \}$.
\end{definition}

As we will see in~\pref{sec:Evaluation-Precondition}, even when using $k$ as
large as $100$ (i.e., $10,000$ partitions for a two-dimensional region) the
precision is still lacking. Furthermore, performance suffers as the number of
DPPre[1] calls grows according to $k^d$.

\subsection{Computing Weakest Preconditions with $\hatr{f}{X}$}
\label{sec:WeakestPrecondition-WithHat}

The \emph{exact} $\weakprer{f}{X}{Y}$ can be computed given $\hatr{f}{X} = \{
(P_1, F_1), \ldots, (P_n, F_n) \}$. Effectively, within each $(P_i, F_i)$ and
for all $x \in P_i$ and output dimension $j$, we have: $    \restr{f}{X}(x)_j =
(F_i x)_j$, which is similar to the form returned by DeepPoly except the bounds
are \emph{exact} --- there is no over-approximation at all. Then, we can
intersect the corresponding polytopes to compute $\weakprer{{F_i}}{{P_i}}{Y}$
on each $(P_i, F_i)$ pair; the union of all such regions is then
$\weakprer{f}{X}{Y}$.

\subsection{Visualizing Decision Boundaries with $\prer{f}{X}{Y}$}

Many neural networks are \emph{classifiers}, meaning they take some input and
produce one of a small (finite) set of outputs. An example of such a network is
the ACAS Xu aircraft avoidance network~\cite{julian2018deep}, which has five
inputs describing the position and velocity of an ``ownship'' and ``intruder,''
transformed through five ``hidden layers'' with 50 dimensions each, then
produces five real-valued outputs $y_1, \ldots, y_5$. The network's output
space is to be interpreted as having five \emph{classification regions},
partitioning the output space depending on which output dimension is
maximal.\footnote{The original ACAS Xu networks take the minimal output
dimension; this can be shown equivalent by inverting the weights and biases of
the final layer.} For example, the network may be said to advise a ``strong
right'' turn when $f(x) \in R_5$, with $R_5$ (``the fifth classification
region'') defined as $    R_5 = \{ y \mid y_5 > y_1 \wedge \cdots \wedge y_5 >
y_4 \}$.

Suppose we want to \emph{visualize the policy learned by the network.} We could
fix all inputs (eg. the velocity and heading) other than the position of the
attacking ship, resulting in a restriction domain of interest denoted $X$.
Then, we could compute $\weakprer{f}{X}{R_5}$ to get the set of all such input
positions for which the network advises to make a ``strong right,'' and use
standard computer graphics tools to plot these points on a graph in a
particular color (say, orange). We may then repeat the process, until we have
plotted $\weakprer{f}{X}{R_1}$ through $\weakprer{f}{X}{R_5}$ on the same plot
in separate colors. The process is shown in~\pref{sec:Evaluation-Precondition},
and the resulting plot can provide precise and immediate insight into the
behavior of an ownship controlled by the network.

\section{Bounded model checking of safety properties using strongest postcondition}
\label{sec:StrongestPostcondition}

We saw in the preceding section how $\hatr{f}{X}$ can be used to compute the
\emph{weakest precondition} of a network's input given conditions on the output,
then described how that primitive could be applied to visualizing decision
boundaries of a neural network. In this section, we consider a complementary
primitive called the \emph{strongest postcondition}, and show how it can be used
to perform bounded model checking~\cite{biere2009bounded} of neural-network
based controllers.

\subsection{The $\post{f}{X}$ Primitive}
We define $\post{f}{X}$ to be a \emph{set containing all output points that can
be mapped to under $f$ given some input in $X$}. Formally, for a neural network
$f : A \to B$ and input polytope $X \subseteq A$, we have
$    \{ f(x) \mid x \in X \} \subseteq \post{f}{X} \subseteq B$.
$\post{f}{X}$ \emph{always} satisfies the following properties:
\begin{enumerate}
    \item $\post{f}{X} \subseteq B$
    \item For any $x \in X$, $f(x) \in \post{f}{X}$.
\end{enumerate}

\subsubsection{Strongest Postcondition}
\label{sec:StrongestPostSub}
Notably, this definition only requires an \emph{over approximation}; there may
be points $y \in \post{f}{X}$ such that there \emph{is not} any $x \in X$
satisfying $f(x) = y$. This leads to the fact that $B$ (the range of $f$),
which we refer to as the \emph{weakest postcondition}, always satisfies
$\post{f}{X} = B$.  However, such a solution is not particularly useful, and
instead in general we would like to find the \emph{strongest postcondition},
denoted $\strongpost{f}{X}$, which exactly satisfies
    $\strongpost{f}{X} = \{ f(x) \mid x \in X \}$.

For \emph{arbitrary functions and regions}, $\strongpost{f}{X}$ may
not be convex or even representable as a union of convex shapes.  However, we
will show by construction in~\pref{sec:StrongestPost-WithHat} that, when
$\hatr{f}{X}$ exists and $X$ is a convex polytope, $\strongpost{f}{X}$ can be
represented precisely as \emph{a union of finitely many convex polytopes}.

\subsection{Computing $\strongpost{f}{X}$ with $\hatr{f}{X}$}
\label{sec:StrongestPost-WithHat}
Suppose we wish to compute $\strongpost{f}{X}$, and know that $\hatr{f}{X} = \{
    (P_1, F_1), $ $\ldots,$ $(P_n, F_n) \}$. We first note that
$\strongpost{f}{X} = \bigcup_{(P_i, F_i) \in \hatr{f}{X}} \strongpost{f}{P_i} =
\bigcup_{(P_i, F_i) \in \hatr{f}{X}} \strongpost{F_i}{P_i}$.  The first
equality holds because the $P_i$s partition $X$ and the second equality holds
because, given any $x \in P_i$, $f(x) = F_i(x)$.

Thus, it suffices to compute $\strongpost{F}{P}$ for a polytope $P$ and affine
function $F$. We note that, in computational geometry field, this corresponds
to \emph{transforming polytope $P$ under affine map $F$}, which in turn can be
shown to correspond to transforming the vertices of $P$ and then taking the
convex hull of the resulting vertices (due to the convexity of $P$ and $F$). In
other words, we have
$ \strongpost{F}{P} = \{ F(x) \mid x \in P \} = \Hull(\{ F(v) \mid v \in \Vert(P) \})$,
giving us, in total
$    \strongpost{f}{X} = \bigcup_{(P_i, F_i) \in \hatr{f}{X}} \Hull(\{ F_i(v)
\mid v \in \Vert(P_i) \})$.

\subsection{Inverted Pendulum Model}
In this section, we will consider a model taken
from~\citet{DBLP:conf/pldi/ZhuXMJ19}, which describes an \emph{inverted
pendulum} with a motor at the base which can apply angular acceleration to the
pendulum. The motor is controlled by a neural network, which has been trained to
keep the pendulum upright. The network takes as input two variables representing
the state of the system, namely the current position of the pendulum and the
current angular velocity. It then produces one output, namely the angular
acceleration it wishes to apply. We call this network an \emph{actor} and
denote it $f$. Notably, we replaced $\tanh$ non-linearities with their
piecewise-linear counterpart the \emph{hard tanh} $ g(x) = \mathrm{min}(\max(x,
-1), 1)$~\cite{collobert2004large}.

In the model used for training and verifying the network, the effect of
applying a particular acceleration $a$ at time $t$ to state $x_t$ to form new
state $x_{t+1}$ is modeled with an \emph{affine function} that takes as input
the previous state $x_t$ and the action $a$ and produces a new state $x_{t+1}$.
We call this function an \emph{environment model} and denote it $E$.

We can compose $E$ with $f$ to produce a new ``transition function''
$T(x_t) = E(x_t, f(x_t)) = x_{t+1}$, which accepts as input a state $x_t$ and
produces as output the corresponding state of the system after applying the
acceleration prescribed by the network for one timestep. Because $f$ is
piecewise-linear and $E$ is affine, it follows that $T$ itself is
piecewise-linear. With this notation, we are now ready to describe a
\emph{correctness specification} for the network.

\begin{definition}
    Given a function $f : A \to B$, we define a \emph{correctness
    specification} $C$ to be a set of tuples $\{ (X_1, Y_1), \ldots, (X_n, Y_n)
    \}$ where each $X_i \subseteq A$ and each $Y_i \subseteq B$.

    We say \emph{$f$ meets specification $C$} if, for all $(X_i, Y_i)$ pairs
    and all $x \in X_i$, $f(x) \in Y_i$ is satisfied.

    In contrast, we say \emph{$f$ violates specification $C$} if there exists
    some $(X_i, Y_i)$ pair and some $x \in X_i$ such that $f(x) \not\in Y_i$.
\end{definition}

In the pendulum example, we will consider three sets of states:
\begin{enumerate}
    \item The set of \emph{initial states} $S_I = \{ (\nu, \omega) \mid -0.35
        \leq \nu \leq 0.35 \wedge -0.35 \leq \omega \leq 0.35 \}$ is the set of
        states the pendulum can be in before applying the network.
    \item The set of \emph{safe states} $S_S = \{ (\nu, \omega) \mid -0.5 \leq
        \nu \leq 0.5 \wedge -0.5 \leq \omega \leq 0.5 \}$ is the set of states
        for which we say the network has succeeded, namely when the pendulum is
        upright and not moving at an unsafe speed.
    \item The set of \emph{unsafe states} $S_U = \{ (\nu, \omega) \mid (\nu,
        \omega) \not\in S_S \}$ is the set of states for which we say the
        network has failed, namely when the pendulum dips below the horizontal
        or moves at an unsafe speed.
\end{enumerate}

We would now like to solve the following \emph{bounded model checking} problem:
\emph{Is there any initial state such that, after applying the neural
controller for $K$ time steps, the pendulum has ever reached an unsafe
state?}
This corresponds to verifying the specification:
\begin{equation}
    \label{eq:BMCSpecification}
    \big\{ (S_I, S_S), (\strongpost{T}{S_I}, S_S), (\strongpost{T}{S_I}^2, S_S),
    \ldots, (\strongpost{T}{S_I}^{K - 1}, S_S) \big\}
\end{equation}
where $\strongpost{T}{S_I}^n$ indicates the repeated application of
$\strongpost{T}{\cdot}$; for instance, $\strongpost{T}{S_I}^2$ $=$
$\strongpost{T}{\strongpost{T}{S_I}}$.

\subsection{Bounded Model Checking with DeepPoly}

One could attempt to use DeepPoly \cite{Singh:POPL2019} here as well, however,
as DeepPoly uses an \emph{over-approximation} for $\post{f}{x}$, we would expect
to run into many false-positives (where DeepPoly reports that the specification
cannot be verified, but is unable to state whether that is due to imprecision in
the abstraction used or because the specification is violated).

\subsection{Bounded Model Checking with ReluPlex}

Another approach we could take is to use the DNN-oriented SMT solver
ReluPlex~\cite{reluplex:CAV2017}. Because $T$ is piecewise-linear, ReluPlex can
handle queries such as: $\exists x \in S_I: T(x) \in S_U$ (with the caveat that,
as $S_U$ is non-convex, it will have to be partitioned into convex polytopes
before ReluPlex can be queried). Notably, however, ReluPlex does not support
universal quantifiers on iterated applications of $T$, i.e. there is no
meaningful way of encoding $\forall k \in \{ 1, 2, \ldots, K \}: \exists x \in
S_I: T^k(x) \in S_U \}$, where $T^k(x)$ is the repeated application of $T$ to
$x$ $k$ times, which is the actual query we wish to run. Instead, we need to
query ReluPlex \emph{separately} for each timestep, forcing it to repeat work
done on verifying properties of the behavior of the network on earlier
timesteps at each new timestep.

\subsection{Bounded Model Checking with $\hatr{f}{X}$ and $\strongpost{f}{X}$}

Alternatively, one can perform the bounded model checking iteratively using
$\strongpost{T}{X}$, computed as discussed in~\pref{sec:StrongestPost-WithHat}.
This works as follows:

Initially, we compute $\strongpost{T}{S_I}$ and check if $\strongpost{T}{S_I}
\cap S_U = \emptyset$. If the intersection is \emph{not} empty, then by
definition there must be some state in $S_I$ that transitions to an unsafe
state (in $S_U$) after applying the network for one timestep and we can report
that the model has violated the specification. Next, we can use the computed
$\strongpost{T}{S_I}$ and compute $   \strongpost{T}{\strongpost{T}{S_I}} =
\strongpost{T}{S_I}^2 $, which we can then check to see if any intersection
with $S_U$ is found. We can repeat this iteratively, effectively ``growing''
the ``reachable set'' of our model at each step. If, upon reaching
$\strongpost{T}{S_I}^{K}$, no intersection with $S_U$ has ever been identified,
we may safely say that the model satisfies the specification in
\pref{eq:BMCSpecification}.  Importantly, we can iteratively reuse the
computation of $\strongpost{T}{S_I}^j$ for computing
$\strongpost{T}{S_I}^{j+1}$, as opposed to the SMT-solver approach which does
not share information between verification of different step lengths.

There are multiple further optimizations that can be performed in this approach.
For example, $\hatr{f}{S_S}$ can be computed and stored once ahead-of-time, so
$\strongpost{T}{S_I}^j$ can be immediately computed without having to
repetitively re-compute our symbolic representation of $f$. Furthermore, if
something in one step has already been included in the post-set of a previous
step, it can be safely ignored when propagating to the next step. For an
extreme example, if $\strongpost{T}{S_I} \subseteq S_I$ (and $S_I \cap S_U =
\emptyset$) we could immediately verify the network for all $K$ timesteps (in
fact, \emph{all timesteps in general}) because it satisfies the inductive
invariant $\strongpost{T}{S_I}^j \subseteq S_I \wedge S_I \cap S_U = \emptyset$
for all $j$.

\section{Patching Deep Neural Networks}
\label{sec:PatchingNetworks}

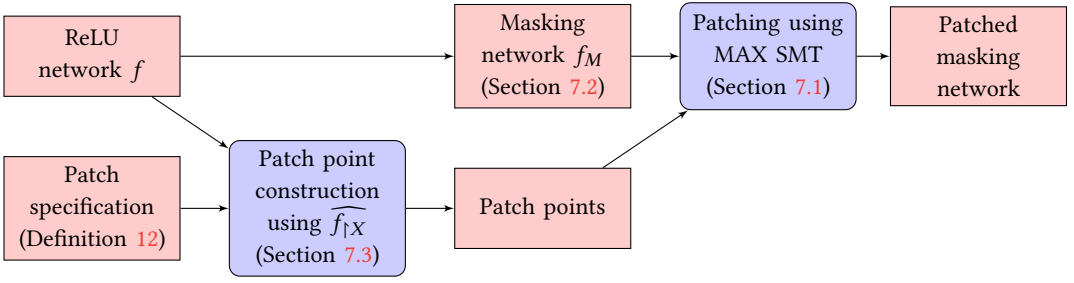
\begin{figure}[t]
    \centering

    \tikzstyle{processblock} = [rectangle, draw, fill=blue!20, 
    text width=6em, text centered, rounded corners, minimum height=4em]
\tikzstyle{datablock} = [rectangle, draw, fill=red!20, 
    text width=6em, text centered, minimum height=3em]
\tikzstyle{line} = [draw, -latex']

\begin{tikzpicture}[node distance = 2cm, auto]
  \node [datablock] (relu) {\small ReLU network $f$};
  \node [datablock, below of=relu] (patchspec) {\small Patch specification (\pref{def:PatchSpecification})};

  \node [processblock, right of=patchspec, node distance=3cm] (pointconstructor) {\small Patch point
  construction using $\hatr{f}{X}$ (\pref{sec:NetPatch-Infinite})};

  \node [datablock, right of=pointconstructor, node distance=3cm] (patchpoints) {\small Patch points};

  \node [datablock, right of=relu, node distance=6cm] (masking) {\small Masking network $f_M$ (\pref{sec:MaskingNetworks})};

  \node [processblock, right of=masking, node distance=3cm] (patchalgo) {\small Patching using MAX SMT (\pref{sec:MAXSMT})};

  \node [datablock, right of=patchalgo, node distance=2.8cm] (patchedmasking) {\small Patched masking network};

  \path [line] (relu) -- (pointconstructor);
  \path [line] (relu) -- (masking);
  \path [line] (patchspec) -- (pointconstructor);
  \path [line] (pointconstructor) -- (patchpoints);
  \path [line] (patchpoints) -- (patchalgo);
  \path [line] (masking) -- (patchalgo);
  \path [line] (patchalgo) -- (patchedmasking);
\end{tikzpicture}
    \caption{Patching deep neural networks}
    \label{fig:patching-overview}
\end{figure}

In this section, we introduce and formalize the problem of \emph{network
patching}, then show how our symbolic representation of a DNN can lead to a
solution to the network patching problem (\pref{fig:patching-overview}). 
As a running example, consider the
expository neural network defined by function $f$ below:
\begin{equation} 
    \label{eq:NetPatch-network}
    f(x) =
    \begin{bmatrix}
        1 & 1 & 1 \\
        0 & -1 & -1
    \end{bmatrix}
    \relu{}\left(
    \begin{bmatrix}
        -1 & 1 \\
        1 & 0 \\
        0 & 1
    \end{bmatrix}
    \begin{bmatrix}
        x_1 \\
        x_2
    \end{bmatrix}
    +
    \begin{bmatrix}
        -0.5 \\ 0 \\ 0
    \end{bmatrix}
    \right)
    +
    \begin{bmatrix}
        0 \\ 1
    \end{bmatrix}
\end{equation}

This network is visualized in~\pref{fig:patch-pre}, with classification regions
$R_1 = \{ y_1 \geq y_2 \}$ and $R_2 = \{ y_2 \geq y_1 \}$. Now, suppose for
expository purposes that we want this network's decision boundaries to look
like that of~\pref{fig:patch-post}, and suppose one wanted to ``patch'' the
network, i.e. make small modifications to its weights, such that this is
accomplished.
Before approaching this problem, we must formalize a representation of the
desired network behavior:
\begin{definition}
    \label{def:PatchSpecification}
    Given a neural network $f : A \to B$, a \emph{patch specification} is a
    finite set of \emph{pairs of convex polytopes} $T = \{ (X_0, Y_0), \ldots,
    (X_n, Y_n) \}$ where each $X_i \in A$ and $Y_i \in B$.
\end{definition}
We can then formalize the concept of a network patch like so:
\begin{definition}
    \label{def:Patch}
    Given a neural network $f(x; \theta)$ parameterized by weights $\theta$ and
    a patch specification $T = \{ (X_0, Y_0), \ldots, (X_n, Y_n) \}$, we define
    a \emph{patch to $f$ satisfying $T$} to be a change $\delta$ to the weights
    $\theta$ such that, for all $x \in X_i$, $f(x; \theta + \delta)$ satisfies
    $f(x; \theta + \delta) \in Y_i$.
\end{definition}

In our running example, we have specified four polytopes
in~\pref{fig:patch-spec}, which can be used to build the patch specification:
$\{ (P_1, R_2), (P_2, R_1), (P_3, R_1), (P_4, R_1) \}$.  Note that the
polytopes in patch specifications are usually \emph{infinite}, meaning the
quantifications in~\pref{def:Patch} are over infinite sets. 

One approach would be to treat this as a normal network training problem, which
we can try to find a solution using gradient descent~\cite{patterson2017deep}.
However, because gradient descent works over finitely-many individual points,
even if the training loss goes to $0$ it is not guaranteed that the actual
constraints (which quantify over infinite sets of points) are all met.
This issue could be addressed using $\hatr{f}{X}$ to implement a train-verify
loop, where we train with gradient descent, then use $\prer{f}{X}{Y}$ to find
points that do not satisfy the patch set, which can then be placed back into the
training set for the next iteration of gradient descent. However, this iterative
process is unfortunately inefficient and in general not guaranteed to terminate.
In the proceeding sections, we will address these issues by formulating the
patching problem as a MAX-SMT instance.

\begin{figure}[t]
    \centering
    \begin{subfigure}[t]{.31\linewidth}
        \includegraphics[width=\linewidth]{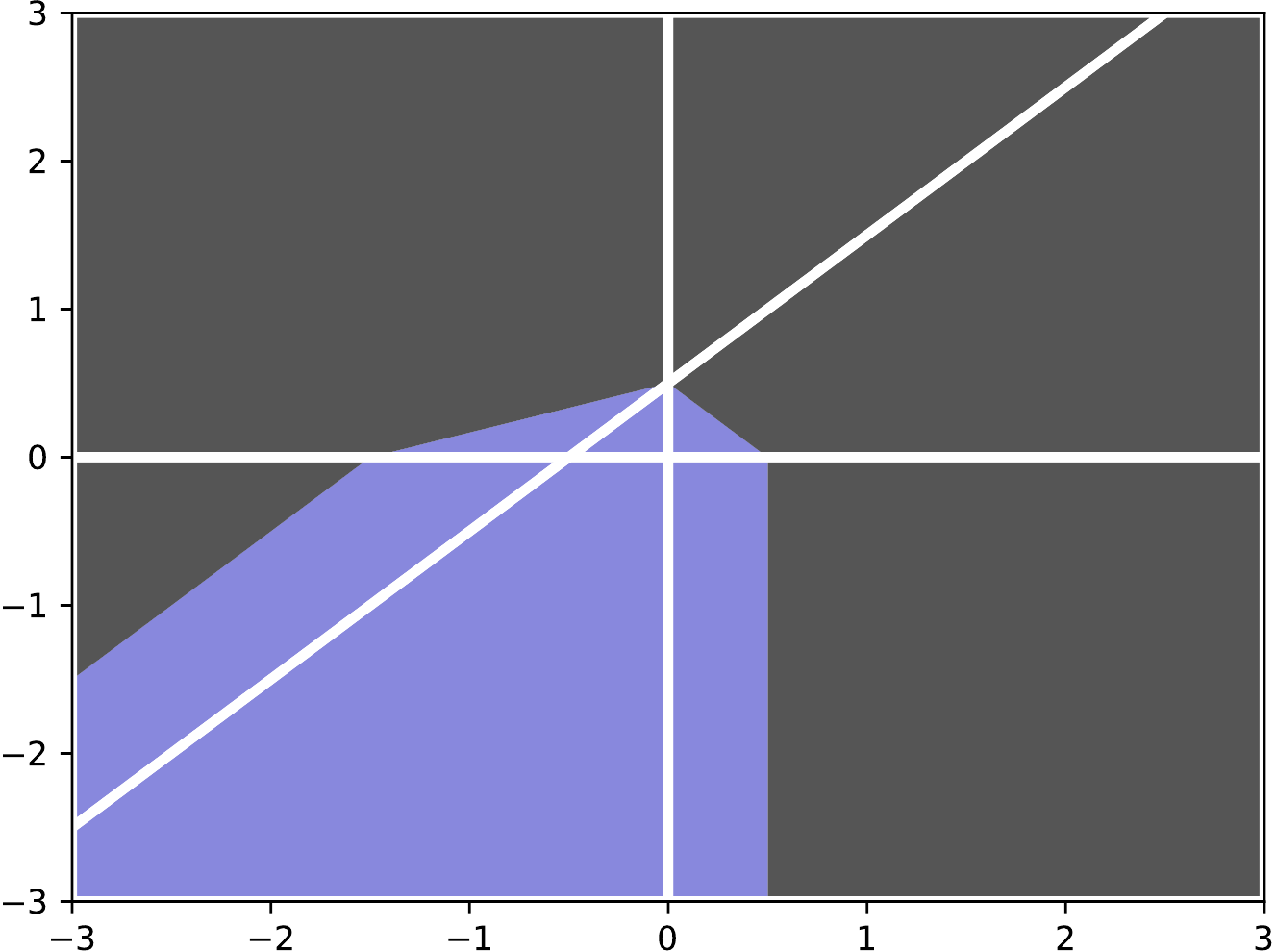}
        \caption{The network's linear partitions and decision boundaries before
        patching.}
        \label{fig:patch-pre}
    \end{subfigure}
\hspace{.5em}
    \begin{subfigure}[t]{.31\linewidth}
        \includegraphics[width=\linewidth]{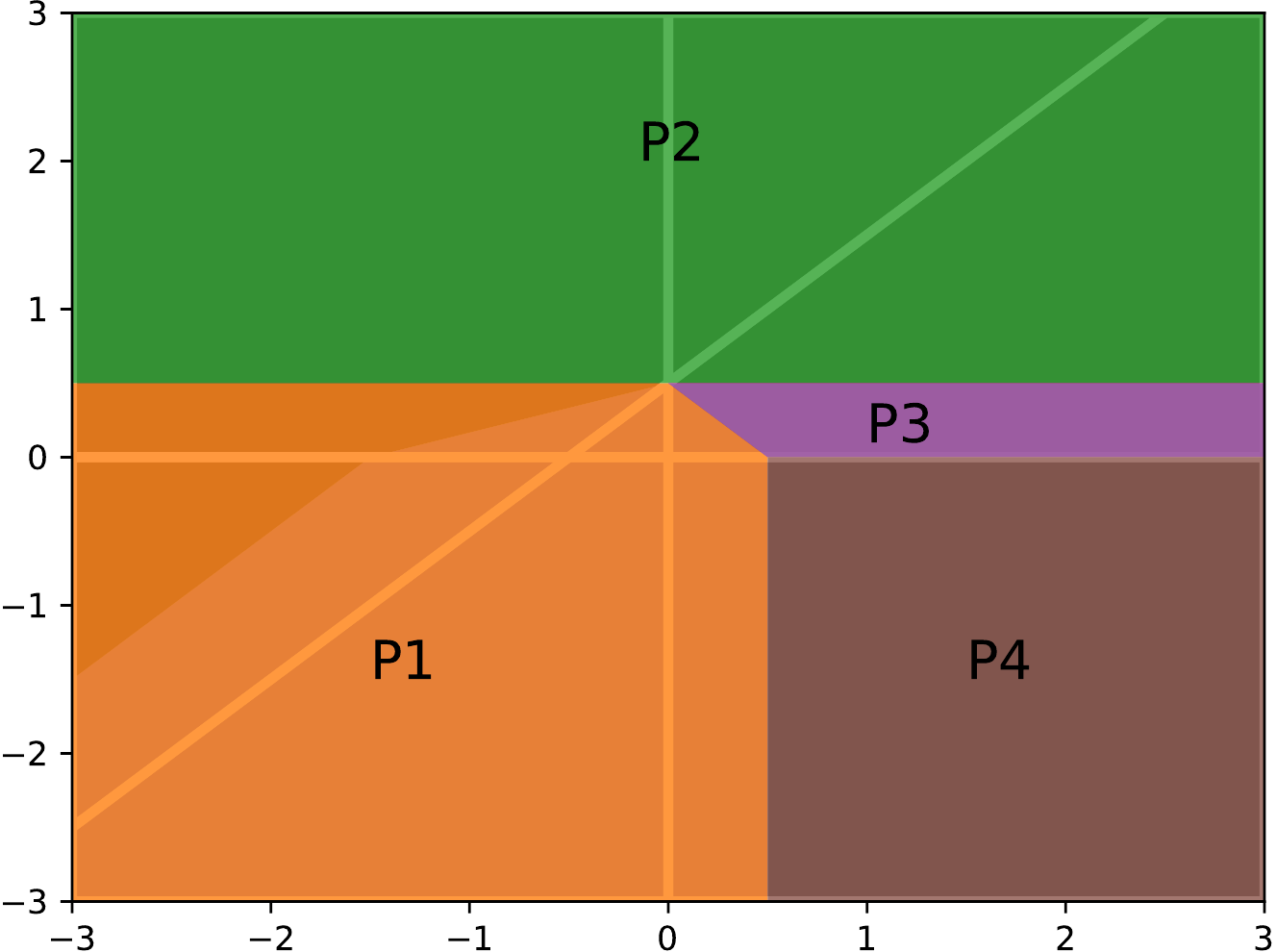}
        \caption{The partitions used in the \emph{patch specification}.}
        \label{fig:patch-spec}
    \end{subfigure}
    \hspace{.5em}
    \begin{subfigure}[t]{.31\linewidth}
        \includegraphics[width=\linewidth]{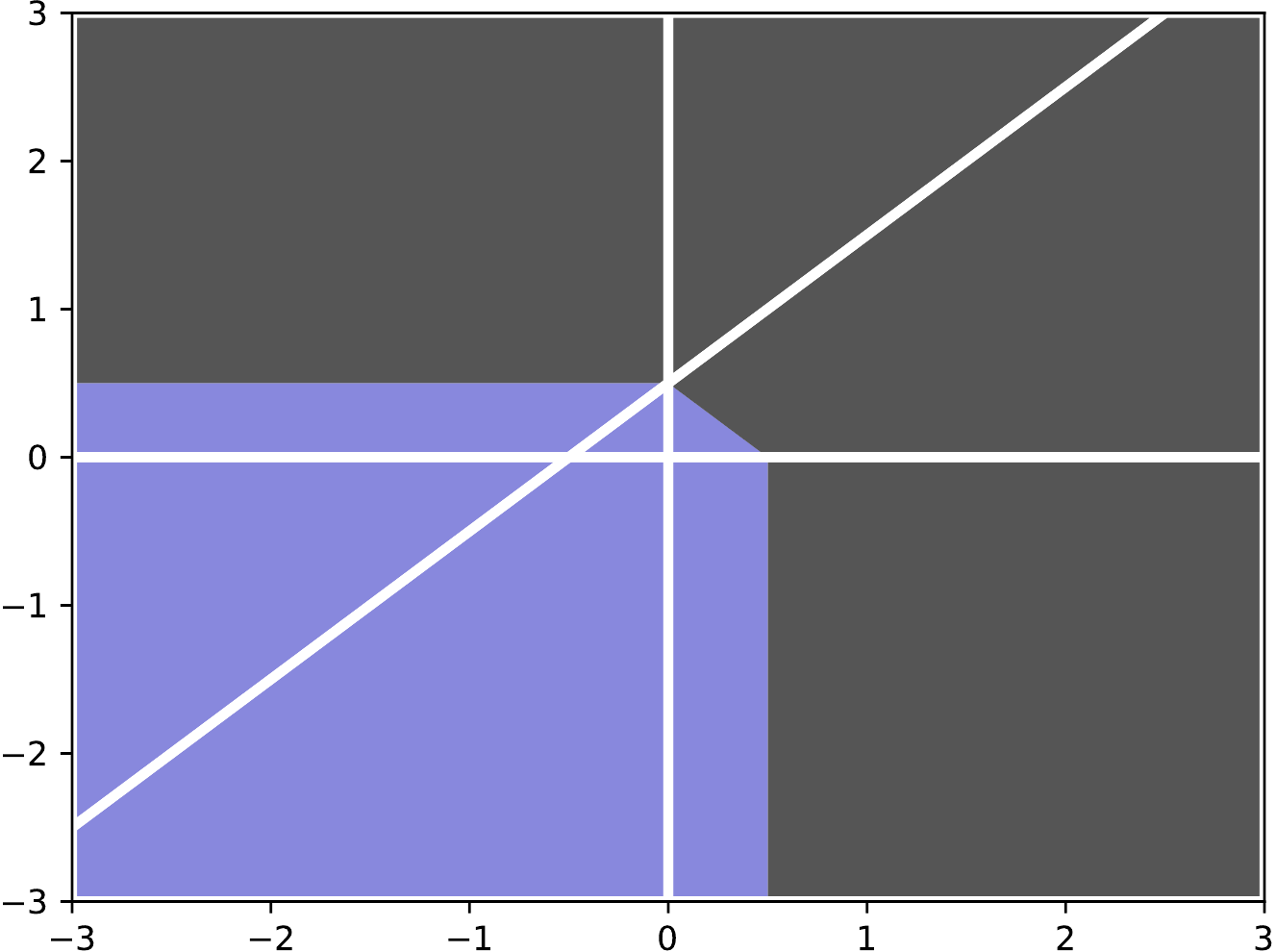}
        \caption{The patched network's linear partitions and decision
        boundaries.}
        \label{fig:patch-post}
    \end{subfigure}
    \caption{Networking patching on DNN $f$ in \pref{eq:NetPatch-network}}
    \label{fig:patch-example}
\end{figure}

\subsection{Deep Neural Network Patching as a Non-Linear MAX SMT}
\label{sec:MAXSMT}
The problem of finding such a patch can be seen as an instance of the more
general \emph{MAX SMT} problem, which looks to find values satisfying as many
given logical formulas as possible. We recall from~\pref{def:Patch} that we
wish to find a $\delta$ satisfying: $\forall (X_i, Y_i) \in T: \forall x \in
X_i: f(x; \theta + \delta) \in Y_i$, which can be directly translated to the
MAX-SMT instance defined by the conjunction $\bigwedge_{(X_i, Y_i) \in T}
\forall x \in X_i: f(x; \theta + \delta) \in Y_i$ (note that $x \in X_i$ and $y
\in Y_i$ can be encoded in any theory supporting rational linear inequalities).

Although \emph{in theory} this problem can be exactly solved by any modern SMT
solver (such as Z3~\cite{TACAS:deMB08}), the extreme non-linearity involved
makes it infeasible in practice. To make it feasible, we would like the problem
to be highly linear, eg. encodable as an LP. Unfortunately, there are numerous
reasons why this is not possible with the above formulation:
\begin{enumerate}
    \item The constraints require that the theory solver can encode the network
        $f$, which is usually highly-non-linear (eg. ReLU layers introduce
        exponentially many possible branches).
    \item It solves for $\delta$ while simultaneously quantifying over the
        infinite $x \in X_i$, resulting in significant non-linearity in the
        problem itself.
    \item If $\delta$ can change multiple weights over multiple layers or
        nodes, the interaction between all weight changes results in a
        high-order polynomial problem instead of a linear one (because the
        linear layers are applied sequentially, eg. $y = w_2(w_1x + b_1) +
        b_2$).
\end{enumerate}

To address the first two of these problems, we will introduce
in~\pref{sec:MaskingNetworks} a new neural network architecture termed
\emph{Masking Networks}. Our core result, \pref{thm:KeyPoints}, will show that
Masking Networks have desirable properties which allow us to (1) replace the
highly-non-linear $f$ with an affine map in each conjunct and (2) lower the
problem of patching on the infinitely-many $x \in X_i$ with patching on
finitely-many \emph{key points}. To address the final issue,
in~\pref{sec:LinearMAXSMT} we will limit the set of weights that can be patched
at the same time to ensure that the resulting problem is still linear. Finally,
in~\pref{sec:IntervalMAXSMT} we will show that extending this restriction to
only changing \emph{a single} weight at a time (referred to as ``weight-wise
patching'') results in an algorithm efficient enough for use on real-world
networks.

\subsection{Masking Networks}
\label{sec:MaskingNetworks}

Our central result in~\pref{thm:KeyPoints} relies on a new type of neural
network, which we propose here and term a \emph{Masking Network}. Masking
Networks strictly generalizes the concept of a feed-forward network with
\relu{} activations and are loosely inspired by the work of~\citet{galu}. We
note that the ideas in this section can be further extended to networks with
arbitrary piecewise linear activation functions (eg. \maxpool{}) but we focus
here on \relu{} for ease of exposition. The essential insight is to fully
separate the \emph{activation pattern} of the network (which defines the
partitioning of $\hat{f}$) from the \emph{result of network} (which defines the
affine maps of $\hat{f}$).

We first review the architecture of standard feed-forward neural networks,
described in~\pref{sec:Overview}. The entire network function $f$ is decomposed
into a series of sequentially-applied \emph{layers}, which we denote $L_1, L_2,
\ldots, L_n$. Given a particular \emph{input vector} to the network, each layer
is said to have an associated \emph{output vector}. Namely, if $x$ is the input
vector, then the output vector of layer $i$ is $L_i(\cdots L_1(x))$. Recall
from~\pref{sec:Overview-DNNs} that, in a standard feed-forward \relu{} network,
the layers alternate between \fullyconnected{} or \conv{} layers (which
correspond to affine maps) and piecewise-linear \relu{} layers, which are
defined component-wise by:
\[ \relu{}(x)_i =
    \begin{cases}
        0 & x_i \leq 0 \\
        x_i & x_i > 0 
    \end{cases}
\]
Meaning that the $i$th component of the output vector is $0$ when the $i$th
component of the input vector is non-positive, and left unchanged otherwise. An
equivalent logical definition would be $(x_i > 0 \wedge \relu{}(x)_i = x_i)
\vee (x_i \leq 0 \wedge \relu{}(x)_i = 0)$.

In a \emph{Masking Network}, four changes are made:

\begin{enumerate}
    \item The inputs and outputs to each layer are \emph{vectors of 2-tuples}
         $((x^a_1, x^v_2), \ldots, (x^a_d, x^v_d))$. In some
        scenarios it will be convenient to refer to the vector consisting of
        just the first-values, noted $x^a = (x^a_1, \ldots, x^a_d)$ and called
        the ``activation vector.'' Correspondingly, we will notate $x^v =
        (x^v_1, \ldots, x^v_d)$ and name this the ``values vector.''
    \item The input vector $x = (x_1, \ldots, x_k)$ to the entire network is
        converted to an input vector where the activation and values vectors
        are equal, i.e. $((x_1, x_1), (x_2, x_2), \ldots, (x_k, x_k))$ or,
        equivalently, $x^a = x^v = x$.
    \item The output of the network is taken to be \emph{only} the values
        vector $x^v$ outputted by the last layer.
    \item \relu{} layers are replaced by \emph{Masked \relu{}} (\mrelu{}) layers, defined below.
\end{enumerate}
In a masking network, \fullyconnected{} or \conv{} layers are associated with
\emph{two} parameters each, $\theta^a$ and $\theta^v$. Given an input with
activation vectors $x^a$ and value vectors $x^v$ to a \fullyconnected{} or
\conv{} layer, $\theta^a$ and $\theta^v$ are applied to the $x^a$ and $x^v$ vectors
independently as follows:
\begin{align*}
    \fullyconnected{}^a(x^a, x^v; \theta^a, \theta^v) &\eqdef \fullyconnected(x^a; \theta^a) \\
    \fullyconnected{}^v(x^a, x^v; \theta^a, \theta^v) &\eqdef \fullyconnected(x^v; \theta^v)
\end{align*}
Given an input vector-of-tuples $((x^a_1, x^v_1), \ldots, (x^a_d,
x^v_d))$ to a \mrelu{} layer, we define its $i$th output tuple to be:
\[
    \mrelu{}\left(\left((x^a_1, x^v_1), \ldots, (x^a_d, x^v_d)\right)\right)_i =
    \begin{cases}
        (0, 0) & x^a_i \leq 0 \\
        (x^a_i, x^v_i) & x^a_i > 0
    \end{cases}
\]

The following properties follow directly from these definitions and are
integral to understanding \maskrelu{} networks:

\begin{enumerate}
    \item A masking network is \emph{piecewise-linear}, and thus $\hatr{f}{X}$
        can be computed for any bounded polytope $X$ by~\pref{thm:PWL-Hat}.
    \item If $x^a = x^v$, then $\mrelu{}(x^a, x^v) = (\relu{}(x^v), \relu{}(x^v))$.
    \item If $x^a = x^v$ and $\theta^a = \theta^v$, then
        $\fullyconnected_{\theta^a, \theta^v}(x^a, x^v) =$
        $(\fullyconnected_{\theta^v}(x^v),$ $\fullyconnected_{\theta^v}(x^v))$
        and similarly for $\conv{}$.
    \item The value of the \emph{activation vectors} ($x^a$) only ever depends
        on the values of prior activation vectors and activation parameters
        ($\theta^a$), never prior value vectors ($x^v$) or value parameters
        ($\theta^v$).
    \item The \emph{non-linear} behavior of the \emph{values vectors} ($x^v$)
        is \emph{entirely} controlled by the activation vectors ($x^a$).  For
        example, if all activation vectors $x^a$ are strictly positive, then
        the entire masking network is affine (regardless of the values of the
        value vectors).
\end{enumerate}
The first property allows us to use the symbolic representation $\hat{f}$ of
the masking network. The next two properties show that we can always encode any
feed-forward \relu{} network as a \maskrelu{} network by setting $\theta^a =
\theta^v = \theta$ for corresponding affine layers (i.e., masking networks are
strictly more expressive than feed-forward \relu{} networks). The final two
properties ensure that:
\begin{enumerate}
    \item The partitioning of $\hat{f}$ is determined \emph{entirely} by the
        activation parameters $\theta^a$.
    \item The affine maps of $\hat{f}$ are determined \emph{entirely} by the
        value parameters $\theta^v$.
\end{enumerate}

Intuitively, now, we can think of changes to the activation parameters $\theta^a$
as \emph{moving the positions of the partitions} in $\hat{f}$ (i.e., changing
the position of the polytopes in~\pref{fig:habit-partitions} and changing the
position of the white lines in~\pref{fig:habit-classes}
and~\pref{fig:patch-pre}). Moreover, changes to the value parameters $\theta^v$
corresponds to changing the classification \emph{within} each partition (i.e.,
changing the position of the decision boundary \emph{within} each white-line
delineated region in~\pref{fig:habit-classes} and~\pref{fig:patch-pre}
\emph{without changing the position of the white lines themselves}).

\subsection{Patching Deep Neural Networks on Infinite Inputs}
\label{sec:NetPatch-Infinite}

The previous observations lead to the following theorem:

\begin{restatable}{theorem}{ThmKeyPoints}
    \label{thm:KeyPoints}
    Given a masking network $f(x; \theta^a, \theta^v)$ parameterized by
    activation parameters $\theta^a$ and value parameters $\theta^v$, 
    an input polytope $X$, output polytope $Y$, and $\hatr{f}{X} = \{ (P_1,
    F_1), \ldots, (P_n, F_n) \}$, then:

    A \emph{values-patched network} $f(x; \theta^a, \theta^v + \delta)$, modifying
    only $\theta^v$, satisfies the property $\forall x \in X: f(x; \theta^a,
    \theta^v + \delta) \in Y$ if and only if it satisfies the property $\forall
    (P_i, F_i) \in \hatr{f(x; \theta^a, \theta^v)}{X}: \forall v \in
    \mathrm{Vert}(P_i): F_i(v; \theta^v + \delta) \in Y$.
\end{restatable}

The above theorem, proved in~\pref{app:key-points}, allows us to translate the
problem of patching over an \emph{infinite} set of points to that of patching on
\emph{finitely-many} vertex points, \emph{as long as} we keep the activation
parameters $\theta^a$ constant. We will use this fact in~\pref{sec:LinearMAXSMT} to
phrase the patching problem as an instance of the linear MAX-SMT problem.

Finally, it is important to note that for any linear partition $(P_i, F_i)$ in
the \emph{unpatched} $\hatr{f(x; \theta^a, \theta^v)}{X}$ we can compute
$F_i(x; \theta^v + \delta)$ for any $\delta$, i.e. the affine map corresponding
to the \emph{values-patched} network $f(x; \theta^a, \theta^v + \delta)$ over
input region $P_i$. The existence of such a map follows from the
observation that the linear partitioning of $\hatr{f(x; \theta^a,
\theta^v)}{X}$ is identical to that of $\hatr{f(x; \theta^a, \theta^v +
\delta)}{X}$ because only the value parameters are modified (so linear
partitions of the unpatched network are still linear partitions of the patched
network.) The exact method of computing $F_i(x; \theta^v + \delta)$ is
unimportant to this discussion, however for expository purposes we provide an
example for the map corresponding to partition $P_4$ in~\pref{fig:patch-spec}:
\begingroup %
\setlength\arraycolsep{2pt}
\[ F(x \in P_4; \theta + \delta) =
    \left(
    \begin{bmatrix}
        1 & 1  & 1 \\
        0 & -1 & -1
    \end{bmatrix} + \delta^3
    \right)
    \begin{bmatrix}
        0 & 0 & 0 \\
        0 & 1 & 0 \\
        0 & 0 & 1
    \end{bmatrix}
    \left(
    \left(
    \begin{bmatrix}
        -1 & 1 \\
        1 & 0 \\
        0 & 1
    \end{bmatrix}
    + \delta^1
    \right)
    \begin{bmatrix}
        x_1 \\
        x_2
    \end{bmatrix}
    +
    \left(
    \begin{bmatrix}
        -0.5 \\
        0 \\
        0
    \end{bmatrix}
    + \delta^2
    \right)
    \right)
    +
    \left(
    \begin{bmatrix}
        0 \\
        1
    \end{bmatrix}
    + \delta^4
    \right)
\]
\endgroup

\subsection{Network Patching as a Linear MAX SMT}
\label{sec:LinearMAXSMT}

In this section, we will show how Masking Networks can be used to address the
issues identified in~\pref{sec:MAXSMT} and finally translate the
highly-non-linear MAX SMT formulation of the patching problem into a linear
one. To do this, we will restrict ourselves to only patching the $\theta^v$
value-parameters, and furthermore only patch the parameters for a \emph{single
node at a time}.

First, we translate the given \relu{} network into an equivalent Masking Network
(which can always be done as Masking Networks are strictly more powerful than
sequential feed-forward networks). We can then address the problem~(1)
from~\pref{sec:MAXSMT}, i.e.\ the non-linearity of $f$, by restricting ourselves
to \emph{only patching the value parameters $\theta^v$}. The value parameters do
not impact the partitioning of $\hat{f}$ (only the affine maps). Thus, this
restriction allows us to break the non-linear MAX SMT query into a series of MAX
SMT queries, one for each linear region, and within each the function is linear
with respect to the input (see the discussion about $F_i(x; \theta^v + \delta)$
in~\pref{sec:NetPatch-Infinite}).

To address problem~(2) from~\pref{sec:MAXSMT}, i.e.\ quantifying over the
infinitely many $x \in X_i$ while solving for $\delta$, we use the result
in~\pref{thm:KeyPoints} to quantify over the finitely-many \emph{vertices of
each $X_i$ intersected with the linear region} instead of over the
infinitely-many points in each $X_i$. \pref{thm:KeyPoints} ensures that the two
are equivalent.

To address the final issue from~\pref{sec:MAXSMT}, i.e. the non-linear
interaction when changing multiple weights in a network, we can limit $\delta$
to only modify the weights corresponding to a single node in the network, which
will prevent such non-linearity arising from sequential layers.  Once all three
of these changes are made, the corresponding satisfiability problem can be
encoded as an LP and more general solvers like Z3 can find solutions to the
maximization problem in some cases.

The corresponding satisfiability problem looks like:
\begin{equation}
    \label{eq:PatchMAXSMT}
    \bigwedge_{(X_i, Y_i) \in T}
    \bigwedge_{(P_j, F_j) \in \hatr{f_M}{X_i}}
    \bigwedge_{v \in \Vert(P_j)}
    F_j(v; \theta + \delta) \in Y_i
\end{equation}
Where $f_M(\cdot; \theta, \theta)$ is the Masking Network equivalent to $f$,
$F_j(v; \theta + \delta)$ is the affine map corresponding to $f_M(\cdot;
\theta, \theta + \delta)$ over linear partition $P_j$, and $\delta$ changes the
weights of at most a single node.

Notably, this involves finitely many conjunctions and $F_j(v; \theta + \delta)$
can be written as an affine function of $\delta$ when $v$ and $\theta$ are
fixed. For example, consider again the expository network visualized
in~\pref{fig:patch-example} and the key point $(3, 0.5)$ lying in polytope
$P_4$. Now, if we only patch the second of the three nodes in the first affine
layer, then $\delta$ has only three non-zero value (one for each of the $3$
inputs to that node along with its bias) which we can write $\delta_1,
\delta_2, \delta_3$. We can then write $F_1(v; \theta, \theta + \delta)$ as:
\begin{align*}
    F(x \in P_4; \theta + \delta) &=
    \begin{bmatrix}
        1 & 1  & 1 \\
        0 & -1 & -1
    \end{bmatrix}
    \begin{bmatrix}
        0 & 0 & 0 \\
        0 & 1 & 0 \\
        0 & 0 & 1
    \end{bmatrix}
    \left(
    \begin{bmatrix}
        -1 & 1 \\
        1 + \delta_1 & \delta_2 \\
        0 & 1
    \end{bmatrix}
    \begin{bmatrix}
        3 \\
        0.5
    \end{bmatrix}
    +
    \begin{bmatrix}
        -0.5 \\
        \delta_3 \\
        0
    \end{bmatrix}
    \right)
    +
    \begin{bmatrix}
        0 \\
        1
    \end{bmatrix} \\
    &=
    \begin{bmatrix}
        3\delta_1 + 0.5\delta_2 + \delta_3 \\
        -3\delta_1 - 0.5\delta_2 - \delta_3
    \end{bmatrix}
    +
    \begin{bmatrix}
        3.5 \\
        -2.5
    \end{bmatrix} \\
    &=
    \begin{bmatrix}
        3 & 0.5 & 1 \\
        -3 & -0.5 & -1
    \end{bmatrix}
    \begin{bmatrix}
        \delta_1 \\ \delta_2 \\ \delta_3
    \end{bmatrix}
    +
    \begin{bmatrix}
        3.5 \\
        -2.5
    \end{bmatrix}
\end{align*}
which is clearly affine in the $\delta_1, \delta_2, \delta_3$
potentially-non-zero components of $\delta$. The exact derivation of the above
matrices is not entirely important; instead, we wish to highlight that this
formulation can translate the corresponding satisfiability problem into a
linear one.

\subsection{Weight-Wise Network Patching as an Interval MAX SMT}
\label{sec:IntervalMAXSMT}

This linear formulation is \emph{still} prohibitively expensive in most
real-world scenarios, due primarily to the large number of ``key points''
necessary (often well over $30,000$). The last optimization we make to the
process is to limit $\delta$ to changing only \emph{a single weight}, which we
refer to as ``weight-wise patching.'' Now, for any given weight in $\theta$,
\pref{eq:PatchMAXSMT} corresponds to a \emph{conjunction of linear inequalities
in one variable}.  Thus, each of the conjuncts corresponds to an \emph{interval
on the corresponding patch $\delta$ in which that constraint is met}. We refer
to the set of all such intervals for the $k$th weight as $I_k$. Finding the
optimal patch for that weight now corresponds to finding an interval $M = [M^l,
M^u]$ that maximizes the number of intervals in $I_k$ that contain it. This can
be done in linear time using a ``linear sweep'' algorithm once the intervals are
sorted. We can repeat this process for every weight in a particular layer (or
the entire network if desired), picking the weight and update which will satisfy
the maximum number of constraints. Finally, we can greedily apply this algorithm
to update multiple weights, at each step making the optimal change to a single
weight in the network. Thus, this greedy algorithm is guaranteed to both
monotonically increase the number of constraints met as well as terminate in a
finite number of steps.

\subsection{Efficiency of Patched Masking Networks}

We close this section with a discussion of the inference-time efficiency of
Masking Networks, in the process identifying another major benefit to
weight-wise patching. On first glance, masking networks appear to double the
amount of computation necessary to perform inference, and indeed, when there is
no relation between the activation and value parameters ($\theta^a$ and
$\theta^v$) this is the case.  However, this is \emph{not} necessarily the case
when the parameterizations \emph{share} many weight values.
As an extreme example, we have already noted that when $\theta^a = \theta^v$,
the network exactly corresponds to a normal \relu{} network. As a slightly more
complex example, if only a single weight in the last affine layer before a
\relu{} layer is changed, then the value of the corresponding node's output
activation and value vector coefficients differ by a constant difference in the
weights multiplied by the corresponding input coefficient, incurring a cost of
only one multiply-accumulate operation.
In that way, if only a small number of weights differ between the values and
activation network, the output of the masking network can be computed by
keeping only ``diffs'' of the values of the nodes during inference.
In~\pref{sec:Evaluation-Netpatch}, on two of our patch specifications we
patched the last affine layer before the last \relu{} layer (so at most one
multiply-accumulate operation of overhead) and in another one we patched the
very last affine layer, which is not followed by a \relu{} layer, so there was
no overhead at all.

\section{Experimental Evaluations}
\label{sec:Evaluation}

In this section, we apply the techniques discussed in this paper to a number of
interesting problems, comparing to prior work where applicable.

Tests were carried out on a dedicated Amazon EC2 c5.metal instance, using
Benchexec~\cite{benchexec} to limit the number of CPU cores to 16 and RAM to
16GB.  Our code is available at\\
\href{https://github.com/95616ARG/SyReNN}{https://github.com/95616ARG/SyReNN}.

\subsection{Understanding Network Behavior using Weakest Precondition}
\label{sec:Evaluation-Precondition}

\begin{figure}
    \centering
    \begin{subfigure}[t]{.30\linewidth}
        \includegraphics[width=\linewidth]{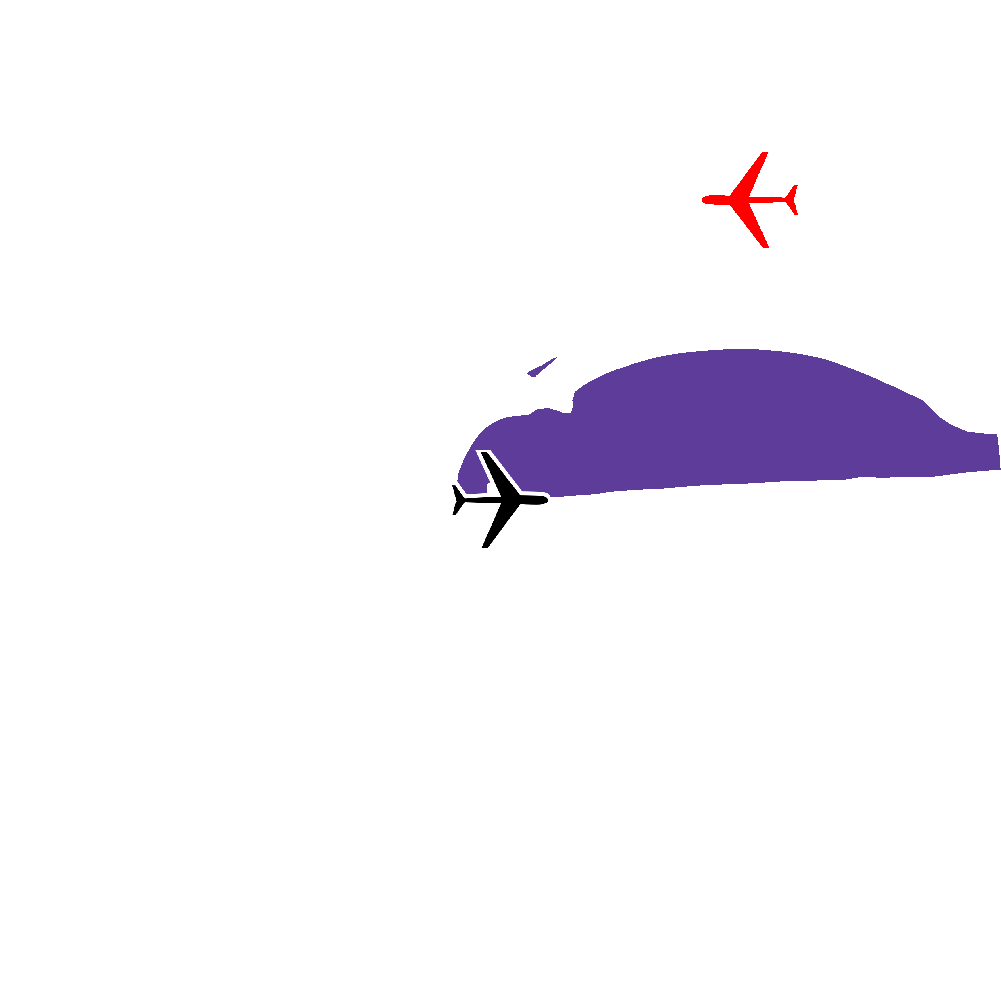}
        \caption{Weakest precondition for Strong Right
        using $\hatr{f}{X}$}
        \label{fig:netviz-fhat-sr}
    \end{subfigure}
    \begin{subfigure}[t]{.30\linewidth}
        \includegraphics[width=\linewidth]{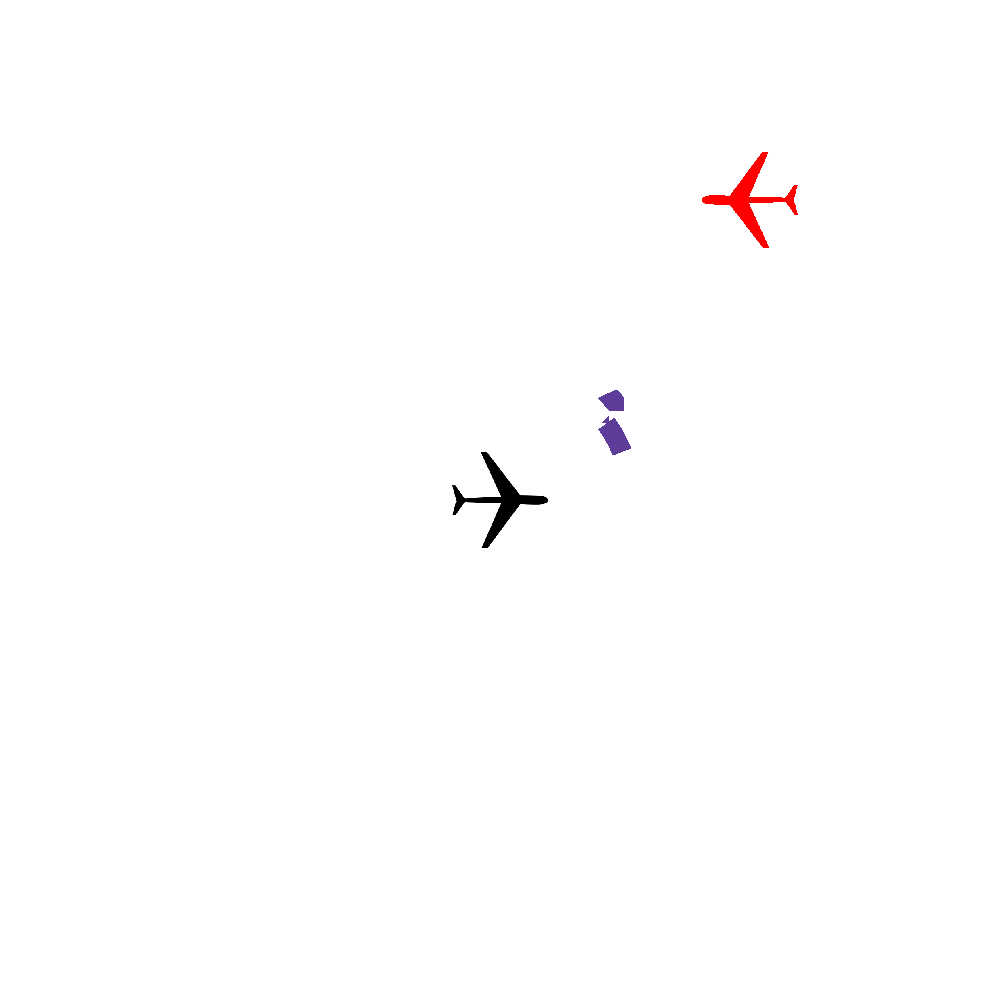}
        \caption{Precondition for Strong Right
        using $\DPPre[k=25]$}
        \label{fig:netviz-dp25-sr}
    \end{subfigure}
    \begin{subfigure}[t]{.30\linewidth}
        \includegraphics[width=\linewidth]{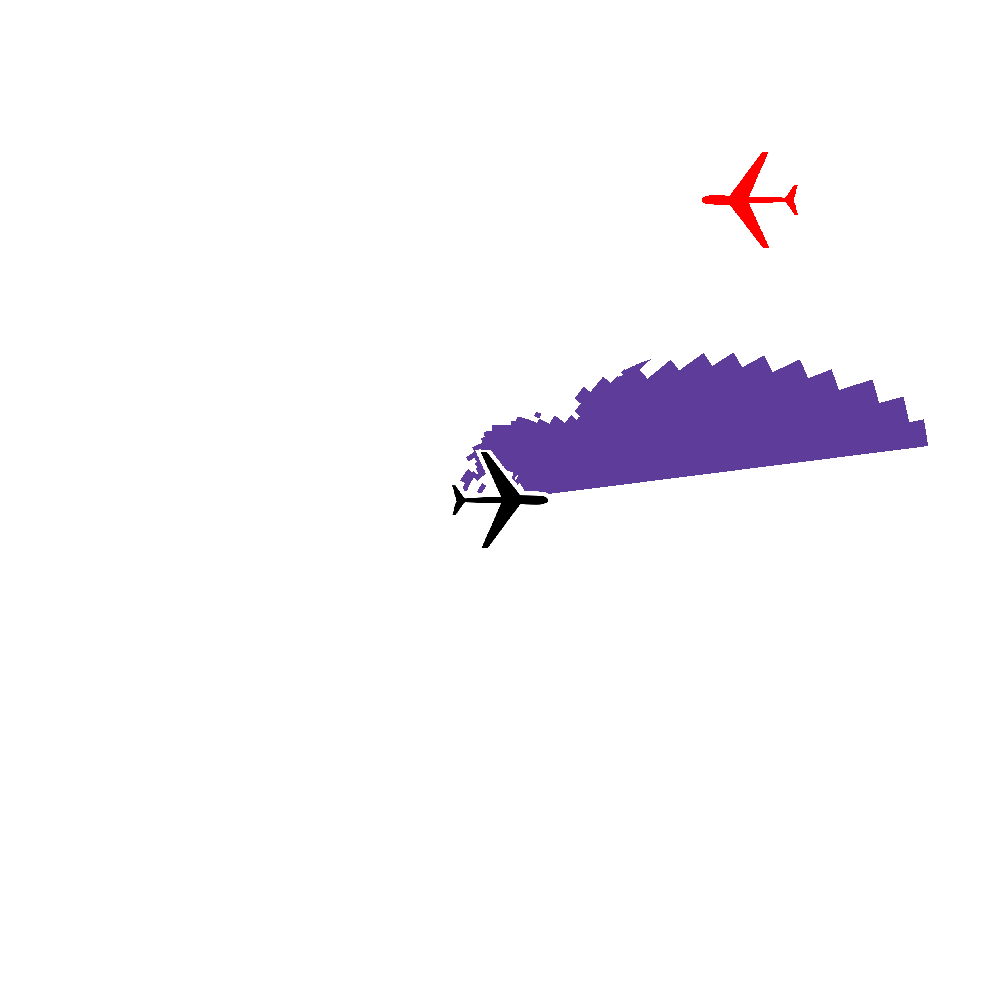}
        \caption{Precondition for Strong Right
        using $\DPPre[k=100]$}
        \label{fig:netviz-dp100-sr}
    \end{subfigure}
    \begin{subfigure}[t]{.30\linewidth}
        \includegraphics[width=\linewidth]{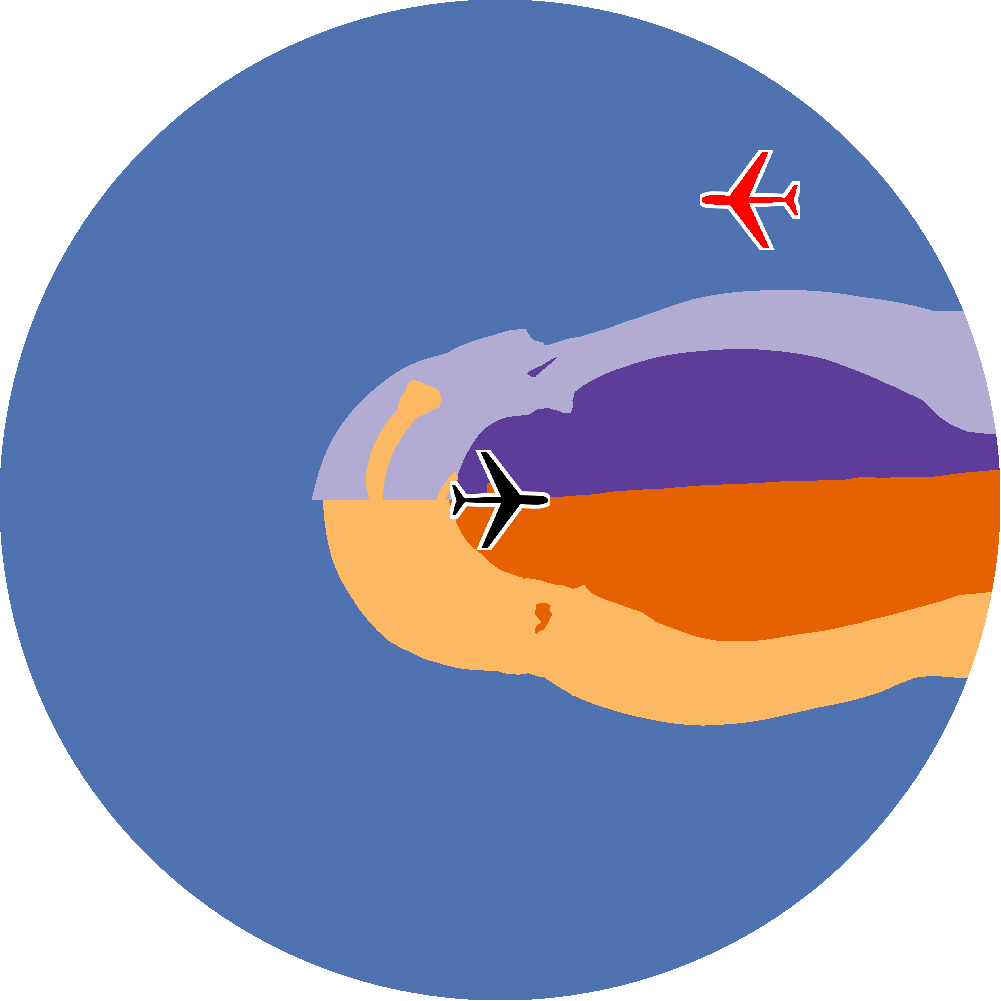}
        \caption{Decision boundaries computed using $\hatr{f}{X}$}
        \label{fig:netviz-fhat-all}
    \end{subfigure}
    \begin{subfigure}[t]{.30\linewidth}
        \includegraphics[width=\linewidth]{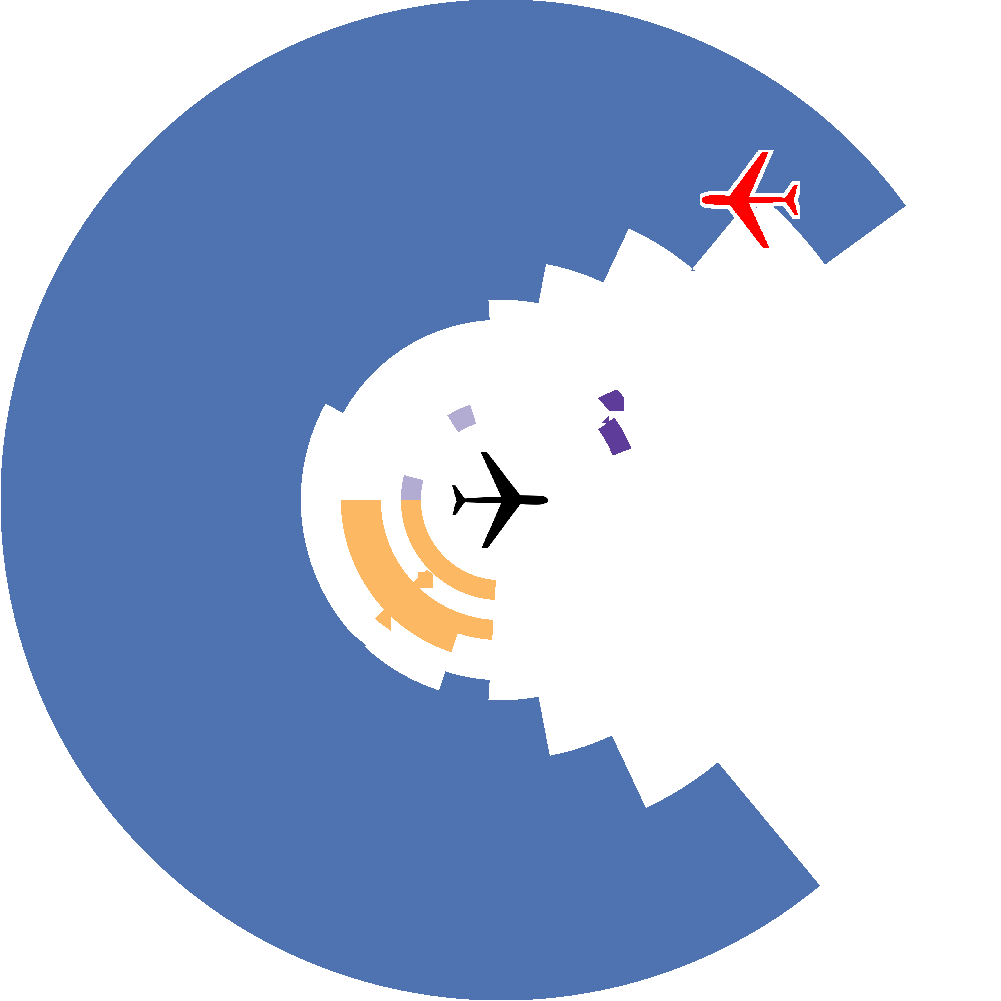}
        \caption{Decision boundaries computed using $\DPPre[k=25]$}
        \label{fig:netviz-dp25-all}
    \end{subfigure}
    \begin{subfigure}[t]{.30\linewidth}
        \includegraphics[width=\linewidth]{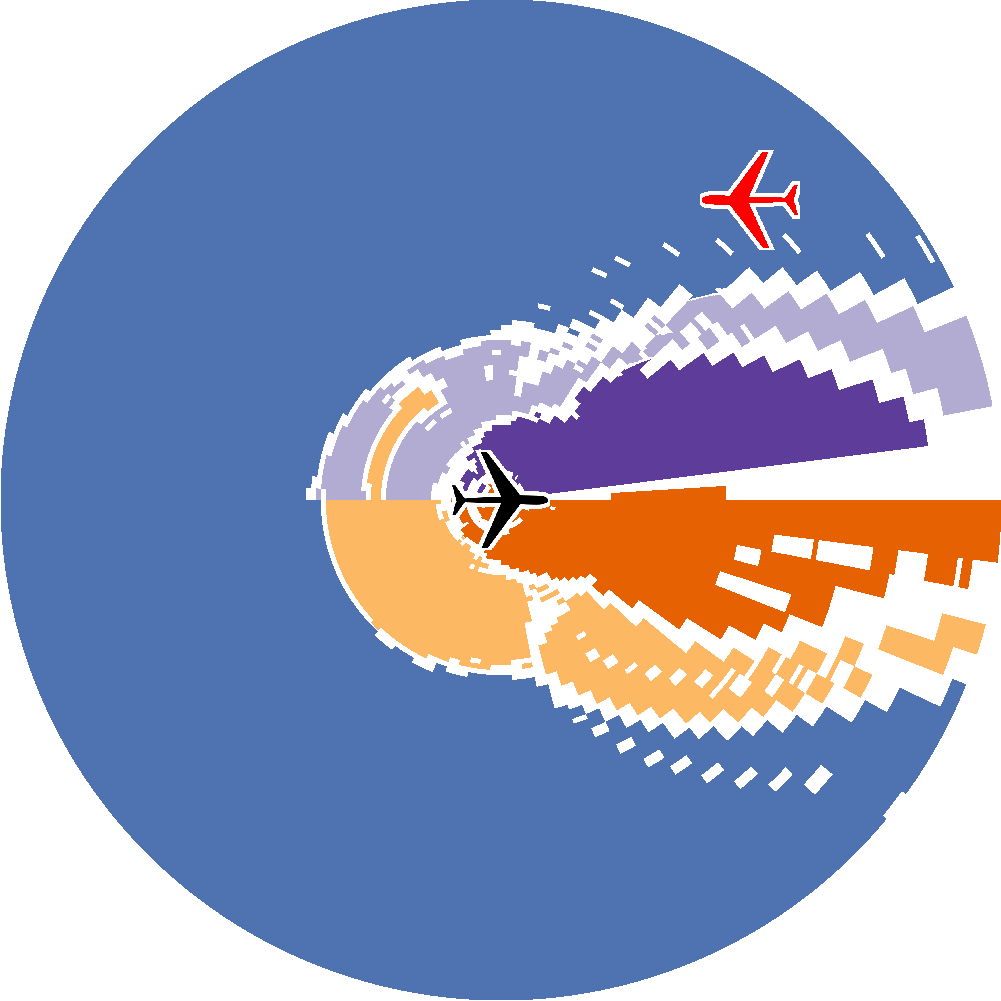}
        \caption{Decision boundaries computed using $\DPPre[k=100]$}
        \label{fig:netviz-dp100-all}
    \end{subfigure}
    \\
    Legend: \textcolor{COCcolor}{\rule[.2\baselineskip]{1em}{2pt}} Clear-of-Conflict,
    \textcolor{WRcolor}{\rule[.2\baselineskip]{1em}{2pt}} Weak Right, 
    \textcolor{SRcolor}{\rule[.2\baselineskip]{1em}{2pt}} Strong Right, 
    \textcolor{SLcolor}{\rule[.2\baselineskip]{1em}{2pt}} Strong Left, 
    \textcolor{WLcolor}{\rule[.2\baselineskip]{1em}{2pt}} Weak Left.
    \caption{Visualization of decision boundaries for the ACAS Xu network.}
    \label{fig:eval-netviz-acas}
\end{figure}

\begin{table}[t]
    \centering
    \caption{Performance of $\hatr{f}{X}$ and $\DPPre$ when used to compute
    weakest precondition and precondition for the ACAS Xu
    network~\citep{julian2018deep}.  $\hatr{f}{X}$ size is the number of
    partitions in the symbolic representation. $k$ is the number of splits used
    for $\DPPre$; the number of partitions used is then $k^2$. Each scenario
    represents a different two-dimensional slice of the input space; within
    each slice, the heading of the intruder relative to the ownship along with
    the speed of each involved plane is fixed.} {\small
    \begin{tabular}{@{}llllll@{}} \toprule
        & & & \multicolumn{3}{c}{$\DPPre$ time (secs)} \\
           \cmidrule{4-6}
        Scenario & $\hatr{f}{X}$ size & $\hatr{f}{X}$ time (secs) & k = 25 & k = 55 & k = 100 \\ \midrule 
        Head-On, Slow        & 33200 & 10.9 & 9.1 & 43.2 & 141.3 \\
        Head-On, Fast        & 30769 & 10.2 & 8.2 & 39.0 & 128.0 \\
        Perpendicular, Slow  & 37251 & 12.5 & 9.2 & 42.9 & 141.7 \\
        Perpendicular, Fast  & 33931 & 11.4 & 8.2 & 39.2 & 127.5 \\
        Opposite, Slow       & 36743 & 12.1 & 9.8 & 46.7 & 152.5 \\
        Opposite, Fast       & 38965 & 13.0 & 9.5 & 45.2 & 147.3 \\
        -Perpendicular, Slow & 36037 & 11.9 & 9.5 & 45.0 & 146.4 \\
        -Perpendicular, Fast & 33208 & 10.9 & 8.3 & 39.5 & 130.2 \\ \bottomrule
    \end{tabular}
    }
    \label{tab:eval-netviz-timing}
\end{table}

For this application, we wanted to investigate the following two research
questions:

\begin{enumerate}
    \item Can we efficiently visualize the decision boundaries of an ACAS Xu
        aircraft-avoidance network using $\hatr{f}{X}$?
    \item How does prior work ($\DPPre[k]$ using DeepPoly \cite{Singh:POPL2019})
        compare, both in efficiency and precision, when used to compute
        $\prer{f}{X}{Y}$?
\end{enumerate}

To investigate these concerns, we wrote programs to compute $\weakprer{f}{X}{Y}$
and $\prer{f}{X}{Y}$ for arbitrary polytopes $X$ and $Y$ using $\hatr{f}{X}$
(\pref{sec:WeakestPrecondition-WithHat}) and $\DPPre$ (\pref{sec:DPPre}),
respectively. Then, for each of the possible advisories produced by the
network, we plotted the precondition set in a human-visualizable form, with the
results shown in~\pref{fig:eval-netviz-acas}. Notably, the precision when using
$\DPPre[k]$ increases when the input domain is pre-partitioned (i.e., $k > 1$),
with the analysis being performed on each partition independently. Thus, we
have shown a progression of plots from the $\DPPre[k]$ approach, each one using
progressively more partitions $k$ for $\DPPre[k]$ (all converging to the
weakest precondition found using $\hatr{f}{X}$). Note that, as we used a
two-dimensional input region, $\DPPre[k]$ uses $k^2$ separate partitions and
calls to $\DPPre[1]$.

The imprecision from using DeepPoly in $\DPPre$ (see~\pref{fig:eval-netviz-acas})
is to be expected, as its symbolic representation was designed to efficiently
answer decision (yes/no) queries about relatively small regions of a
high-dimensional input space. By contrast, in this experiment we are stretching
the DeepPoly representation beyond that goal by extracting preconditions over
large regions of a low-dimensional input space.

\pref{tab:eval-netviz-timing} shows the time taken by 
each analysis. Note that the time taken does \emph{not} include plotting
time; only the time necessary to compute the precondition sets. 
As we increase the number of partitions (equal to $k^2$) used by $\DPPre$, the
precision goes up (as shown in Figures
\ref{fig:netviz-dp25-sr}-\ref{fig:netviz-dp100-sr}) but so does the time taken
for analysis. We find that, except for very small $k$ where the precondition
set is nearly-empty, the $\DPPre$ analysis is slower than that using $\hatr{f}{X}$
(which provides the \emph{exact} $\weakprer{f}{X}{Y}$).

Both of these results match our expectations; $\hatr{f}{X}$ is well-suited for
immediately determining $\weakprer{f}{X}{Y}$, while attempting to use the
DeepPoly representation (which is optimized for the decision procedure case) is
advisable only when precision is not particularly important.

\subsection{Bounded Model Checking of Safety Properties using Strongest Postcondition}
\label{sec:Evaluation-Postcondition}
For this application, we wanted to investigate the following two research
questions:

\begin{enumerate}
    \item Can we use $\hatr{f}{X}$ to efficiently perform bounded model
        checking on a neural-network controller?
    \item How efficient is prior work (ReluPlex \cite{reluplex:CAV2017}) at
        performing bounded model checking on a neural-network controller?
\end{enumerate}

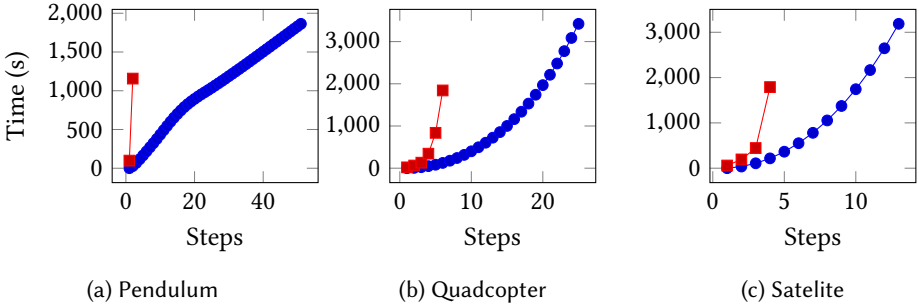
\begin{figure}
    \centering
    \begin{subfigure}[t]{.30\linewidth}
        \begin{tikzpicture}
            \begin{axis}[scale=0.4, xlabel=Steps, ylabel=Time (s)]
                \addplot table [x=Step, y=Cumulative Time, col sep=comma] {data/netverify/pendulum-fhat.csv};
                \addplot table [x=Step, y=Cumulative Time, col sep=comma] {data/netverify/pendulum-rpx.csv};
            \end{axis}
        \end{tikzpicture}
        \caption{Pendulum}
        \label{fig:netverify-pendulum}
    \end{subfigure}
    \begin{subfigure}[t]{.30\linewidth}
        \begin{tikzpicture}
            \begin{axis}[scale=0.4, xlabel=Steps]
                \addplot table [x=Step, y=Cumulative Time, col sep=comma] {data/netverify/quadcopter-fhat.csv};
                \addplot table [x=Step, y=Cumulative Time, col sep=comma] {data/netverify/quadcopter-rpx.csv};
            \end{axis}
        \end{tikzpicture}
        \caption{Quadcopter}
        \label{fig:netverify-pendulum}
    \end{subfigure}
    \begin{subfigure}[t]{.30\linewidth}
        \begin{tikzpicture}
            \begin{axis}[scale=0.4, xlabel=Steps]
                \addplot table [x=Step, y=Cumulative Time, col sep=comma] {data/netverify/satelite-fhat.csv};
                \addplot table [x=Step, y=Cumulative Time, col sep=comma] {data/netverify/satelite-rpx.csv};
            \end{axis}
        \end{tikzpicture}
        \caption{Satelite}
        \label{fig:netverify-satelite}
    \end{subfigure}
    \caption{Performance of bounded model checking (BMC) for three neural-network controllers
    using $\hatr{f}{X}$ (blue line) and ReluPlex (red line).
    The x-axis is the number of steps used in BMC, y-axis is the time taken in
    seconds.  A timeout of 1 hour was used. For the Pendulum model, the
    $\hatr{f}{X}$ approach stopped after finding a counter-example on the 51st
    step (after approximately 30 minutes).}
    \label{fig:netverify-graph}
\end{figure}

\begin{table}[t]
    \centering
    \caption{Maximum number of steps verified by $\hatr{f}{X}$  and ReluPlex 
    approaches for BMC before timeout. Timeout was set to 1 hour. ``*''
    indicates a counter example was found after that many steps.}
    {\small
    \begin{tabular}{@{}lll@{}} \toprule
        Model & $\hatr{f}{X}$ Steps & ReluPlex Steps \\ \midrule 
        Pendulum & 51* & 2 \\
        Quadcopter & 25 & 6 \\
        Satelite & 13 & 4 \\ \bottomrule
    \end{tabular}
    }
    \label{tab:eval-netverify-summary}
\end{table}

To answer these questions, we chose a subset of neural-network controller
models from \cite{DBLP:conf/pldi/ZhuXMJ19} that have two-dimensional states and
affine state-transition functions. We utilized the initial, safe, and unsafe
sets from that work. As an initial optimization, we removed from the initial
set any states which could be guaranteed (based only on the environment
dynamics and maximum/minimum output bounds on the network controller) to map
into the initial set after one timestep. This resulted in a disjunctive initial
state.

We then performed bounded model checking of the models using ReluPlex \cite{reluplex:CAV2017} and
$\hatr{f}{X}$ as described in~\pref{sec:StrongestPostcondition}.
\pref{fig:netverify-graph} shows the time taken to perform the analysis for each
network up to a given timestep of the model, while \pref{tab:eval-netverify-summary} summarizes
these results.

As we can see, bounded model checking with $\hatr{f}{X}$  can be significantly more
efficient than using ReluPlex. This is primarily because the approach using
$\hatr{f}{X}$  \emph{directly} computes the strongest postcondition on network
output after each step, which can then be re-used efficiently in the
strongest-postcondition computation for the next timestep.  By contrast, re-use
between timestep computations for ReluPlex is not possible, meaning verifying
each progressive timestep becomes progressively more challenging.

\subsection{Patching Deep Neural Networks}
\label{sec:Evaluation-Netpatch}
\begin{figure}
    \centering
    \begin{subfigure}[t]{.30\linewidth}
        \begin{tikzpicture}
            \begin{axis}[scale=0.4, xlabel=Iteration, ylabel=\% Constraints Met]
                \addplot table [x=Iteration, y=Percent Constraints Met, col sep=comma] {data/netpatch/pockets.csv};
                \addplot table [x=Iteration, y=Percent Generalized Constraints Met, col sep=comma] {data/netpatch/pockets.csv};
            \end{axis}
        \end{tikzpicture}
        \caption{``Pockets'' Specification}
        \label{fig:netpatch-spec1}
    \end{subfigure}
    \begin{subfigure}[t]{.30\linewidth}
        \begin{tikzpicture}
            \begin{axis}[scale=0.4, xlabel=Iteration]
                \addplot table [x=Iteration, y=Percent Constraints Met, col sep=comma] {data/netpatch/bands.csv};
                \addplot table [x=Iteration, y=Percent Generalized Constraints Met, col sep=comma] {data/netpatch/bands.csv};
            \end{axis}
        \end{tikzpicture}
        \caption{``Bands'' Specification}
        \label{fig:netpatch-spec2}
    \end{subfigure}
    \begin{subfigure}[t]{.30\linewidth}
        \begin{tikzpicture}
            \begin{axis}[scale=0.4, xlabel=Iteration]
                \addplot table [x=Iteration, y=Percent Constraints Met, col sep=comma] {data/netpatch/symmetry.csv};
                \addplot table [x=Iteration, y=Percent Generalized Constraints Met, col sep=comma] {data/netpatch/symmetry.csv};
            \end{axis}
        \end{tikzpicture}
        \caption{``Symmetry'' specification.}
        \label{fig:netpatch-spec0}
    \end{subfigure}
    \caption{Weights changed vs. percent of constraints met. The blue line
    shows the percent of constraints met on the input slice that the patch
    targeted, while the red line shows the percent of constraints met on a
    slice at higher velocity.}
    \label{fig:netpatch-graph}
\end{figure}
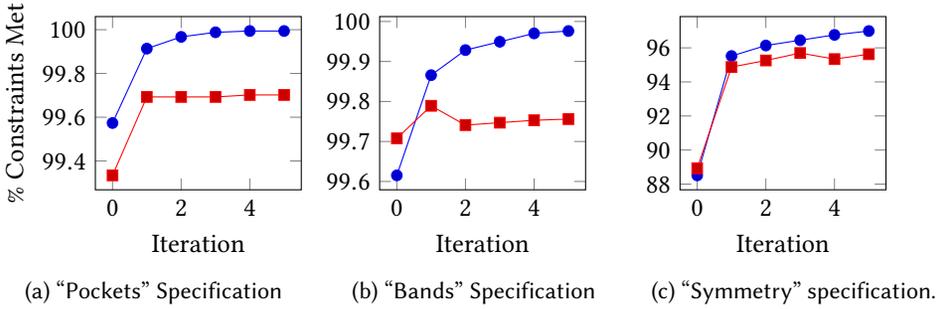

\begin{table}[t]
    \centering
    \caption{Summarizes the percent of constraints met and time taken per
    iteration for patching three specifications. ``Patched Slice'' is the input
    domain of interest that the patch spec applies to, while ``Other Slice''
    shows how patching one slice can generalize to patches for another.}
    {\small
    \begin{tabular}{@{}lllllll@{}} \toprule
        & \multicolumn{4}{c}{Constraints Met} & & \\
           \cmidrule{2-5}
         & \multicolumn{2}{c}{Patched Slice} & \multicolumn{2}{c}{Other Slice} & \multicolumn{2}{c}{Time Per Iteration} \\
           \cmidrule{2-3} \cmidrule{4-5} \cmidrule{6-7}

        Patch    & Initial & Final    & Initial & Final  & Mean & Std. Dev. \\ \midrule 
        Pockets  & 99.6\%  & 99.994\% & 99.3\%  & 99.7\% & 23.6 & 15.2 \\
        Bands    & 99.6\%  & 99.98\%  & 99.7\%  & 99.8\% & 15.9 & 2.3 \\
        Symmetry & 88.5\%  & 97.0\%   & 88.9\%  & 95.6\% & 4.7  & 0.1 \\ \bottomrule
    \end{tabular}
    }
    \label{tab:eval-netpatch-summary}
\end{table}

\newcommand{\netpatchqualitative}[3]{
\begin{figure}
    \centering
    \begin{subfigure}[t]{.30\linewidth}
        \trimmedacas{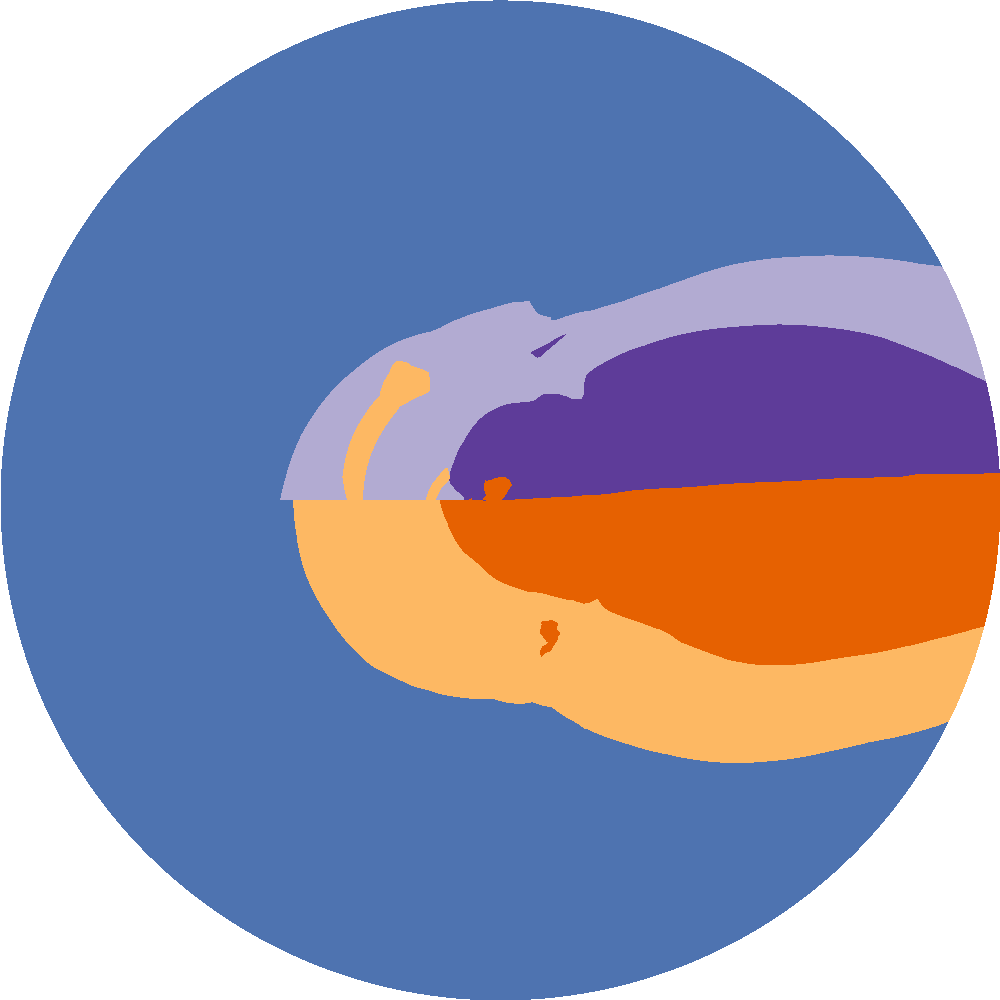}
        \caption{Before patching.}
        \label{fig:netpatch-#2-pre}
    \end{subfigure}
    \hfill
    \begin{subfigure}[t]{.30\linewidth}
        \trimmedacas{figures/netpatch/#1/patch_001}
        \caption{After one iteration (weight change).}
        \label{fig:netpatch-#2-iter0}
    \end{subfigure}
    \hfill
    \begin{subfigure}[t]{.30\linewidth}
        \trimmedacas{figures/netpatch/#1/patch_005}
        \caption{After five iterations (weight changes).}
        \label{fig:netpatch-#2-iter4}
    \end{subfigure}
    \hfill
    \begin{subfigure}[t]{.30\linewidth}
        \trimmedacas{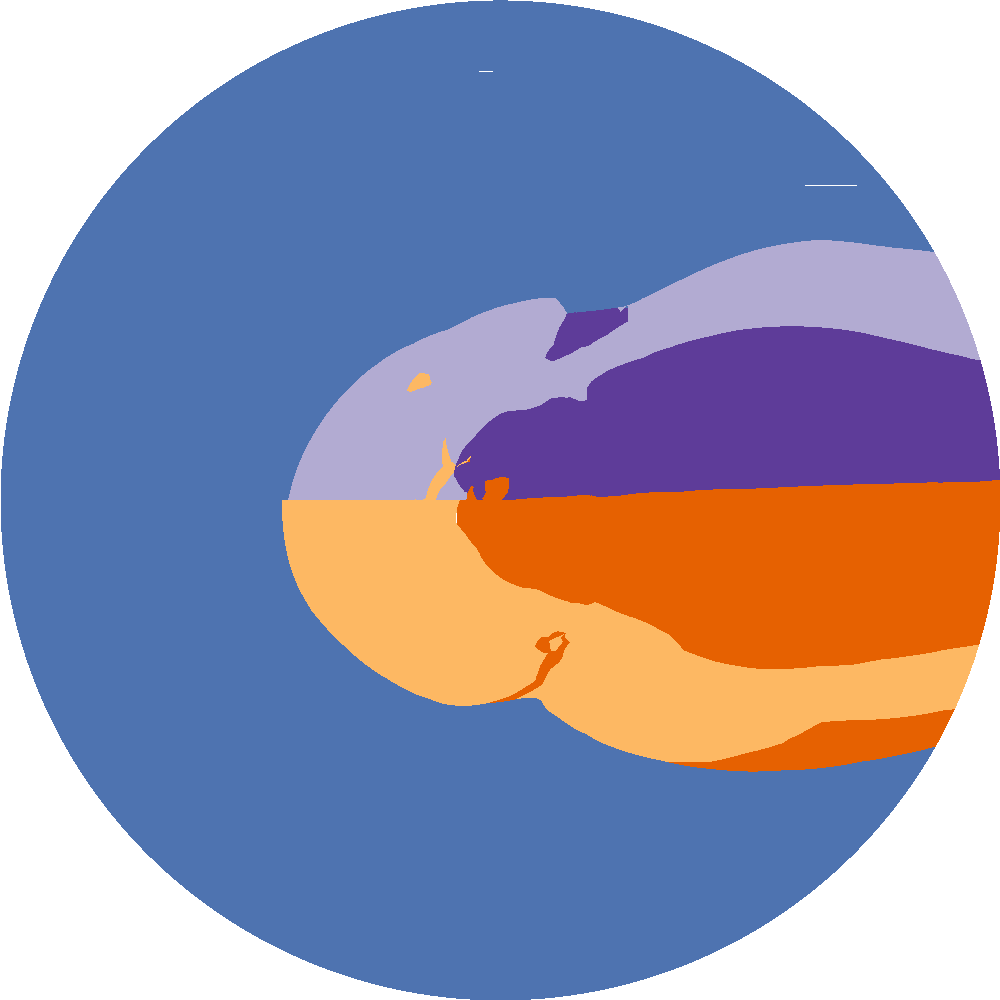}
        \caption{Before patching (generalization region).}
        \label{fig:netpatch-#2-pre-gen}
    \end{subfigure}
    \hfill
    \begin{subfigure}[t]{.30\linewidth}
        \trimmedacas{figures/netpatch/#1/gen_001}
        \caption{Generalization after one iteration (weight change) on the patch region.}
        \label{fig:netpatch-#2-iter0-gen}
    \end{subfigure}
    \hfill
    \begin{subfigure}[t]{.30\linewidth}
        \trimmedacas{figures/netpatch/#1/gen_005}
        \caption{Generalization after five iterations (weight change) on the patch region.}
        \label{fig:netpatch-#2-iter4-gen}
    \end{subfigure}
    \\
    Legend: \textcolor{COCcolor}{\rule[.2\baselineskip]{1em}{2pt}} Clear-of-Conflict,
    \textcolor{WRcolor}{\rule[.2\baselineskip]{1em}{2pt}} Weak Right, 
    \textcolor{SRcolor}{\rule[.2\baselineskip]{1em}{2pt}} Strong Right, 
    \textcolor{SLcolor}{\rule[.2\baselineskip]{1em}{2pt}} Strong Left, 
    \textcolor{WLcolor}{\rule[.2\baselineskip]{1em}{2pt}} Weak Left.
    \caption{Network patching for the ``#2'' spec (#3).}
    \label{fig:netpatch-#2-images}
\end{figure}
}
\netpatchqualitative{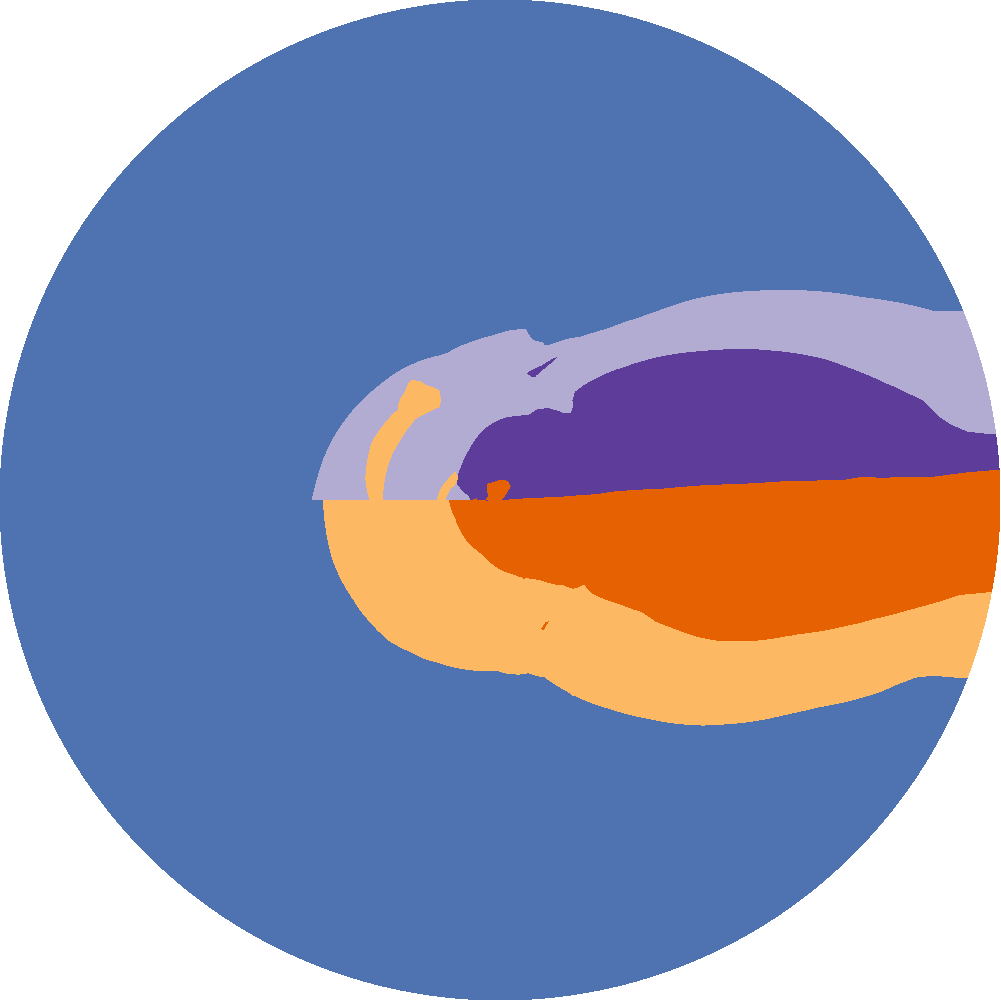}{Pockets}{Removing the ``pockets'' of strong-left and strong-right}
\netpatchqualitative{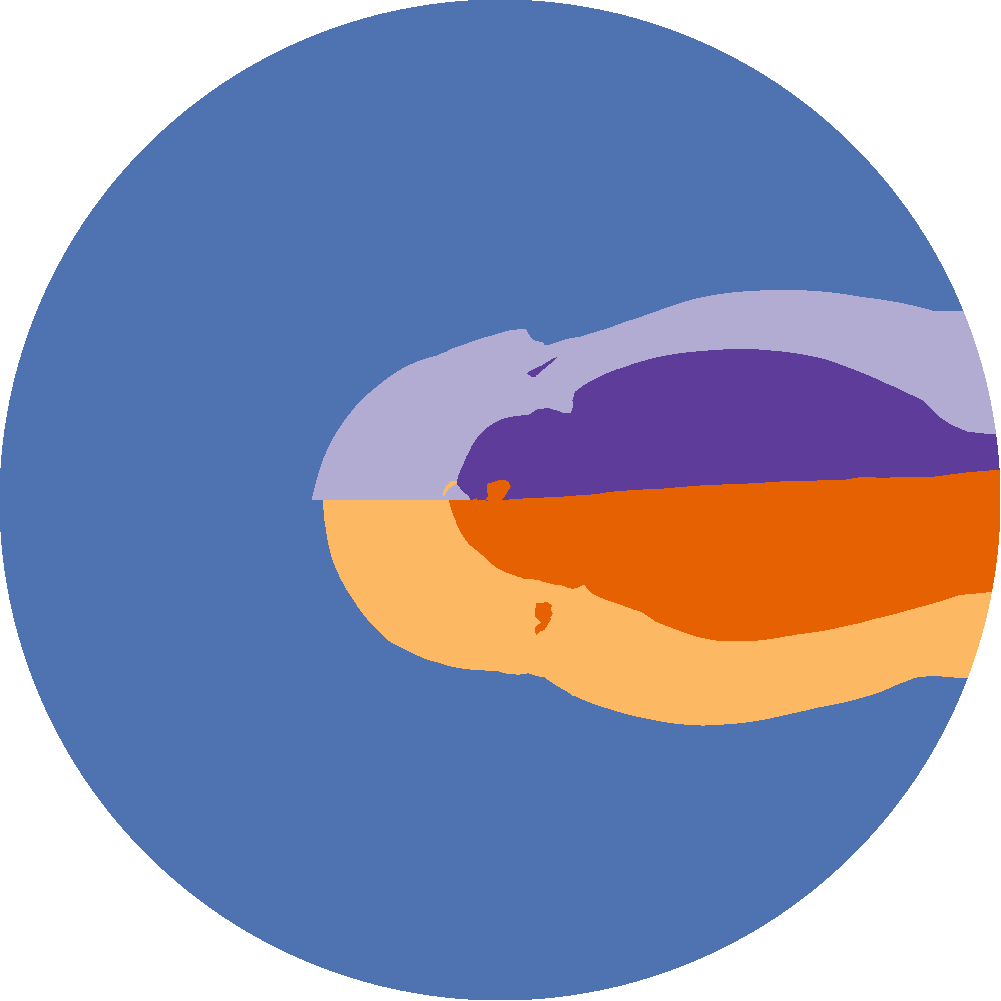}{Bands}{Removing the band of weak-left behind the origin}
\netpatchqualitative{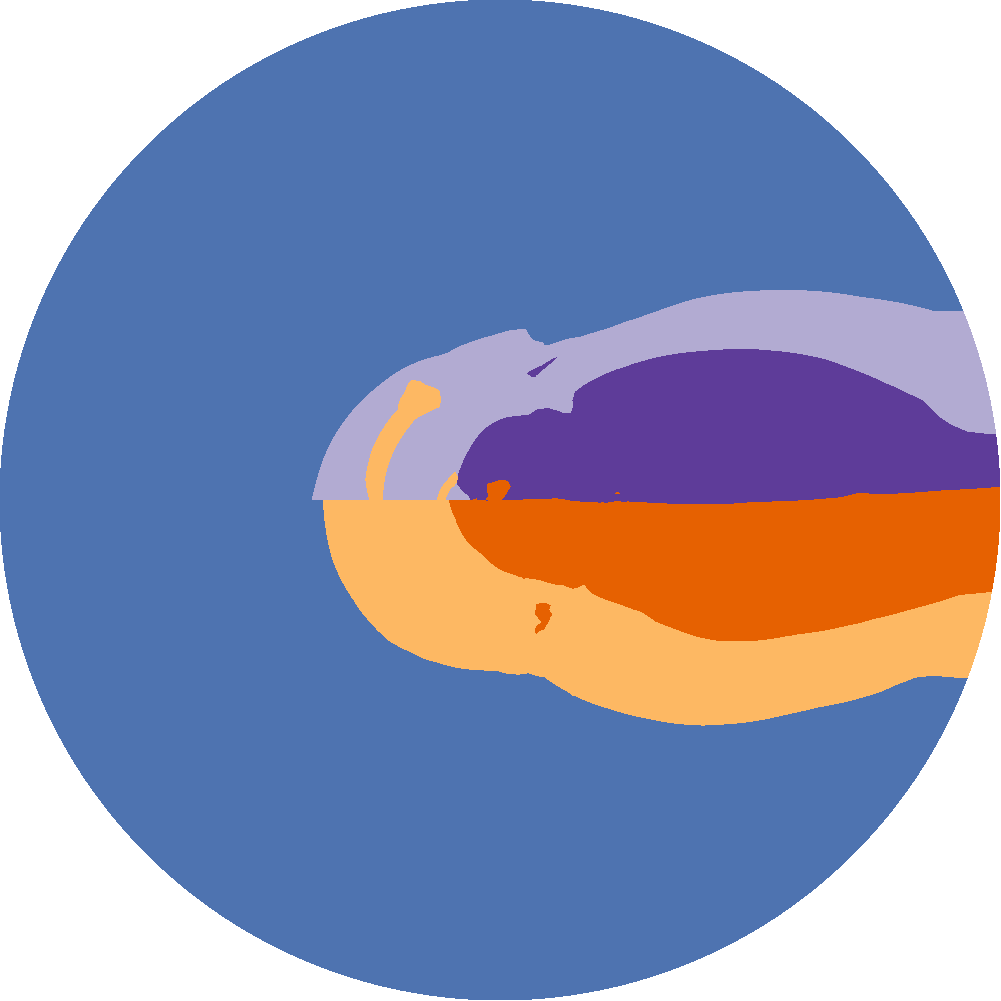}{Symmetry}{Lowering the main decision boundary between strong-left and strong-right to become symmetrical}

For this application, we wanted to understand the following three research
questions:

\begin{enumerate}
    \item Can an ACAS Xu network be patched to correct a number of undesired
        behaviors?
    \item How well do patches on a single two-dimensional subsets of the input
        space generalize to (i.e. fix the same behavior on) other subsets of
        the input space?
    \item How well does the greedy MAX-SMT solver described in~\ref{sec:MAXSMT}
        work, both in terms of efficiency and optimization?
\end{enumerate}

To answer the first two questions, we took the ACAS Xu network visualized
in~\pref{fig:eval-netviz-acas} and attempted to patch three suspicious
behaviors of the network using the approach described
in~\pref{sec:PatchingNetworks}. The ``Pockets'' patch specification
(\pref{fig:netpatch-Pockets-images}) attempts to get rid of the ``pockets'' of
strong left/strong right in regions that are otherwise weak left/weak right.
The ``Bands'' specification (\pref{fig:netpatch-Bands-images}) attempts to get
rid of the weak-left region behind and to the left of the ownship (at the
origin). The ``Symmetry'' specification (\pref{fig:netpatch-Symmetry-images})
attempts to lower the decision boundary between strong left/strong right to be
symmetrical.

After patching, we plotted the patched network in the same two-dimensional
region it was originally plotted (and then patched) over, as well as another
two-dimensional region formed by increasing the velocities of the ownship and
attacking ship by $50$ kilometers per hour (to evaluate generalization to input
regions not explicitly patched). The before and after plots are presented
in~\pref{fig:netpatch-Pockets-images},~\pref{fig:netpatch-Bands-images},
and~\pref{fig:netpatch-Symmetry-images}.

To answer the third question quantitatively, we timed the performance on the
three evaluations shown above and present it
in~\pref{tab:eval-netpatch-summary} along with the percent of constraints met
at each step in~\pref{fig:netpatch-graph}.

We find that our proposed approach works very well; by changing only five
weights, we were able to near-perfectly apply most of the patches (eg. notice
the disappearance of the ``offending'' regions
between~\pref{fig:netpatch-Pockets-pre} and~\pref{fig:netpatch-Pockets-iter4}).
Furthermore, for the most part, the patches appear to \emph{generalize} to
regions not explicitly considered in the patching process (eg. notice the
disappearance of the out-of-place ``strong left'' region
between~\pref{fig:netpatch-Pockets-pre-gen}
and~\pref{fig:netpatch-Pockets-iter0-gen} after changing only a single weight).
The efficacy of the greedy MAX-SMT solver is further exemplified by the
quantitative results in~\pref{tab:eval-netpatch-summary}
and~\pref{fig:netpatch-graph}, which shows that it is very effective at quickly
meeting a large fraction of the desired constraints. However, patching was not,
in general, a ``silver bullet,'' and sometimes undesired changes were
introduced (cf.~\pref{fig:netpatch-Symmetry-iter4-gen}, which has introduced a
small ``strong right'' region into what was otherwise ``weak left'' on the
generalized region). Avoiding such issues is an interesting direction for
future work, but for the most part, we believe the results shown here to be
extremely promising for the application of network patching in practice.

\section{Related Work}
\label{sec:RelatedWork}

This section describes prior work relevant to each of the contributions of this
paper (\pref{sec:Contributions}).

\subsubsubsection{A symbolic representation of deep neural networks}
\citet{Xiang:arxiv2017} solve the problem of exactly computing the reach set of
a neural network given an arbitrary convex input polytope. 
However, the authors use an algorithm
that relies on explicitly enumerating all exponentially-many ($2^n$) possible
signs at each \relu{} layer. By contrast, our algorithm adapts to the actual
input polytopes, efficiently restricting its consideration to activations that
are actually possible.

\citet{DBLP:conf/nips/Thrun94} presents an early approach for extraction of
if-then-else rules from artificial neural networks.
\citet{DBLP:conf/nips/BastaniPS18} learn decision tree policies guided by a DNN
policy that was learned via reinforcement learning. This decision tree could be
seen as a particular form of symbolic representation of the underlying DNN.

\subsubsubsection{Understanding network behavior using weakest precondition.}
Because DNNs are difficult to meaningfully interpret, researchers have tried to
understand the behavior of such networks. For instance, networks have been
shown to be vulnerable to \emph{adversarial examples}---inputs perturbed in a
way imperceptible to humans but which are misclassified by the
network~\citep{Szegedy:ICLR2014,Goodfellow:ICLR2015,DeepFool:CVPR2016,carlini2018audio}--and
\emph{fooling examples}---inputs that are completely unrecognizable by humans
but recognized by DNNs~\citep{Nguyen:CVPR2015}. 
\citet{DBLP:conf/esann/BreutelMH03} presents an iterative refinement algorithm
that computes an overapproximation of the weakest precondition as a polytope
where the required output is also a polytope.

\subsubsubsection{Bounded model checking of safety properties using strongest postcondition.}
\citet{DBLP:conf/mbmv/ScheiblerWWB15} verify the safety of a machine-learning
controller with BMC using the SMT-solver iSAT3, but support small unrolling
depths and basic safety properties. \citet{DBLP:conf/pldi/ZhuXMJ19} use a
synthesis procedure to generate a safe deterministic program that can enforce
safety conditions by monitoring the deployed DNN and preventing potentially
unsafe actions.
The presence of adversarial and fooling inputs for DNNs as well as applications
of DNNs in safety-critical systems has led to efforts to verify and certify
DNNs~\cite{Bastani:NIPS2016,reluplex:CAV2017,Ehlers:ATVA2017,Huang:CAV2017,ai2:SP2018,Bunel:NIPS2018,Weng:ICML2018,Singh:POPL2019,Anderson:PLDI2019}.
\emph{Approximate reachability analysis} for neural networks safely
overapproximates the set of possible outputs
\citep{ai2:SP2018,Xiang:arxiv2017,Xiang:ACC2018,Weng:ICML2018,Dutta:NFM2018,ReluVal:Usenix2018}.

\subsubsubsection{Patching deep neural networks.}
Prior work focuses on enforcing constraints on the network during training.
DiffAI \citep{diffai2018} is an approach to train neural networks that are
certifiably robust to adversarial perturbations. DL2 \citep{fischer2019dl2}
allows for training and querying neural networks with logical constraints.

\section{Conclusion}
\label{sec:Conclusion}

We proposed a \emph{symbolic neural network representation}, denoted $\hat{f}$,
which decomposes a piecewise-linear neural network into its constituent affine
maps.
We applied this symbolic representation to three problems relating to trained
neural networks. First, we showed how the representation can be used to compute
\emph{weakest preconditions}, which allow one to exactly visualize network
decision boundaries. Next, we demonstrated the symbolic representation's use
for computing \emph{strongest postconditions}, allowing us to perform iterative
bounded model checking. Finally, we introduced the problem of \emph{patching}
neural networks, which can be interpreted as explicitly correcting their
decision boundaries. We presented an efficient and effective approach to
patching neural networks, relying on a new neural network architecture, the
symbolic representation, and a dedicated MAX SMT solver.

\begin{acks}
    We thank Nina Amenta, Yong Jae Lee, Mukund Sundararajan, and Cindy
    Rubio-Gonzalez for their feedback and suggestions on this work, along with
    Amazon Web Services Cloud credits for research.
\end{acks}

\bibliography{main}


\begin{thebibliography}{38}


\ifx \showCODEN    \undefined \def \showCODEN     #1{\unskip}     \fi
\ifx \showDOI      \undefined \def \showDOI       #1{#1}\fi
\ifx \showISBNx    \undefined \def \showISBNx     #1{\unskip}     \fi
\ifx \showISBNxiii \undefined \def \showISBNxiii  #1{\unskip}     \fi
\ifx \showISSN     \undefined \def \showISSN      #1{\unskip}     \fi
\ifx \showLCCN     \undefined \def \showLCCN      #1{\unskip}     \fi
\ifx \shownote     \undefined \def \shownote      #1{#1}          \fi
\ifx \showarticletitle \undefined \def \showarticletitle #1{#1}   \fi
\ifx \showURL      \undefined \def \showURL       {\relax}        \fi
\providecommand\bibfield[2]{#2}
\providecommand\bibinfo[2]{#2}
\providecommand\natexlab[1]{#1}
\providecommand\showeprint[2][]{arXiv:#2}

\bibitem[\protect\citeauthoryear{??}{ERA}{2019}]%
        {ERAN}
 \bibinfo{year}{2019}\natexlab{}.
\newblock \bibinfo{title}{{ETH} Robustness Analyzer for Neural Networks
  ({ERAN})}.
\newblock \bibinfo{howpublished}{\url{https://github.com/eth-sri/eran}}.
\newblock
\newblock
\shownote{Accessed: 2019-05-01.}


\bibitem[\protect\citeauthoryear{Anderson, Pailoor, Dillig, and
  Chaudhuri}{Anderson et~al\mbox{.}}{2019}]%
        {Anderson:PLDI2019}
\bibfield{author}{\bibinfo{person}{Greg Anderson}, \bibinfo{person}{Shankara
  Pailoor}, \bibinfo{person}{Isil Dillig}, {and} \bibinfo{person}{Swarat
  Chaudhuri}.} \bibinfo{year}{2019}\natexlab{}.
\newblock \showarticletitle{Optimization and Abstraction: {A} Synergistic
  Approach for Analyzing Neural Network Robustness}.
\newblock \bibinfo{journal}{\emph{CoRR}}  \bibinfo{volume}{abs/1904.09959}
  (\bibinfo{year}{2019}).
\newblock


\bibitem[\protect\citeauthoryear{Bastani, Ioannou, Lampropoulos, Vytiniotis,
  Nori, and Criminisi}{Bastani et~al\mbox{.}}{2016}]%
        {Bastani:NIPS2016}
\bibfield{author}{\bibinfo{person}{Osbert Bastani}, \bibinfo{person}{Yani
  Ioannou}, \bibinfo{person}{Leonidas Lampropoulos}, \bibinfo{person}{Dimitrios
  Vytiniotis}, \bibinfo{person}{Aditya~V. Nori}, {and} \bibinfo{person}{Antonio
  Criminisi}.} \bibinfo{year}{2016}\natexlab{}.
\newblock \showarticletitle{Measuring Neural Net Robustness with Constraints}.
  In \bibinfo{booktitle}{\emph{Advances in Neural Information Processing
  Systems}}.
\newblock


\bibitem[\protect\citeauthoryear{Bastani, Pu, and Solar{-}Lezama}{Bastani
  et~al\mbox{.}}{2018}]%
        {DBLP:conf/nips/BastaniPS18}
\bibfield{author}{\bibinfo{person}{Osbert Bastani}, \bibinfo{person}{Yewen Pu},
  {and} \bibinfo{person}{Armando Solar{-}Lezama}.}
  \bibinfo{year}{2018}\natexlab{}.
\newblock \showarticletitle{Verifiable Reinforcement Learning via Policy
  Extraction}. In \bibinfo{booktitle}{\emph{Advances in Neural Information
  Processing Systems 31: Annual Conference on Neural Information Processing
  Systems 2018, NeurIPS 2018, 3-8 December 2018, Montr{\'{e}}al, Canada.}}
  \bibinfo{pages}{2499--2509}.
\newblock


\bibitem[\protect\citeauthoryear{Beyer}{Beyer}{2016}]%
        {benchexec}
\bibfield{author}{\bibinfo{person}{Dirk Beyer}.}
  \bibinfo{year}{2016}\natexlab{}.
\newblock \showarticletitle{Reliable and reproducible competition results with
  benchexec and witnesses (report on SV-COMP 2016)}. In
  \bibinfo{booktitle}{\emph{International Conference on Tools and Algorithms
  for the Construction and Analysis of Systems {(TACAS)}}}. Springer,
  \bibinfo{pages}{887--904}.
\newblock


\bibitem[\protect\citeauthoryear{Biere, Cimatti, Clarke, Strichman, and
  Zhu}{Biere et~al\mbox{.}}{2009}]%
        {biere2009bounded}
\bibfield{author}{\bibinfo{person}{Armin Biere}, \bibinfo{person}{Alessandro
  Cimatti}, \bibinfo{person}{Edmund~M Clarke}, \bibinfo{person}{Ofer
  Strichman}, {and} \bibinfo{person}{Yunshan Zhu}.}
  \bibinfo{year}{2009}\natexlab{}.
\newblock \showarticletitle{Bounded Model Checking.}
\newblock \bibinfo{journal}{\emph{Handbook of satisfiability}}
  \bibinfo{volume}{185}, \bibinfo{number}{99} (\bibinfo{year}{2009}),
  \bibinfo{pages}{457--481}.
\newblock


\bibitem[\protect\citeauthoryear{Breutel, Maire, and Hayward}{Breutel
  et~al\mbox{.}}{2003}]%
        {DBLP:conf/esann/BreutelMH03}
\bibfield{author}{\bibinfo{person}{Stephan Breutel},
  \bibinfo{person}{Fr{\'{e}}d{\'{e}}ric Maire}, {and} \bibinfo{person}{Ross
  Hayward}.} \bibinfo{year}{2003}\natexlab{}.
\newblock \showarticletitle{Extracting Interface Assertions from Neural
  Networks in Polyhedral Format}. In \bibinfo{booktitle}{\emph{{ESANN} 2003,
  11th European Symposium on Artificial Neural Networks, Bruges, Belgium, April
  23-25, 2003, Proceedings}}. \bibinfo{pages}{463--468}.
\newblock


\bibitem[\protect\citeauthoryear{Bunel, Turkaslan, Torr, Kohli, and
  Mudigonda}{Bunel et~al\mbox{.}}{2018}]%
        {Bunel:NIPS2018}
\bibfield{author}{\bibinfo{person}{Rudy~R. Bunel}, \bibinfo{person}{Ilker
  Turkaslan}, \bibinfo{person}{Philip H.~S. Torr}, \bibinfo{person}{Pushmeet
  Kohli}, {and} \bibinfo{person}{Pawan~Kumar Mudigonda}.}
  \bibinfo{year}{2018}\natexlab{}.
\newblock \showarticletitle{A Unified View of Piecewise Linear Neural Network
  Verification}. In \bibinfo{booktitle}{\emph{Advances in Neural Information
  Processing Systems 31: Annual Conference on Neural Information Processing
  Systems 2018, NeurIPS 2018, 3-8 December 2018, Montr{\'{e}}al, Canada.}}
  \bibinfo{pages}{4795--4804}.
\newblock


\bibitem[\protect\citeauthoryear{Carlini and Wagner}{Carlini and
  Wagner}{2018}]%
        {carlini2018audio}
\bibfield{author}{\bibinfo{person}{Nicholas Carlini} {and}
  \bibinfo{person}{David Wagner}.} \bibinfo{year}{2018}\natexlab{}.
\newblock \showarticletitle{Audio adversarial examples: Targeted attacks on
  speech-to-text}. In \bibinfo{booktitle}{\emph{2018 IEEE Security and Privacy
  Workshops (SPW)}}. IEEE, \bibinfo{pages}{1--7}.
\newblock


\bibitem[\protect\citeauthoryear{Collobert}{Collobert}{2004}]%
        {collobert2004large}
\bibfield{author}{\bibinfo{person}{Ronan Collobert}.}
  \bibinfo{year}{2004}\natexlab{}.
\newblock \emph{\bibinfo{title}{Large scale machine learning}}.
\newblock \bibinfo{thesistype}{Ph.D. Dissertation}.
  \bibinfo{school}{Universit{\'e} de Paris VI}.
\newblock


\bibitem[\protect\citeauthoryear{de~Moura and Bj{\o}rner}{de~Moura and
  Bj{\o}rner}{2008}]%
        {TACAS:deMB08}
\bibfield{author}{\bibinfo{person}{L. de Moura} {and} \bibinfo{person}{N.
  Bj{\o}rner}.} \bibinfo{year}{2008}\natexlab{}.
\newblock \showarticletitle{{Z3}: {A}n Efficient {SMT} Solver}. In
  \bibinfo{booktitle}{\emph{International Conference on Tools and Algorithms
  for the Construction and Analysis of Systems {(TACAS)}}}.
\newblock


\bibitem[\protect\citeauthoryear{Devlin, Chang, Lee, and Toutanova}{Devlin
  et~al\mbox{.}}{2018}]%
        {BERT:CoRR2018}
\bibfield{author}{\bibinfo{person}{Jacob Devlin}, \bibinfo{person}{Ming{-}Wei
  Chang}, \bibinfo{person}{Kenton Lee}, {and} \bibinfo{person}{Kristina
  Toutanova}.} \bibinfo{year}{2018}\natexlab{}.
\newblock \showarticletitle{{BERT:} Pre-training of Deep Bidirectional
  Transformers for Language Understanding}.
\newblock \bibinfo{journal}{\emph{CoRR}}  \bibinfo{volume}{abs/1810.04805}
  (\bibinfo{year}{2018}).
\newblock


\bibitem[\protect\citeauthoryear{Dutta, Jha, Sankaranarayanan, and
  Tiwari}{Dutta et~al\mbox{.}}{2018}]%
        {Dutta:NFM2018}
\bibfield{author}{\bibinfo{person}{Souradeep Dutta}, \bibinfo{person}{Susmit
  Jha}, \bibinfo{person}{Sriram Sankaranarayanan}, {and}
  \bibinfo{person}{Ashish Tiwari}.} \bibinfo{year}{2018}\natexlab{}.
\newblock \showarticletitle{Output range analysis for deep feedforward neural
  networks}. In \bibinfo{booktitle}{\emph{NASA Formal Methods Symposium}}.
  Springer, \bibinfo{pages}{121--138}.
\newblock


\bibitem[\protect\citeauthoryear{Ehlers}{Ehlers}{2017}]%
        {Ehlers:ATVA2017}
\bibfield{author}{\bibinfo{person}{Ruediger Ehlers}.}
  \bibinfo{year}{2017}\natexlab{}.
\newblock \showarticletitle{Formal verification of piece-wise linear
  feed-forward neural networks}. In \bibinfo{booktitle}{\emph{International
  Symposium on Automated Technology for Verification and Analysis {(ATVA)}}}.
\newblock


\bibitem[\protect\citeauthoryear{Fiat, Malach, and Shalev-Shwartz}{Fiat
  et~al\mbox{.}}{2019}]%
        {galu}
\bibfield{author}{\bibinfo{person}{Jonathan Fiat}, \bibinfo{person}{Eran
  Malach}, {and} \bibinfo{person}{Shai Shalev-Shwartz}.}
  \bibinfo{year}{2019}\natexlab{}.
\newblock \showarticletitle{Decoupling Gating from Linearity}.
\newblock \bibinfo{journal}{\emph{arXiv preprint arXiv:1906.05032}}
  (\bibinfo{year}{2019}).
\newblock


\bibitem[\protect\citeauthoryear{Fischer, Balunovic, Drachsler-Cohen, Gehr,
  Zhang, and Vechev}{Fischer et~al\mbox{.}}{2019}]%
        {fischer2019dl2}
\bibfield{author}{\bibinfo{person}{Marc Fischer}, \bibinfo{person}{Mislav
  Balunovic}, \bibinfo{person}{Dana Drachsler-Cohen}, \bibinfo{person}{Timon
  Gehr}, \bibinfo{person}{Ce Zhang}, {and} \bibinfo{person}{Martin Vechev}.}
  \bibinfo{year}{2019}\natexlab{}.
\newblock \showarticletitle{DL2: Training and Querying Neural Networks with
  Logic}. In \bibinfo{booktitle}{\emph{International Conference on Machine
  Learning}}.
\newblock


\bibitem[\protect\citeauthoryear{Gehr, Mirman, Drachsler{-}Cohen, Tsankov,
  Chaudhuri, and Vechev}{Gehr et~al\mbox{.}}{2018}]%
        {ai2:SP2018}
\bibfield{author}{\bibinfo{person}{Timon Gehr}, \bibinfo{person}{Matthew
  Mirman}, \bibinfo{person}{Dana Drachsler{-}Cohen}, \bibinfo{person}{Petar
  Tsankov}, \bibinfo{person}{Swarat Chaudhuri}, {and}
  \bibinfo{person}{Martin~T. Vechev}.} \bibinfo{year}{2018}\natexlab{}.
\newblock \showarticletitle{{AI2:} Safety and Robustness Certification of
  Neural Networks with Abstract Interpretation}. In
  \bibinfo{booktitle}{\emph{2018 {IEEE} Symposium on Security and Privacy, {SP}
  2018, Proceedings, 21-23 May 2018, San Francisco, California, {USA}}}.
\newblock


\bibitem[\protect\citeauthoryear{Goodfellow, Bengio, and Courville}{Goodfellow
  et~al\mbox{.}}{2016}]%
        {Goodfellow:DeepLearning2016}
\bibfield{author}{\bibinfo{person}{Ian Goodfellow}, \bibinfo{person}{Yoshua
  Bengio}, {and} \bibinfo{person}{Aaron Courville}.}
  \bibinfo{year}{2016}\natexlab{}.
\newblock \bibinfo{booktitle}{\emph{Deep Learning}}.
\newblock \bibinfo{publisher}{MIT Press}.
\newblock
\newblock
\shownote{\url{http://www.deeplearningbook.org}.}


\bibitem[\protect\citeauthoryear{Goodfellow, Shlens, and Szegedy}{Goodfellow
  et~al\mbox{.}}{2015}]%
        {Goodfellow:ICLR2015}
\bibfield{author}{\bibinfo{person}{Ian~J. Goodfellow},
  \bibinfo{person}{Jonathon Shlens}, {and} \bibinfo{person}{Christian
  Szegedy}.} \bibinfo{year}{2015}\natexlab{}.
\newblock \showarticletitle{Explaining and Harnessing Adversarial Examples}. In
  \bibinfo{booktitle}{\emph{International Conference on Learning
  Representations, {ICLR}}}.
\newblock


\bibitem[\protect\citeauthoryear{Huang, Kwiatkowska, Wang, and Wu}{Huang
  et~al\mbox{.}}{2017}]%
        {Huang:CAV2017}
\bibfield{author}{\bibinfo{person}{Xiaowei Huang}, \bibinfo{person}{Marta
  Kwiatkowska}, \bibinfo{person}{Sen Wang}, {and} \bibinfo{person}{Min Wu}.}
  \bibinfo{year}{2017}\natexlab{}.
\newblock \showarticletitle{Safety verification of deep neural networks}. In
  \bibinfo{booktitle}{\emph{International Conference on Computer Aided
  Verification {(CAV)}}}.
\newblock


\bibitem[\protect\citeauthoryear{Jeannet and Min{\'e}}{Jeannet and
  Min{\'e}}{2009}]%
        {apronlib}
\bibfield{author}{\bibinfo{person}{Bertrand Jeannet} {and}
  \bibinfo{person}{Antoine Min{\'e}}.} \bibinfo{year}{2009}\natexlab{}.
\newblock \showarticletitle{Apron: A library of numerical abstract domains for
  static analysis}. In \bibinfo{booktitle}{\emph{International Conference on
  Computer Aided Verification {(CAV)}}}. Springer, \bibinfo{pages}{661--667}.
\newblock


\bibitem[\protect\citeauthoryear{Julian, Kochenderfer, and Owen}{Julian
  et~al\mbox{.}}{2018}]%
        {julian2018deep}
\bibfield{author}{\bibinfo{person}{Kyle~D Julian}, \bibinfo{person}{Mykel~J
  Kochenderfer}, {and} \bibinfo{person}{Michael~P Owen}.}
  \bibinfo{year}{2018}\natexlab{}.
\newblock \showarticletitle{Deep neural network compression for aircraft
  collision avoidance systems}.
\newblock \bibinfo{journal}{\emph{Journal of Guidance, Control, and Dynamics}}
  \bibinfo{volume}{42}, \bibinfo{number}{3} (\bibinfo{year}{2018}),
  \bibinfo{pages}{598--608}.
\newblock


\bibitem[\protect\citeauthoryear{Katz, Barrett, Dill, Julian, and
  Kochenderfer}{Katz et~al\mbox{.}}{2017}]%
        {reluplex:CAV2017}
\bibfield{author}{\bibinfo{person}{Guy Katz}, \bibinfo{person}{Clark Barrett},
  \bibinfo{person}{David~L Dill}, \bibinfo{person}{Kyle Julian}, {and}
  \bibinfo{person}{Mykel~J Kochenderfer}.} \bibinfo{year}{2017}\natexlab{}.
\newblock \showarticletitle{Reluplex: An efficient SMT solver for verifying
  deep neural networks}. In \bibinfo{booktitle}{\emph{International Conference
  on Computer Aided Verification {(CAV)}}}.
\newblock


\bibitem[\protect\citeauthoryear{Krizhevsky, Sutskever, and Hinton}{Krizhevsky
  et~al\mbox{.}}{2017}]%
        {Krizhevsky:CACM2017}
\bibfield{author}{\bibinfo{person}{Alex Krizhevsky}, \bibinfo{person}{Ilya
  Sutskever}, {and} \bibinfo{person}{Geoffrey~E. Hinton}.}
  \bibinfo{year}{2017}\natexlab{}.
\newblock \showarticletitle{ImageNet classification with deep convolutional
  neural networks}.
\newblock \bibinfo{journal}{\emph{Commun. {ACM}}} \bibinfo{volume}{60},
  \bibinfo{number}{6} (\bibinfo{year}{2017}), \bibinfo{pages}{84--90}.
\newblock


\bibitem[\protect\citeauthoryear{Mirman, Gehr, and Vechev}{Mirman
  et~al\mbox{.}}{2018}]%
        {diffai2018}
\bibfield{author}{\bibinfo{person}{Matthew Mirman}, \bibinfo{person}{Timon
  Gehr}, {and} \bibinfo{person}{Martin Vechev}.}
  \bibinfo{year}{2018}\natexlab{}.
\newblock \showarticletitle{Differentiable abstract interpretation for provably
  robust neural networks}. In \bibinfo{booktitle}{\emph{International
  Conference on Machine Learning {ICML}}}.
\newblock


\bibitem[\protect\citeauthoryear{Moosavi{-}Dezfooli, Fawzi, and
  Frossard}{Moosavi{-}Dezfooli et~al\mbox{.}}{2016}]%
        {DeepFool:CVPR2016}
\bibfield{author}{\bibinfo{person}{Seyed{-}Mohsen Moosavi{-}Dezfooli},
  \bibinfo{person}{Alhussein Fawzi}, {and} \bibinfo{person}{Pascal Frossard}.}
  \bibinfo{year}{2016}\natexlab{}.
\newblock \showarticletitle{DeepFool: {A} Simple and Accurate Method to Fool
  Deep Neural Networks}. In \bibinfo{booktitle}{\emph{2016 {IEEE} Conference on
  Computer Vision and Pattern Recognition, {CVPR} 2016, Las Vegas, NV, USA,
  June 27-30, 2016}}. \bibinfo{pages}{2574--2582}.
\newblock


\bibitem[\protect\citeauthoryear{Nguyen, Yosinski, and Clune}{Nguyen
  et~al\mbox{.}}{2015}]%
        {Nguyen:CVPR2015}
\bibfield{author}{\bibinfo{person}{Anh~Mai Nguyen}, \bibinfo{person}{Jason
  Yosinski}, {and} \bibinfo{person}{Jeff Clune}.}
  \bibinfo{year}{2015}\natexlab{}.
\newblock \showarticletitle{Deep neural networks are easily fooled: High
  confidence predictions for unrecognizable images}. In
  \bibinfo{booktitle}{\emph{{IEEE} Conference on Computer Vision and Pattern
  Recognition, {(CVPR)} 2015, Boston, MA, USA, June 7-12, 2015}}.
  \bibinfo{pages}{427--436}.
\newblock


\bibitem[\protect\citeauthoryear{Patterson and Gibson}{Patterson and
  Gibson}{2017}]%
        {patterson2017deep}
\bibfield{author}{\bibinfo{person}{Josh Patterson} {and} \bibinfo{person}{Adam
  Gibson}.} \bibinfo{year}{2017}\natexlab{}.
\newblock \bibinfo{booktitle}{\emph{Deep learning: A practitioner's approach}}.
\newblock \bibinfo{publisher}{" O'Reilly Media, Inc."}.
\newblock


\bibitem[\protect\citeauthoryear{Scheibler, Winterer, Wimmer, and
  Becker}{Scheibler et~al\mbox{.}}{2015}]%
        {DBLP:conf/mbmv/ScheiblerWWB15}
\bibfield{author}{\bibinfo{person}{Karsten Scheibler}, \bibinfo{person}{Leonore
  Winterer}, \bibinfo{person}{Ralf Wimmer}, {and} \bibinfo{person}{Bernd
  Becker}.} \bibinfo{year}{2015}\natexlab{}.
\newblock \showarticletitle{Towards Verification of Artificial Neural
  Networks}. In \bibinfo{booktitle}{\emph{Methoden und Beschreibungssprachen
  zur Modellierung und Verifikation von Schaltungen und Systemen, {MBMV} 2015,
  Chemnitz, Germany, March 3-4, 2015.}} \bibinfo{pages}{30--40}.
\newblock


\bibitem[\protect\citeauthoryear{Singh, Gehr, P{\"{u}}schel, and Vechev}{Singh
  et~al\mbox{.}}{2019}]%
        {Singh:POPL2019}
\bibfield{author}{\bibinfo{person}{Gagandeep Singh}, \bibinfo{person}{Timon
  Gehr}, \bibinfo{person}{Markus P{\"{u}}schel}, {and}
  \bibinfo{person}{Martin~T. Vechev}.} \bibinfo{year}{2019}\natexlab{}.
\newblock \showarticletitle{An abstract domain for certifying neural networks}.
\newblock \bibinfo{journal}{\emph{{PACMPL}}} \bibinfo{volume}{3},
  \bibinfo{number}{{POPL}} (\bibinfo{year}{2019}),
  \bibinfo{pages}{41:1--41:30}.
\newblock


\bibitem[\protect\citeauthoryear{Szegedy, Vanhoucke, Ioffe, Shlens, and
  Wojna}{Szegedy et~al\mbox{.}}{2016}]%
        {Szegedy:CVPR2016}
\bibfield{author}{\bibinfo{person}{Christian Szegedy}, \bibinfo{person}{Vincent
  Vanhoucke}, \bibinfo{person}{Sergey Ioffe}, \bibinfo{person}{Jon Shlens},
  {and} \bibinfo{person}{Zbigniew Wojna}.} \bibinfo{year}{2016}\natexlab{}.
\newblock \showarticletitle{Rethinking the inception architecture for computer
  vision}. In \bibinfo{booktitle}{\emph{Proceedings of the IEEE conference on
  computer vision and pattern recognition {(CVPR)}}}.
\newblock


\bibitem[\protect\citeauthoryear{Szegedy, Zaremba, Sutskever, Bruna, Erhan,
  Goodfellow, and Fergus}{Szegedy et~al\mbox{.}}{2014}]%
        {Szegedy:ICLR2014}
\bibfield{author}{\bibinfo{person}{Christian Szegedy},
  \bibinfo{person}{Wojciech Zaremba}, \bibinfo{person}{Ilya Sutskever},
  \bibinfo{person}{Joan Bruna}, \bibinfo{person}{Dumitru Erhan},
  \bibinfo{person}{Ian~J. Goodfellow}, {and} \bibinfo{person}{Rob Fergus}.}
  \bibinfo{year}{2014}\natexlab{}.
\newblock \showarticletitle{Intriguing properties of neural networks}. In
  \bibinfo{booktitle}{\emph{International Conference on Learning
  Representations, {ICLR}}}.
\newblock


\bibitem[\protect\citeauthoryear{Thrun}{Thrun}{1994}]%
        {DBLP:conf/nips/Thrun94}
\bibfield{author}{\bibinfo{person}{Sebastian Thrun}.}
  \bibinfo{year}{1994}\natexlab{}.
\newblock \showarticletitle{Extracting Rules from Artifical Neural Networks
  with Distributed Representations}. In \bibinfo{booktitle}{\emph{Advances in
  Neural Information Processing Systems 7, {[NIPS} Conference, Denver,
  Colorado, USA, 1994]}}. \bibinfo{pages}{505--512}.
\newblock


\bibitem[\protect\citeauthoryear{Wang, Pei, Whitehouse, Yang, and Jana}{Wang
  et~al\mbox{.}}{2018}]%
        {ReluVal:Usenix2018}
\bibfield{author}{\bibinfo{person}{Shiqi Wang}, \bibinfo{person}{Kexin Pei},
  \bibinfo{person}{Justin Whitehouse}, \bibinfo{person}{Junfeng Yang}, {and}
  \bibinfo{person}{Suman Jana}.} \bibinfo{year}{2018}\natexlab{}.
\newblock \showarticletitle{Formal Security Analysis of Neural Networks using
  Symbolic Intervals}. In \bibinfo{booktitle}{\emph{27th {USENIX} Security
  Symposium, {USENIX} Security 2018, Baltimore, MD, USA, August 15-17, 2018.}}
  \bibinfo{pages}{1599--1614}.
\newblock


\bibitem[\protect\citeauthoryear{Weng, Zhang, Chen, Song, Hsieh, Daniel,
  Boning, and Dhillon}{Weng et~al\mbox{.}}{2018}]%
        {Weng:ICML2018}
\bibfield{author}{\bibinfo{person}{Tsui{-}Wei Weng}, \bibinfo{person}{Huan
  Zhang}, \bibinfo{person}{Hongge Chen}, \bibinfo{person}{Zhao Song},
  \bibinfo{person}{Cho{-}Jui Hsieh}, \bibinfo{person}{Luca Daniel},
  \bibinfo{person}{Duane~S. Boning}, {and} \bibinfo{person}{Inderjit~S.
  Dhillon}.} \bibinfo{year}{2018}\natexlab{}.
\newblock \showarticletitle{Towards Fast Computation of Certified Robustness
  for ReLU Networks}. In \bibinfo{booktitle}{\emph{International Conference on
  Machine Learning, {(ICML)}}}.
\newblock


\bibitem[\protect\citeauthoryear{Xiang, Tran, Rosenfeld, and Johnson}{Xiang
  et~al\mbox{.}}{2018}]%
        {Xiang:ACC2018}
\bibfield{author}{\bibinfo{person}{Weiming Xiang},
  \bibinfo{person}{Hoang{-}Dung Tran}, \bibinfo{person}{Joel~A. Rosenfeld},
  {and} \bibinfo{person}{Taylor~T. Johnson}.} \bibinfo{year}{2018}\natexlab{}.
\newblock \showarticletitle{Reachable Set Estimation and Safety Verification
  for Piecewise Linear Systems with Neural Network Controllers}. In
  \bibinfo{booktitle}{\emph{2018 Annual American Control Conference, {(ACC)}}}.
\newblock


\bibitem[\protect\citeauthoryear{Xiang, Tran, and Johnson}{Xiang
  et~al\mbox{.}}{2017}]%
        {Xiang:arxiv2017}
\bibfield{author}{\bibinfo{person}{Weiming Xiang}, \bibinfo{person}{Hoang-Dung
  Tran}, {and} \bibinfo{person}{Taylor~T Johnson}.}
  \bibinfo{year}{2017}\natexlab{}.
\newblock \showarticletitle{Reachable set computation and safety verification
  for neural networks with ReLU activations}.
\newblock \bibinfo{journal}{\emph{arXiv preprint arXiv:1712.08163}}
  (\bibinfo{year}{2017}).
\newblock


\bibitem[\protect\citeauthoryear{Zhu, Xiong, Magill, and Jagannathan}{Zhu
  et~al\mbox{.}}{2019}]%
        {DBLP:conf/pldi/ZhuXMJ19}
\bibfield{author}{\bibinfo{person}{He Zhu}, \bibinfo{person}{Zikang Xiong},
  \bibinfo{person}{Stephen Magill}, {and} \bibinfo{person}{Suresh
  Jagannathan}.} \bibinfo{year}{2019}\natexlab{}.
\newblock \showarticletitle{An inductive synthesis framework for verifiable
  reinforcement learning}. In \bibinfo{booktitle}{\emph{Proceedings of the 40th
  {ACM} {SIGPLAN} Conference on Programming Language Design and Implementation,
  {PLDI} 2019, Phoenix, AZ, USA, June 22-26, 2019.}} \bibinfo{pages}{686--701}.
\newblock


\end{thebibliography}

\clearpage
\appendix
\section{\pref{thm:no-fhat}}
\label{app:no-fhat}
\ThmNoFHat*

\begin{proof}
    One such function is $f(x) = x^2$, which is clearly differentiable
    everywhere with $\frac{d}{dx}f = 2x$.

    Suppose there existed some $\hat{f} = \{ (P_1, F_1), (P_2, F_2), \ldots,
    (P_n, F_n) \}$. Then, as the domain of $f$ is infinite ($\mathbb{R}$) and
    is partitioned with finitely many polytopes, at least one $P_i$ must have
    infinitely many (and thus more than one) distinct points.

    Now, consider any two points $x_1 \neq x_2 \in P_i$. Then we have:

    \begin{equation}
        \begin{aligned}
            &\left(\frac{x_1 + x_2}{2}\right)^2 = F_i\frac{x_1 + x_2}{2} \\
            \implies &\frac{x_1^2 + 2x_1x_2 + x_2^2}{4} = \frac{F_ix_1 + F_ix_2}{2} \\
            \implies &\frac{x_1^2 + 2x_1x_2 + x_2^2}{4} = \frac{x_1^2 + x_2^2}{2} \\
            \implies &x_1^2 + 2x_1x_2 + x_2^2 = 2x_1^2 + 2x_2^2 \\
            \implies &2x_1x_2 = x_1^2 + x_2^2 \\
            \implies &0 = x_1^2 - 2x_1x_2 + x_2^2 \\
            \implies &0 = (x_1 - x_2)^2 \\
            \implies &x_1 = x_2 \\
        \end{aligned}
    \end{equation}

    Which contradicts our assumption that $x_1 \neq x_2$.
\end{proof}

\section{\pref{thm:no-restricted-fhat}}
\ThmNoRestrictedFHat*

\begin{proof}
    We observe that the argument from~\pref{thm:no-fhat} holds here as well;
    if $X$ is a non-singleton, non-empty polytope, then it contains infinitely
    many points and so any finite partitioning of it will have at least one
    partition containing more than one distinct point. From there, the proof is
    identical.
\end{proof}

\section{\pref{thm:only-restricted-fhat}}
\ThmOnlyRestrictedFHat*

\begin{proof}
    Consider the function $f(x_1, x_2) = x_1^2 - x_2^2 = (x_1 + x_2)(x_1 -
    x_2)$ and the bounded polytope $X$ defined by $\{ x \mid 0 \leq x_1 \leq 1
    \wedge x_1 - x_2 = 1 \}$.

    We first show that $\hat{f}$ does not exist. Suppose $\hat{f} = \{ (P_1,
    F_1), \ldots, (P_n, F_n) \}$ exists. Then, it holds that:

    \[
        \hatr{f}{Z} = \{ (P_1 \cap Z, F_1), \ldots, (P_n \cap Z, F_n) \}
    \]

    Also exists, where $Z$ is the convex polytope defined by $x_2 = 0$.
    However, $\restr{f}{Z} = f(x_1, 0) = x_1^2$, so this implies that
    $\hat{(x \mapsto x^2)}$ exists, contradicting the proof of non-existence
    given in~\pref{thm:no-fhat}.

    Finally, we show that $\hatr{f}{X}$ does exist, where $X = \{ x \mid x_1
    \leq 1 \wedge x_1 - x_2 = 1$. In that case, all $x$ values satisfy $x_1 -
    x_2 = 1$, so $f(x) = (x_1 + x_2)(x_1 - x_2) = x_1 + x_2$, which is affine.
    Thus $\hatr{f}{X} = \{ (X, x \mapsto x_1 + x_2) \}$ is a valid solution.
\end{proof}

\section{\pref{thm:PWL-Hat}}
\ThmPWLHat*

\begin{proof}
    Let our piecewise-linear function have linear regions $\{ R_1, \ldots, R_k
    \}$ with associated affine maps $\{ M_1, \ldots, M_k \}$.

    Let $\hat{g} = \{ (P_1, F_1), \ldots, (P_m, F_m) \}$.

    We note that, because the $P_i$s partition the domain of $g$, it suffices
    to show the construction on any of the $(P_i, F_i)$ pairs individually; the
    union of the results for each $(P_i, F_i)$ pairs will form $f \otimes
    \hat{g}$.

    Now, consider any $(P_i, F_i)$ pair in $\hat{g}$. We compute the
    \emph{post-set of $P_i$ under $F_i$}:

    \[
        Y_i = \{ F_ix \mid x \in P_i \}
    \]

    Which is a convex polytope, the vertices of which can be computed by
    applying $F_i$ to each of the finitely-many vertices of $P_i$ (this a
    well-known fact in convex geometry).

    Next, we compute $Y^j_i = Y_i \cap R_j$, where $R_j$ again is the $j$th
    linear region of $f$ associated with map $B_j$. This again is computable,
    as there are finitely many such $R_j$s and the intersection of two convex
    polytopes can be computed by conjunction of their half-space faces.

    We have now partitioned $Y_i$ according to the linear regions of $f$, all
    that remains is to find a corresponding partitioning of $P_i$ such that
    each partition $P^j_i$ maps to $Y^j_i$ under $F_i$.

    For each vertex $v_l$ of $Y^j_i$, we then compute $c_l$, a convex
    combination of the vertices of $Y_i$ that forms $v_l$, i.e.
    $\mathrm{Vert}(Y_i)\cdot c_l = v_l$. Notably, to deal with non-invertible
    $F_i$ maps, when multiple $Y^j_i$ share the same vertex $v_l$, we must pick
    the same combination $c_l$ (which can be accomplished via a memo dictionary
    of previously found $(v_l, c_l)$ pairs). Furthermore, when $v_l$ is also a
    vertex of $Y_i$, we must use the corresponding vertex from $P_i$.

    Now, for each of the $(v_l, c_l)$ pairs, we compute $p_l =
    \mathrm{Vert}(P_i)c_l$. We note that, because $F_i$ is affine:

    \[
        F_i(p_l) = F_i(\mathrm{Vert}(P_i)c_l) = \mathrm{Vert}(Y_i)c_l = v_l
    \]

    We note now that the $p_l$s partition $P_i$ because the $v_l$s partition
    $Y_i$ and we mapped each $v_l$ back uniquely to a $p_l$ (mapping vertices
    to vertices where applicable). We now define:

    \[
        P^j_i =
        \{ p_l \mid
        \text{ the corresponding } v_l \text{ is a vertex of } Y^j_i \}
    \]

    Finally, we claim that:

    \[
        \hat{f \circ g} = \{ (P^1_i, M_1 \circ A_i), (P^2_i, M_2 \circ A_i), \ldots, (P^k_i, M_k \circ A_i) \mid 1 \leq i \leq m \}
    \]

    This follows as all points in $P_i$ are first mapped to $F_ix \in Y^j_i$ by
    $g$ then each point $F_ix$ in $Y^j_i$ is mapped to $B_j(A_ix) = (B_jA_i)x$
    by the application of $f$.
\end{proof}

\section{Corollaries \ref{cor:fullyconnected-hat}---\ref{cor:maxpool-hat}}
These follow as all of the described functions are well-known to be
piecewise-linear.

\section{\pref{cor:network-hat}}
\label{app:network-hat}
\CorNetworkHat*

\begin{proof}
    Let the layers be $l_1, l_2, \ldots, l_n$ such that $f = l_n \circ l_{n-1}
    \circ \cdots \circ l_1$.

    Now, we have:

    \[
        \begin{aligned}
            l_n \otimes \cdots \otimes l_2 \otimes l_1 \otimes \hat{g}
            &= l_n \otimes \cdots \otimes l_2 \otimes \hat{(l_1 \circ g)} \\
            &= l_n \otimes \cdots \otimes \hat{(l_2 \circ l_1 \circ g)} \\
            \vdots \\
            &= \hat{(l_n \circ \cdots \circ l_1 \circ g)} \\
            &= \hat{(f \circ g)} \\
            &= f \otimes \hat{g} \\
        \end{aligned}
    \]

    Thus, $f \otimes \hat{g}$ is computable by repeated application of $l_i
    \otimes \cdot$ for descending $i > 0$, and $l_i \otimes \cdot$ was shown
    computable in Corollaries
    \ref{cor:fullyconnected-hat}---\ref{cor:maxpool-hat}.
\end{proof}

\section{\pref{thm:split-plane}}
\label{app:split-plane}
\ThmSplitPlane*

\begin{proof}
    It is well-known in computational geometry that the vertices of the
    intersection of a polytope with a hyperplane is the intersection of its
    edges with the plane. The algorithm computes the ratio along each edge that
    the intersection is reached using standard linear algebra, i.e.:

    \[
        \begin{aligned}
            &(Q + \alpha(R - Q))_d = 0 \\
            \implies &Q_d + \alpha(R_d - Q_d) = 0 \\
            \implies &\alpha = -\frac{Q_d}{R_d - Q_d} \\
        \end{aligned}
    \]
    Because $F$ is affine, the ratio holds when taking the preimage as well
    (which is what is returned).
\end{proof}

\section{\pref{thm:relu-extend}}
\label{app:relu-extend}
\ThmReluExtend*

\begin{proof}
    It suffices to show that the algorithm correctly partitions each input
    polytope $P$ such that the signs within a partition are constant. Notably,
    because of convexity, it suffices to show that the signs of the vertices of
    each partition are constant.

    We maintain two invariants every time we process some $P, F$ pair from the
    queue. The first is that the corresponding polytope will only be added to
    $Y$ if the signs of all vertices are constant (or zero). The second is that
    at each step, we partition the polytope into two new ones (using
    \SplitPlane) such that fewer sign switches happen in each than the original
    polytope. This follows from the correctness of the SplitPlane algorithm.

    The first invariant ensures that, if it halts, the algorithm is correct.
    The second ensures that it will halt, as there are only finitely many
    dimensions to consider.
\end{proof}

\section{\pref{thm:KeyPoints}}
\label{app:key-points}

\ThmKeyPoints*

\begin{proof}
    This follows as the function within each partition is affine, and thus
    convex. If all of the vertices fall inside of $Y$, then by definition of
    convexity all of the interior points must as well (and vice-versa).
\end{proof}

\end{document}